\documentclass[11pt,letterpaper]{article}

\usepackage[T1]{fontenc}
\usepackage[utf8]{inputenc}
\usepackage{lmodern}
\usepackage[margin=1in]{geometry}
\usepackage[round,authoryear]{natbib}

\usepackage{amsmath,amssymb,amsfonts,bm,mathtools,mathrsfs}
\usepackage{amsthm}
\usepackage[dvipsnames]{xcolor}
\usepackage{graphicx}
\usepackage{booktabs}
\usepackage{microtype}
\usepackage{nicefrac}
\usepackage{enumitem}
\usepackage{xurl}
\usepackage[colorlinks=true,allcolors=MidnightBlue]{hyperref}
\usepackage[capitalize,nameinlink]{cleveref}

\allowdisplaybreaks
\providecommand{\coloneq}{\coloneqq}

\newcommand{\E}{\mathbb{E}}
\newcommand{\PB}{\mathbb{P}}
\newcommand{\RB}{\mathbb{R}}
\newcommand{\ind}{\mathbf{1}}
\newcommand{\prn}[1]{\left(#1\right)}
\newcommand{\brk}[1]{\left[#1\right]}
\newcommand{\brc}[1]{\left\{#1\right\}}
\newcommand{\abs}[1]{\left|#1\right|}
\newcommand{\norm}[1]{\left\|#1\right\|}

\newcommand{\wtilde}[1]{\widetilde{#1}}
\newcommand{\diag}{\operatorname{diag}}

\newcommand{\TV}{\mathrm{TV}}

\def\gS{{\mathcal S}}
\def\gF{{\mathcal F}}
\def\gG{{\mathcal G}}

\def\bB{{\bm B}}

\def\bD{{\bm D}}
\def\be{{\bm e}}
\def\bI{{\bm I}}

\def\bP{{\bm P}}
\def\bPhi{{\bm\Phi}}
\def\bp{{\bm p}}
\def\bQ{{\bm Q}}
\def\bR{{\bm R}}
\def\br{{\bm r}}
\def\bv{{\bm v}}
\def\bu{{\bm u}}
\def\bq{{\bm q}}
\def\bV{{\bm V}}
\def\bY{{\bm Y}}
\def\bU{{\bm U}}
\def\bDelta{{\bm\Delta}}
\def\bmu{{\bm\mu}}
\def\bnu{{\bm\nu}}
\def\bone{{\bm 1}}
\def\tmix{t_{\operatorname{mix}}}
\def\mumin{\mu_{\min}}

\hypersetup{
  pdftitle={Near-Minimax-Optimal Rates for Asynchronous TD Learning with Markovian Data},
  pdfauthor={Yang Peng}
}

\newtheorem{theorem}{Theorem}[section]
\newtheorem{lemma}[theorem]{Lemma}
\newtheorem{proposition}[theorem]{Proposition}
\newtheorem{corollary}[theorem]{Corollary}
\newtheorem{assumption}[theorem]{Assumption}
\theoremstyle{remark}
\newtheorem{remark}[theorem]{Remark}
\newtheorem{example}[theorem]{Example}

\title{Near-Minimax-Optimal Rates for \\ Asynchronous TD Learning with Markovian Data}

\author{Yang Peng\\
Yau Mathematical Sciences Center, Tsinghua University\\
\texttt{yang-peng@mail.tsinghua.edu.cn}}
\date{}

\begin{document}
\maketitle

\begin{abstract}
We study the last iterate of standard importance-weighted temporal-difference (TD) learning
from a single trajectory of a finite Markov decision process.
The trajectory follows a behavior policy $b$ whose state chain is
irreducible and may be periodic, while a target policy $\pi$ covered by $b$
is evaluated using action importance ratios bounded by $\varrho$.
For discount factor $\gamma$, write $H=(1-\gamma)^{-1}$, and let
$\mumin$ denote the minimum stationary probability and
$\mathfrak h_\star$ the largest expected hitting time of the reverse state chain;
both quantities refer to the behavior state chain.
We prove that last-iterate TD learning achieves sup-norm error at most
$\varepsilon$ with high probability using
\[
\wtilde O\prn{
  \frac{\varrho H^3}{\mumin\varepsilon^2}
  +\mathfrak h_\star
}
\]
transitions, for $0<\varepsilon\leq1$.
Setting $b=\pi$ and $\varrho=1$ recovers the on-policy rate.
The first term captures statistical error and retains the cubic
effective-horizon dependence of synchronous TD learning.
The second term controls initialization error and is additive in
$\mathfrak h_\star$.
For aperiodic behavior chains, the comparison $\mathfrak h_\star=\wtilde O(\tmix/\mumin)$ recovers
the corresponding mixing-time bound, where $\tmix$ is the total-variation
mixing time.
The proof uses an anchored local Poisson equation in reverse time to control stochastic fluctuations without a mixing-time factor, and a
compensation argument for successive hitting costs to bound initialization error.
A bound on the expected cumulative propagation mass extends this argument from constant step sizes to decreasing step sizes.
Given the policy ratios and their bound, an empirical state-count rule
attains the same hitting-time bound without prior knowledge of any unknown
model parameters.
We formulate the argument for importance-weighted nonnegative affine prediction updates,
covering successor features, transition-dependent discounting, and
coupled finite-mode predictions with the same rate.
A perturbation estimate also yields a time-uniform guarantee for
online estimation of return variance.
A three-state, two-action construction gives matching minimax lower bounds
for the statistical and hitting-time terms up to logarithmic factors
on the parameter classes considered here.
\end{abstract}

\section{Introduction}
\label{Section:intro}
Policy evaluation asks how much discounted reward a fixed policy will collect from each starting state.
It is a basic component of reinforcement learning, both as a task in its own right and as a step in improving a policy.
Temporal-difference (TD) learning estimates these values by updating the current estimate after each observed transition \citep{sutton1988learning}.
In its standard asynchronous form, TD learning updates only the coordinate corresponding to the current state and stores only the value estimate, making it well suited to online learning from a single trajectory.
We allow this trajectory to follow a behavior policy $b$ while evaluating
a possibly different target policy $\pi$. The importance ratio of the
observed action multiplies the entire TD residual.

The statistical analysis depends on how those samples are obtained.
Under synchronous generative sampling, one can request an independent transition from every state in each round.
Under trajectory sampling, the next state is sampled from the transition distribution at the current state under the behavior policy.
Some states are visited less frequently than others, and consecutive updates are dependent.
Alongside the effective horizon $H=(1-\gamma)^{-1}$, existing guarantees
use the minimum stationary probability $\mumin$ to measure state coverage
and the mixing time $\tmix$ to quantify temporal dependence
\citep{qu2020finite,li2022async,li2024q}.
In the off-policy setting considered here, these chain parameters refer to
the behavior policy, and $\varrho$ bounds the action importance ratios.

For on-policy synchronous tabular TD learning, \citet[Theorem~1]{li2024q} established the sharp sample complexity $\wtilde O(H^3/\varepsilon^2)$ per state.
The total number of transitions is therefore $\wtilde O(dH^3/\varepsilon^2)$, where $d$ is the number of states, and the notation $\wtilde O$ suppresses polylogarithmic factors in the problem parameters.
For a single Markov trajectory, their asynchronous Q-learning theorem specializes to TD learning when there is one action and gives
\begin{equation*}
\wtilde O\prn{\frac{H^4}{\mumin\varepsilon^2}
                   +\frac{\tmix H}{\mumin}}.
\end{equation*}
This bound has additive dependence on the mixing time, but leaves a factor-$H$ gap with the synchronous statistical rate.

We prove that importance-weighted asynchronous TD learning retains the cubic statistical term with an additive transient governed by hitting times of the reversed state chain.
Specifically, for $0<\varepsilon\leq1$, the last-iterate estimator with a suitable step-size schedule achieves high-probability sup-norm accuracy $\varepsilon$ with
\begin{equation}\label{eq:intro_rate}
\wtilde O\prn{\frac{\varrho H^3}{\mumin\varepsilon^2}
                   +\mathfrak h_\star}
\end{equation}
transitions, where $\mathfrak h_\star$ is the largest expected hitting time between two states of the reversed behavior chain.
Setting $b=\pi$ and $\varrho=1$ recovers the on-policy rate.
For aperiodic behavior chains, Section~\ref{subsection:state_chain} relates this quantity to mixing time:
$\mathfrak h_\star\lesssim(\tmix/\mumin)\log(2/\mumin)$.
Consequently, we recover the sufficient transition count in terms of the mixing time:
\begin{equation}\label{eq:intro_mixing_rate}
\wtilde O\prn{\frac{\varrho H^3}{\mumin\varepsilon^2}
                   +\frac{\tmix}{\mumin}}.
\end{equation}
The hitting-time bound can be substantially smaller for some chains,
as Example~\ref{example:hitting_vs_mixing} illustrates.
Our results hold for any initial state distribution and for finite,
irreducible behavior chains, including non-reversible and periodic chains.
The target chain need not be irreducible.
Allowing periodic behavior chains permits direct application to seasonal inventory models when the calendar phase is included in the state \citep{riis1965discounted,kapuscinski1998capacitated}, as in  Example~\ref{example:seasonal}.


The proof separates stochastic fluctuations from initialization error while retaining their exact weighted random propagation matrices.
For the stochastic term, a local Poisson equation in reverse time and a target-policy variance identity for discounted returns give a uniform expected-energy bound, which we convert into an exponential tail bound.
For initialization, we represent the remaining propagation mass as the survival probability of an auxiliary particle.
Applying a classical hitting-time reversal identity to its successive relocation costs bounds their total contribution by a single $\mathfrak h_\star$ term, without an extra factor of $H$.
These two estimates yield the statistical term and additive transient in \eqref{eq:intro_rate}.
For decreasing steps, an expected accumulated-mass bound replaces the
time-homogeneous particle-lifetime equations used in the analysis for constant steps. 
This bound preserves the additive hitting-time cost when the step size changes.
An empirical state-count rule attains the same rate using the known policy
ratios and their bound, without knowing the transition law, stationary
probabilities, or hitting times.

Theorem~\ref{theorem:hitting_minimax} gives a minimax lower bound matching \eqref{eq:intro_rate} up to logarithmic factors over the hitting-time classes specified there.
Corollary~\ref{corollary:budget_minimax} gives the corresponding mixing-time minimax lower bound for \eqref{eq:intro_mixing_rate}.

We develop these arguments for a class of importance-weighted affine prediction updates that
transport information along the observed transition through a nonnegative
matrix with row sums at most $\gamma$.
Theorem~\ref{theorem:affine_transfer} gives the same hitting-time rate for
this class, with the number of predictions entering only logarithmically
when all prediction components are updated at each visited state.
Its applications include successor features, predictions with
transition-dependent discounts, and coupled finite-mode predictions.
We also use a perturbation bound to analyze simultaneous online
estimation of the mean and variance of the discounted return, obtaining
a sufficient count $\wtilde O(\varrho H^5/(\mumin\varepsilon^2)+\mathfrak h_\star)$
for variance accuracy $\varepsilon$.

\subsection{Related work}

\paragraph{Comparison with existing guarantees.}
The comparisons in this paragraph and Table~\ref{table:comparison}
concern the on-policy case $b=\pi$ and $\varrho=1$.
For variance-reduced asynchronous TD learning, \citet[Theorem~4]{li2022async} establish a guarantee with the same cubic statistical term but an additive $\tmix H/\mumin$ transient.
Our result applies to the ordinary TD learning recursion and controls the transient by $\mathfrak h_\star$; its mixing-time consequence removes the extra effective-horizon factor.
Table~\ref{table:comparison} summarizes the comparison.

\begin{table}[t]
\centering
\small
\caption{Sufficient total transition counts on the on-policy subclass ($\varrho=1$), up to logarithmic factors. The first row uses a synchronous generative model; the other rows use a single Markov trajectory. Our hitting-time bound covers constant and decreasing step sizes and allows periodic chains; its mixing-time consequence assumes aperiodicity.}
\label{table:comparison}
\begin{tabular}{@{}ll@{}}
\toprule
Method and source & Total transition count \\
\midrule
Synchronous TD learning \citep[Theorem~1]{li2024q}
  & $\wtilde O(dH^3/\varepsilon^2)$ \\
Asynchronous TD learning \citep[Theorem~4, one action]{li2024q}
  & $\wtilde O(H^4/(\mumin\varepsilon^2)+\tmix H/\mumin)$ \\
Variance-reduced TD learning \citep[Theorem~4, one action]{li2022async}
  & $\wtilde O(H^3/(\mumin\varepsilon^2)+\tmix H/\mumin)$ \\
Asynchronous TD learning (Theorems~\ref{theorem:main}, \ref{theorem:decay})
  & $\wtilde O(H^3/(\mumin\varepsilon^2)+\mathfrak h_\star)$ \\
Mixing-time consequence (Corollary~\ref{corollary:mixing_time})
  & $\wtilde O(H^3/(\mumin\varepsilon^2)+\tmix/\mumin)$ \\
\bottomrule
\end{tabular}
\end{table}

\paragraph{Off-policy TD learning.}
\citet{chen2021offpolicy} analyze a generalized multi-step off-policy
Bellman operator and obtain finite-sample guarantees for importance
sampling, Tree-Backup, Retrace, and related algorithms. Their
formulation estimates action values and includes importance weights
inside the bootstrapped target. Our main theorem concerns the scalar
state-value update with the entire residual weighted by the current
action ratio, as specified in \eqref{eq:td_update_vector}.
We identify its explicit dependence on the effective horizon, the
ratio bound, state coverage, and the additive behavior hitting time.
The weighted SARSA application in Section~\ref{Section:extensions}
states its own update and uses state--action coverage parameters.

\paragraph{Markovian stochastic approximation.}
Classical work establishes the convergence of asynchronous stochastic approximation and Q-learning \citep{tsitsiklis1994asynchronous}.
Finite-sample analyses include moment bounds for linear recursions \citep{srikant2019finite} and Lyapunov bounds for contractive recursions \citep{chen2024lyapunov}.
\citet{mou2024optimal} establish last-iterate squared-error bounds and instance-dependent guarantees for averaged iterates that match a local minimax benchmark, with applications to TD($\lambda$) learning.
\citet{samsonov2024improved} derive high-probability bounds for linear TD learning with Polyak--Ruppert tail averaging, using exponential stability of random matrix products and data dropping in the Markovian setting.
\citet{wu2025uncertainty} obtain high-probability Euclidean error bounds for averaged linear TD learning whose leading term reflects its asymptotic variance.
For batch policy evaluation, \citet{duan2022singlepath} analyze multi-step least-squares TD estimators from a single trajectory and obtain a leading statistical term without a mixing-time factor under correct specification; mixing still enters the sample-size threshold.
Our target is high-probability sup-norm control of the tabular last iterate, with explicit dependence on $H$, $\varrho$, $\mumin$, and the reverse hitting-time parameter $\mathfrak h_\star$.
We use reverse-time analysis to bound the accumulated Bellman noise.

\paragraph{Hitting and return times.}
\citet{dai2021loop} analyze a loop estimator for discounted policy evaluation from a single trajectory, whose statistical error bound depends multiplicatively on a maximal expected hitting-time parameter.
\citet{chandak2022concentration} establish time-uniform concentration
for contractive Markovian stochastic approximation, with applications
to asynchronous Q-learning and TD learning.
Hitting times enter their Poisson-equation constants; their bounds do
not separate statistical error from an additive hitting-time transient.
\citet{singh2026regret} analyze online Q-learning with time-varying
exploration. For uniform exploration, their Theorem~4.3 gives
$\wtilde O(H^5/\varepsilon^2)$ sample complexity with the other model
parameters fixed. This result concerns learning the optimal Q-function;
our cubic statistical term concerns fixed-policy evaluation.
\citet{even2023stochastic} derive hitting-time bounds for stochastic
averaged gradient methods under Markov sampling.
\citet{powell2024stochastic} analyze incremental surrogate methods under
return-time conditions that also allow periodic sampling.
For standard TD learning, our bound has a cubic statistical term and
an additive reverse hitting-time term.

\paragraph{Reverse-time analysis.}
\citet{tarres2014online} use reverse martingale decompositions for
stochastic approximation with independent observations.
\citet{huo2022bias} use the reverse Markov kernel to analyze the
stationary bias of constant-step linear stochastic approximation.
We use reverse-time Poisson equations to obtain last-iterate error
bounds with cubic effective-horizon dependence.
For initialization, we apply the classical hitting-time reversal identity
of \citet{beveridge2010exit} to obtain an additive hitting-time cost.

\subsection{Contributions}
Our contributions are threefold.
\begin{enumerate}[leftmargin=*,itemsep=3pt,topsep=4pt]
\item We establish near-minimax transition bounds for standard
importance-weighted asynchronous TD learning from a single behavior trajectory, with both constant
and decreasing step sizes. The decreasing-step guarantee holds
simultaneously at all times beyond a transient threshold.
The lower bounds match both the statistical and hitting-time terms
on the parameter classes in Theorem~\ref{theorem:hitting_minimax}.
\item We develop an analysis that separates stochastic fluctuations from
initialization error: a reverse-time Poisson argument yields the cubic
statistical term, and hitting-cost compensation yields an additive
initialization cost. The behavior state chain need only be finite and irreducible,
so periodic chains are allowed.
\item We formulate these estimates as a transfer theorem for
importance-weighted nonnegative affine prediction updates, covering successor features,
transition-dependent discounting, and finite prediction modes.
A perturbation estimate further gives a time-uniform guarantee for
joint online estimation of the return mean and variance.
\end{enumerate}

\subsection{Organization}
Sections~\ref{Section:setup} and \ref{Section:results} give the TD learning
model and main guarantees.
Section~\ref{Section:framework} states the affine prediction framework.
Sections~\ref{Section:reverse} and \ref{Section:initialization} develop its
stochastic and initialization estimates.
Section~\ref{Section:extensions} applies the framework to other prediction
algorithms and return-variance estimation.
Section~\ref{Section:lower_bounds} proves the matching lower bounds with their precise parameter ranges.
The appendices collect supporting lemmas and detailed proofs.

\section{Problem Formulation}
\label{Section:setup}
\subsection{Markov decision process and policies}

Let $\gS=[d]$ and $\mathcal A$ be the finite state and action spaces of a Markov decision process.
Fix a stationary behavior policy $b$ and a stationary target policy $\pi$.
The behavior policy induces a time-homogeneous process
$\{(S_t,A_t,R_t,S_{t+1})\}_{t\geq0}$.
More precisely, for the history
\[
\gF_t\coloneq\sigma(\bV_0,S_0,A_0,R_0,S_1,\ldots,A_{t-1},R_{t-1},S_t),
\]
we assume that, for every Borel set $B\subseteq[0,1]$, action $a\in\mathcal A$,
and state $x\in\gS$,
\begin{equation}\label{eq:data_model}
\PB(A_t=a,R_t\in B,S_{t+1}=x\mid\gF_t)
=b(a\mid S_t)\mathcal P(B,x\mid S_t,a)\quad\text{almost surely}.
\end{equation}
Here $\mathcal P$ is a reward-transition kernel from
$\gS\times\mathcal A$ to $[0,1]\times\gS$.
Thus the reward and the next state may be dependent given the current state and action.
The policies are known and satisfy action coverage, with a known ratio bound
$\varrho\geq1$:
\begin{equation}\label{eq:policy_coverage}
 \pi(a\mid s)>0\ \Longrightarrow\ b(a\mid s)>0,
 \qquad
 W_t=\rho_t\coloneq\frac{\pi(A_t\mid S_t)}{b(A_t\mid S_t)},
 \qquad 0\leq W_t\leq\varrho.
\end{equation}
Ratios on actions with $b(a\mid s)=0$ are set to zero.
One can take $\varrho=\rho_{\max}$, the largest ratio over the support
of $b$. In particular, $\E_b[W_t\mid\gF_t]=1$.
Write the behavior and target transition matrices as
\begin{equation}\label{eq:behavior_target_kernels}
 \bP=P_b,\qquad \bQ=P_\pi,
 \qquad
 Q(s,x)=\E_b[W_t\ind_{\{S_{t+1}=x\}}\mid S_t=s].
\end{equation}
Unless the subscript $\pi$ is displayed, trajectory expectations refer
to the behavior policy.

We impose the following conditions on the rewards and initialization.

\begin{assumption}[Rewards and initialization]\label{assumption:bounded}
The data satisfy \eqref{eq:data_model}--\eqref{eq:policy_coverage}.
The rewards satisfy $0\leq R_t\leq1$ almost surely.
The discount factor satisfies $\gamma\in(0,1)$, and $H\coloneq(1-\gamma)^{-1}$.
The initial estimate satisfies $\bm0\leq\bV_0\leq H\bone$ entrywise.
The initial state has an arbitrary distribution vector $\bnu$ on $\gS$.
\end{assumption}

Writing $r_\pi(s)\coloneq\E_b[W_tR_t\mid S_t=s]$, the target value function is
\begin{equation*}
V^\pi(s)\coloneq\E_s^\pi\brk{\sum_{k\geq0}\gamma^kR_k},
\qquad
\bV^\pi=\br_\pi+\gamma\bQ\bV^\pi.
\end{equation*}
Here $\E_s^\pi$ denotes expectation for the target-policy process started at state $s$, regardless of the actual initial distribution $\bnu$.
In particular, $\bm0\leq\bV^\pi\leq H\bone$.
Throughout, inequalities between vectors are entrywise, $\norm{\cdot}_\infty$ is the maximum absolute coordinate, and $\norm{\cdot}_{\infty\to\infty}$ is its induced matrix norm.

\subsection{State chain, hitting times, and mixing}
\label{subsection:state_chain}

To quantify state coverage and temporal dependence, we impose the following assumption on the behavior state chain.

\begin{assumption}[State chain]\label{assumption:mixing}
The behavior transition matrix $\bP=P_b$ is irreducible.
\end{assumption}

Its stationary distribution $\bmu=\bmu_b$ is unique and satisfies
$\mumin\coloneq\min_{s\in\gS}\mu(s)=\mu_{b,\min}>0$.
Throughout, $\mathfrak h_\star=\mathfrak h_{b,\star}$ and
$\tmix=t_{\mathrm{mix},b}$ also refer to the behavior state chain.
The chain may be periodic.
No irreducibility assumption is imposed on the target matrix $\bQ$.
Our bounds use the largest expected hitting time of the reverse chain.
Define
\begin{equation}\label{eq:main_hitting_parameter}
P^\star(x,y)\coloneq\frac{\mu(y)P(y,x)}{\mu(x)},
\qquad
\mathfrak h_\star\coloneq\max_{x,y\in\gS}\E_x^\star\tau_y,
\qquad
\tau_y\coloneq\inf\{n\geq0:X_n=y\},
\end{equation}
where $\E_x^\star$ denotes expectation for the chain with transition
matrix $\bP^\star$ started at $x$.
The hitting time is zero when the starting state is already the target.
Thus $\mathfrak h_\star$ measures the longest expected time needed to
reach a specified state in reverse time. It is finite by irreducibility
of the finite reverse chain.
The reverse chain arises because the propagation weights in TD learning
depend on future observations in forward time; reversing the trajectory makes
these weights depend only on observations already exposed.

For comparison with mixing-time bounds, define
\begin{equation}\label{eq:mixing_profile}
d_{\mathrm{TV}}(k)\coloneq\max_{s\in\gS}
\norm{\bP^k(s,\cdot)-\bmu}_{\TV},
\qquad
\tmix\coloneq\min\brc{k\geq1:d_{\mathrm{TV}}(k)\leq1/4},
\end{equation}
where $\norm{\lambda-\lambda'}_{\TV}
=\frac12\sum_s\abs{\lambda(s)-\lambda'(s)}$, and the minimum of the empty
set is $\infty$.
If $\bP$ is also aperiodic, it is uniformly geometrically ergodic and
$\tmix<\infty$. In this case, define the amplified mixing interval
\[
\tau\coloneq\tmix\left\lceil\log_2\frac{2}{\mumin}\right\rceil.
\]
Then Lemma~\ref{lemma:reverse_hitting_mixing} proves
\[
\mathfrak h_\star\leq\frac{2\tau}{\mumin}
=\wtilde O\prn{\frac{\tmix}{\mumin}}.
\]
Consequently, a hitting-time bound yields a mixing-time bound by this
comparison for aperiodic chains, while retaining more information about
the individual chain.

\begin{example}[Policy evaluation with seasonal demand]\label{example:seasonal}
Periodic transition and reward models arise naturally in operations
research, including inventory systems with seasonal demand
\citep{riis1965discounted,kapuscinski1998capacitated}.
Consider a finite-capacity inventory system with finitely many inventory
states and a fixed behavior replenishment policy that may depend on the season.
If the demand distribution follows an $m$-period cycle, with $m\geq2$, include the phase
in the state:
\[
S_t=(X_t,J_t),\qquad J_{t+1}=J_t+1\pmod m,
\qquad J_t\in\{0,\ldots,m-1\}.
\]
Writing $P_j^b$ for the inventory transition kernel in phase $j$, the
augmented process has the time-homogeneous transition kernel
\[
P\bigl((x,j),(y,k)\bigr)
=P_j^b(x,y)\mathbf1_{\{k=j+1\!\!\!\pmod m\}}.
\]
Assume that this finite chain is irreducible and that the one-step
costs are bounded and normalized to $[0,1]$; these costs play the role
of rewards in the TD learning recursion.
Every return time is a multiple of $m$, so the chain is periodic.
Its stationary mass in each phase is $1/m$, whereas a trajectory from
a fixed initial state occupies a single phase at each time. Thus
\[
\norm{\bP^t((x,j),\cdot)-\bmu}_{\TV}\geq1-\frac1m
\qquad\text{for all }t\geq0.
\]
Here $\tmix=\infty$, but $\mathfrak h_\star<\infty$.
Our hitting-time guarantees therefore apply directly to TD learning
along the observed trajectory of this seasonal system, including the
evaluation of other replenishment policies covered by the behavior policy.
\end{example}

\subsection{Asynchronous TD learning and error propagation}
For step sizes $0<\eta_t\leq1/\varrho$, importance-weighted asynchronous TD learning performs the update
\begin{equation}\label{eq:td_update_vector}
\bV_{t+1}=\bV_t+\eta_tW_t\be_{S_t}
\prn{R_t+\gamma V_t(S_{t+1})-V_t(S_t)},
\end{equation}
where $\be_s\in\RB^d$ is the standard basis vector with a one in coordinate
$s$ and zeros elsewhere, so only the currently visited coordinate changes.
Since $0\leq\eta_tW_t\leq1$ and $1+\gamma H=H$, induction in \eqref{eq:td_update_vector} shows that $\bm0\leq\bV_t\leq H\bone$ for all $t$.
The importance ratio multiplies the entire temporal-difference error.
When $b=\pi$, $W_t=1$ almost surely and the update reduces to standard on-policy TD learning.
Our main results use deterministic constant steps $\eta_t\equiv\eta$ and a
nonincreasing schedule with eventual rate $\log(t)/t$, independent of the
target precision, confidence level, and terminal time.
Appendix~\ref{Appendix:unknown_mu} provides a data-driven step-size rule
that uses only $H$, $d$, the known ratio bound $\varrho$, and observed state counts, and requires no
prior knowledge of unknown model parameters.

\paragraph{Error propagation.}
Define the estimation error and the weighted Bellman innovation by
\begin{equation*}
\bDelta_t\coloneq\bV_t-\bV^\pi,
\qquad
\beta(y,w,r,x)\coloneq w\prn{r+\gamma V^\pi(x)-V^\pi(y)}.
\end{equation*}
Here the weighted Bellman innovation is the importance-weighted temporal-difference error evaluated at
the true value function $\bV^\pi$.
The target Bellman equation implies $\E_b[\beta(S_t,W_t,R_t,S_{t+1})\mid\gF_t]=0$.
For each source state $y$, this innovation lies in the interval $\varrho[-V^\pi(y),H-V^\pi(y)]$, of width $\varrho H$.
Set
\begin{equation}\label{eq:random_matrix}
\widehat{\bB}_t\coloneq\bI-\eta_tW_t\be_{S_t}\be_{S_t}^{\top}
                +\eta_tW_t\gamma\be_{S_t}\be_{S_{t+1}}^{\top},
\end{equation}
and, for $0\leq s<t$, define
\begin{equation}\label{eq:random_product}
\widehat{\bPhi}_{t,s}\coloneq\widehat{\bB}_{t-1}\cdots\widehat{\bB}_s,
\qquad \widehat{\bPhi}_{s,s}\coloneq\bI.
\end{equation}
Every matrix in \eqref{eq:random_matrix} is nonnegative with row sums at most one.
Subtracting the true value from \eqref{eq:td_update_vector} gives the exact identities
\begin{align}
\bDelta_{t+1}
&=\widehat{\bB}_t\bDelta_t
 +\eta_t\be_{S_t}\beta(S_t,W_t,R_t,S_{t+1}),\notag\\
\bDelta_T&=\widehat{\bPhi}_{T,0}\bDelta_0+\bY_T,\label{eq:error_decomposition}\\
\bY_T&\coloneq\sum_{t=0}^{T-1}\eta_t
\widehat{\bPhi}_{T,t+1}\be_{S_t}\beta(S_t,W_t,R_t,S_{t+1}).\label{eq:stochastic_convolution}
\end{align}
The first term in \eqref{eq:error_decomposition} is the remaining initial error.
The second accumulates the transition and reward noise through the same random matrices that propagate that error.
Although each weighted Bellman innovation is centered in forward time, its coefficient in \eqref{eq:stochastic_convolution} depends on future observations.
This dependence is the reason for the reverse-time analysis below, which
retains the importance weight as part of the edge mark and uses the
reverse behavior kernel $\bP^\star$.


\section{Main Results}
\label{Section:results}
We give explicit guarantees for asynchronous TD learning with constant
and decreasing step sizes.
Both guarantees separate a cubic statistical term from an additive reverse hitting-time cost.

\subsection{Constant step size}\label{subsection:constant_steps}

\begin{theorem}[Last-iterate off-policy TD learning with a constant step size]\label{theorem:main}
Under Assumptions~\ref{assumption:bounded} and \ref{assumption:mixing}, let $0<\varepsilon\leq1$ and $0<\delta<1$.
Define
\begin{equation*}
L\coloneq\log\frac{8d}{\delta\mumin},
\qquad
\eta\coloneq\frac{\varepsilon^2}{1024\varrho H^2L^2},
\end{equation*}
and take the integers
\begin{equation}\label{eq:main_hitting_budget}
b_{\mathrm{blk}}\coloneq\left\lceil8\prn{\frac{H}{\eta\mumin}+\mathfrak h_\star}\right\rceil,
\qquad
k\coloneq\left\lceil\log_2\frac{4dH}{\delta\varepsilon\mumin}\right\rceil,
\qquad T_\star\coloneq b_{\mathrm{blk}}k.
\end{equation}
Here $b_{\mathrm{blk}}$ and $k$ are the length and number of blocks used to control
initialization error; this block decomposition is used only in the proof and is not part of the TD learning algorithm.

For every integer terminal time $T\geq T_\star$ fixed in advance, the iterates of \eqref{eq:td_update_vector} with $\eta_t\equiv\eta$ satisfy
\begin{equation}\label{eq:main_guarantee}
\PB_{\bnu}\prn{\norm{\bV_T-\bV^\pi}_\infty\leq\varepsilon}\geq1-\delta.
\end{equation}
Consequently, a sufficient total transition count is
\begin{equation}\label{eq:hitting_sample_complexity}
\wtilde O\prn{\frac{\varrho H^3}{\mumin\varepsilon^2}+\mathfrak h_\star}.
\end{equation}
\end{theorem}

The two terms correspond to stochastic fluctuations and the time needed to remove the initial error, respectively.
The comparison in Section~\ref{subsection:state_chain} yields the
following mixing-time consequence.

\begin{corollary}[Guarantee in terms of mixing time]\label{corollary:mixing_time}
Under the conditions of Theorem~\ref{theorem:main}, assume in addition
that $\bP$ is aperiodic, and let
\[
b_{\mathrm{mix}}\coloneq
\left\lceil\frac8{\mumin}\prn{\frac H\eta+2\tau}\right\rceil.
\]
The guarantee \eqref{eq:main_guarantee} holds for every integer terminal
time $T\geq b_{\mathrm{mix}}k$ fixed in advance.
Consequently, a sufficient total transition count is
\begin{equation}\label{eq:sample_complexity}
\wtilde O\prn{\frac{\varrho H^3}{\mumin\varepsilon^2}
                  +\frac{\tmix}{\mumin}}.
\end{equation}
\end{corollary}

In the on-policy case $b=\pi$ and $\varrho=1$, relative to the one-action specialization of
\citet[Theorem~4]{li2024q}, \eqref{eq:sample_complexity} improves the
statistical dependence from $H^4$ to $H^3$ and removes a factor of $H$
from the mixing transient.
The hitting-time formulation can give a substantially smaller bound
than the mixing-time formulation for some chains.
\begin{example}[Fast hitting with slow mixing]\label{example:hitting_vs_mixing}
Consider the two-state behavior chain
\[
\bP_\xi\coloneq
\begin{pmatrix}
\xi&1-\xi\\
1-\xi&\xi
\end{pmatrix},
\qquad 0<\xi<1/2.
\]
This chain is reversible with stationary distribution $(1/2,1/2)$.
The time to reach the other state is geometric with success probability
$1-\xi$, so
\[
\mumin=\frac12,
\qquad
\mathfrak h_\star=\frac1{1-\xi}\leq2.
\]
In contrast, its total variation mixing profile and mixing time are
\[
d_{\mathrm{TV}}(t)=\frac12(1-2\xi)^t,
\qquad
\tmix=
\left\lceil\frac{\log2}{-\log(1-2\xi)}\right\rceil,
\]
and hence $\tmix\to\infty$ as $\xi\downarrow0$.
The chain reaches either state quickly, even though its nearly
periodic behavior delays convergence to stationarity.
For fixed $H$, $\varrho$, $\varepsilon$, and $\delta$, the sufficient transition
count in Theorem~\ref{theorem:main} therefore remains
bounded as $\xi\downarrow0$, whereas the mixing-time bound in
Corollary~\ref{corollary:mixing_time} diverges.
At $\xi=0$, the chain alternates deterministically between the two states;
Theorem~\ref{theorem:main} still applies with $\mathfrak h_\star=1$.
\end{example}

\subsection{Step sizes decreasing at rate \texorpdfstring{$\log(t)/t$}{log(t)/t}}

The constant step size above is chosen for a prescribed accuracy.
The next theorem uses a single decreasing schedule that depends only on
$H$, $d$, $\varrho$, and $\mumin$.

\begin{theorem}[Time-uniform off-policy TD learning with $\log(t)/t$ step sizes]\label{theorem:decay}
Under Assumptions~\ref{assumption:bounded} and \ref{assumption:mixing},
define, for every $t\geq0$,
\begin{equation}\label{eq:decay_schedule}
\overline L_t\coloneq\log\frac{16d\varrho H(t+8)}{\mumin},
\qquad
\eta_t\coloneq\min\brc{\frac1{\varrho H},
                  \frac{4096H\overline L_t}{\mumin(t+8)}}.
\end{equation}
There exist universal constants $C_0,C>0$ such that the following holds.
For any $0<\delta<1$, define 
$L_t\coloneq\overline L_t+\log(1/\delta)$.
Define
\begin{equation}\label{eq:decay_transient}
T_\star\coloneq
\min\left\{
t\in\mathbb N:
t\geq
\max\left\{
8,\,
C_0L_t\prn{\frac{\varrho H^2}{\mumin}+\mathfrak h_\star}
\right\}
\right\}.
\end{equation}
For every initial distribution $\bnu$, the event
\begin{equation}\label{eq:decay_rate}
\mathcal E_\delta\coloneq
\left\{\begin{gathered}
\forall T\in\mathbb N,\ T\geq T_\star:\\
\norm{\bV_T-\bV^\pi}_\infty
\leq\frac{H}{T+8}
 +896\sqrt2\sqrt{\frac{\varrho H^3\overline L_T L_T^2}{\mumin(T+8)}}
 +\frac{16384\varrho H^3\overline L_T}{\mumin(T+8)}
\end{gathered}\right\}
\end{equation}
satisfies $\PB_{\bnu}(\mathcal E_\delta)\geq1-\delta$.
Moreover, on $\mathcal E_\delta$, for every
$0<\varepsilon\leq1$ and every integer $T\geq8$ satisfying
\begin{equation}\label{eq:decay_sample}
T\geq C\prn{
\frac{\varrho H^3\overline L_T L_T^2}{\mumin\varepsilon^2}
+\mathfrak h_\star L_T},
\end{equation}
we have $\norm{\bV_T-\bV^\pi}_\infty\leq\varepsilon$.
Consequently, with probability at least $1-\delta$, all iterates
after a sufficient transition count of
\[
\wtilde O\prn{
\frac{\varrho H^3}{\mumin\varepsilon^2}+\mathfrak h_\star}
\]
have sup-norm error at most $\varepsilon$.
\end{theorem}

The schedule is nonincreasing and eventually has order
$H\log(t)/ (\mumin t)$, and its guarantee matches the sample complexity of the constant-step guarantee up to logarithmic factors.
For aperiodic chains, the comparison in Section~\ref{subsection:state_chain} also yields the
mixing-time bound \eqref{eq:sample_complexity}.
The event $\mathcal E_\delta$ guarantees that the error bound holds
simultaneously for all $T\geq T_\star$, whereas
Theorem~\ref{theorem:main} gives a guarantee for each terminal time
fixed in advance.
The proof in Appendix~\ref{Appendix:decay} treats general deterministic
nonincreasing step sizes on a fixed interval, then applies those bounds
to the last half of the trajectory.

\begin{remark}[Unknown minimum stationary probability]\label{remark:unknown_mu}
When $\mumin$ is unknown, the empirical state frequencies can be used to
choose a step size at the beginning of each doubling epoch and keep it
fixed within that epoch. Theorem~\ref{theorem:unknown_mu} in
Appendix~\ref{Appendix:unknown_mu} gives the algorithm and proves the
$\wtilde O(\varrho H^3/(\mumin\varepsilon^2)+\mathfrak h_\star)$ transition complexity
of Theorem~\ref{theorem:decay},
with a time-uniform guarantee for its last iterate.
The rule uses only $H$, $d$, $\varrho$, and observed state counts, and its guarantee
also allows periodic chains.
Adaptive learning rates based on empirical visitation frequencies also
appear in \citet[Section~III-D and Theorem~3]{li2022async}.
\end{remark}

Theorem~\ref{theorem:hitting_minimax} shows that the statistical and
hitting-time terms are jointly minimax optimal, up to logarithmic factors,
on the parameter classes specified there.
Corollary~\ref{corollary:budget_minimax} gives the corresponding result
under a mixing-time budget.

\section{A Framework for Affine Prediction}
\label{Section:framework}
\label{Section:overview}
The proof separates the propagation generated by the observed marked
trajectory from the prediction being learned. We formulate this separation
for a class of importance-weighted block updates that includes off-policy TD learning, and then establish
the two estimates used by all three step-size rules.

\subsection{Affine prediction along a Markov trajectory}

Let the state chain satisfy Assumption~\ref{assumption:mixing}. On each
transition, observe a mark $\Omega_t=(W_t,\bm C_t,\bm A_t)$, where
$\bm C_t\in\RB^m$ and $\bm A_t\in\RB^{m\times m}$. Given the complete
history through $S_t$, the pair $(\Omega_t,S_{t+1})$ has a fixed
time-homogeneous kernel depending only on $S_t$. Its state marginal is
$\bP$. The mark and next state may be dependent. The weight satisfies
\begin{equation}\label{eq:affine_weights}
 0\leq W_t\leq\varrho,\qquad
 \E_b[W_t\mid S_t=s]=1,\qquad \varrho\geq1.
\end{equation}
Thus weighting the conditional marked transition law by $W_t$ defines
a probability kernel, whose physical-state marginal is
$Q(s,x)=\E_b[W_t\ind_{\{S_{t+1}=x\}}\mid S_t=s]$.
We refer to this as the target marked kernel; its state marginal $\bQ$ need not be irreducible.
The physical trajectory is always generated by the behavior kernel $\bP$.
Assume that almost surely
\begin{equation}\label{eq:affine_primitives}
 \bm0\leq\bm C_t\leq\bone,
 \qquad \bm A_t\geq\bm0,
 \qquad \bm A_t\bone\leq\gamma\bone.
\end{equation}
The same $\gamma\in(0,1)$ and $H=(1-\gamma)^{-1}$ are used throughout.
For a table $\bm Z:\gS\to\RB^m$, define
\[
 (\mathcal T\bm Z)(s)
 =\E_b[W_t(\bm C_t+\bm A_t\bm Z(S_{t+1}))\mid S_t=s],
 \qquad
 \norm{\bm Z}_{\max}=\max_{s\in\gS,\,1\leq k\leq m}|Z_k(s)|.
\]
Conditions~\eqref{eq:affine_weights}--\eqref{eq:affine_primitives} make $\mathcal T$ a
$\gamma$-contraction in this norm that preserves $[0,H]^{d\times m}$.
It therefore has a unique fixed point $\bm Z^\star$ in this box.
Starting from any $\bm Z_0\in[0,H]^{d\times m}$, update
\begin{equation}\label{eq:affine_update}
 \bm Z_{t+1}(s)=\bm Z_t(s)+\eta_tW_t\ind_{\{S_t=s\}}
 \bigl(\bm C_t+\bm A_t\bm Z_t(S_{t+1})-\bm Z_t(s)\bigr).
\end{equation}
All $m$ components at the observed state are updated. If $0\leq\eta_t\varrho\leq1$,
the iterates remain in the same box. Off-policy TD learning is the case
$m=1$, $C_t=R_t$, $\bm A_t=[\gamma]$, and
$W_t=\pi(A_t\mid S_t)/b(A_t\mid S_t)$, where the nonbold symbol
$A_t$ denotes the sampled action.

\begin{theorem}[Transfer to affine block prediction]
\label{theorem:affine_transfer}
Under the marked-kernel model and
\eqref{eq:affine_weights}--\eqref{eq:affine_primitives}, the
following guarantees hold for \eqref{eq:affine_update}, under every
initial distribution of the physical state chain. Fix $0<\delta<1$;
the precision parameters below satisfy $0<\varepsilon\leq1$.
\begin{enumerate}[label=(\roman*),leftmargin=*]
\item The constant-step guarantee of Theorem~\ref{theorem:main} holds
with $\norm{\bm Z_T-\bm Z^\star}_{\max}$ in place of
$\norm{\bV_T-\bV^\pi}_\infty$ and with $d$ replaced by $dm$ in its
logarithms. In particular, one may use
\begin{align*}
 L&=\log\frac{8dm}{\delta\mumin},\qquad
 \eta=\frac{\varepsilon^2}{1024\varrho H^2L^2},\\
 b_{\mathrm{blk}}&=\left\lceil8\left(\frac H{\eta\mumin}+\mathfrak h_\star\right)\right\rceil,
 \qquad
 k=\left\lceil\log_2\frac{4dmH}{\delta\varepsilon\mumin}\right\rceil.
\end{align*}
For each deterministic $T\geq b_{\mathrm{blk}}k$, the error is at most
$\varepsilon$ with probability at least $1-\delta$.
\item The time-uniform guarantee of Theorem~\ref{theorem:decay} holds
with the same error replacement and $d$ replaced by $dm$ in
$\overline L_t$, hence also in its step size and $L_t$.
The constants, transient threshold, and error bound are unchanged.
\item The empirical-occupancy rule and time-uniform guarantee of
Theorem~\ref{theorem:unknown_mu} hold with the same error replacement
and $d$ replaced by $dm$ in $\ell_t$ and $\mathcal L_t$.
The counts $N_n(x)$ and their minimum are still computed over the
$d$ physical states.
\end{enumerate}
For each rule, the sufficient transition count for
$0<\varepsilon\leq1$ is
\begin{equation}\label{eq:affine_sample_complexity}
 \wtilde O\left(\frac{\varrho H^3}{\mumin\varepsilon^2}
                  +\mathfrak h_\star\right).
\end{equation}
Here $\mumin$ and $\mathfrak h_\star$ refer to the physical state chain;
the number $m$ of prediction components enters only logarithmically.
\end{theorem}

The proof is given after the interval estimates below. The class permits
random transport matrices and dependence among all components of an edge
mark. Section~\ref{Section:extensions} gives applications in which these
matrices couple several predictions.
Since scalar TD learning is a subclass, the lower bounds in
Section~\ref{Section:lower_bounds} also apply to this full affine class.

\subsection{Noise and propagation as separate inputs}

View a table as a vector indexed by $(s,k)$. Let $\mathcal B_t$ be the
linear map that leaves every row except $S_t$ unchanged and maps the
updated row to
$(1-\eta_tW_t)\bm z(S_t)+\eta_tW_t\bm A_t\bm z(S_{t+1})$.
Write
\[
 \mathcal\Phi_{T,s}=\mathcal B_{T-1}\cdots\mathcal B_s,
 \qquad \mathcal\Phi_{s,s}=\bI,
 \qquad
 \bm h_t=W_t(\bm C_t+\bm A_t\bm Z^\star(S_{t+1})-\bm Z^\star(S_t)).
\]
For a vector $v\in\RB^m$, write $\iota_s v$ for the table that equals
$v$ at row $s$ and is zero elsewhere. Unrolling the error gives
\begin{equation}\label{eq:affine_decomposition}
 \bm Z_T-\bm Z^\star
 =\mathcal\Phi_{T,s}(\bm Z_s-\bm Z^\star)+\mathcal Y_{T,s},
 \qquad
 \mathcal Y_{T,s}=\sum_{t=s}^{T-1}\eta_t
       \mathcal\Phi_{T,t+1}\iota_{S_t}\bm h_t.
\end{equation}
The first term measures memory of the initial table on this interval.
The second is a stochastic convolution whose propagation matrices depend on future
observations. Reading the interval backward makes these matrices predictable.
An anchored local Poisson equation corrects the reverse drift; a variance
certificate then controls the accumulated martingale variance.

\begin{proposition}[Domination by scalar propagation]
\label{prop:affine_domination}
For every realized trajectory and every step sequence in $[0,1/\varrho]$,
\begin{equation}\label{eq:affine_domination}
 \norm{\mathcal\Phi_{T,s}}_{\max\to\max}
 \leq\norm{\widehat{\bPhi}_{T,s}}_{\infty\to\infty},
\end{equation}
where the scalar propagator on the right uses the same state trajectory
and importance weights, steps, and discount $\gamma$, as in \eqref{eq:random_matrix}.
\end{proposition}
\begin{proof}
Both propagators are nonnegative, so their induced maximum norms equal
their largest row sums. Start the block recursion from the all-ones table
and the scalar recursion from $\bone$. If every block component at state
$x$ is at most the scalar value $w(x)$, then after an update at $s$,
\[
 (1-\eta W)z_k(s)+\eta W\sum_\ell A_{k\ell}z_\ell(s')
 \leq(1-\eta W)w(s)+\eta W\gamma w(s').
\]
Unchanged rows retain the comparison. Induction proves the comparison of
all row sums and hence \eqref{eq:affine_domination}.
\end{proof}

Thus the hitting-time analysis in Section~\ref{Section:initialization}
applies to every block update through the physical chain. The remaining
task is to prove a stochastic estimate with the same statistical scale as
TD learning. The next section does so for deterministic nonincreasing
steps on an arbitrary interval, which also permits conditioning at the
boundaries of the empirical-occupancy epochs.

\section{Reverse-Time Analysis}
\label{Section:reverse}
\label{Section:concentration}
\subsection{Reverse transport and its mass budget}

Fix an interval $[s,T)$ of length $N=T-s$, first with a stationary
initial state. Write $\mathsf K(\mathrm d\omega\mid y,x)$ for the
conditional law of its edge mark given a forward transition $y\to x$.
Set
\begin{equation}\label{eq:reverse_path}
 X_n=S_{T-n},\qquad \omega_{n+1}=\Omega_{T-n-1},\qquad
 a_n=\eta_{T-n-1},\qquad
 \bp_n=\mathcal\Phi_{T,T-n}^{\top}\be_{(j,k)}.
\end{equation}
The last expression selects a terminal state-component pair $(j,k)$.
The weight component of $\omega_{n+1}$ is written $w_{n+1}$.
In the scalar TD specialization,
$\bp_n=\widehat{\bPhi}_{T,T-n}^{\top}\be_j$.
Use the reverse filtration
\begin{equation}\label{eq:reverse_filtration}
 \gG_n=\sigma(X_0,\omega_1,X_1,\ldots,\omega_n,X_n).
\end{equation}
Stationary path factorization, exactly as in
Lemma~\ref{lemma:reverse_kernel}, gives
\begin{equation}\label{eq:reverse_reward_law}
 \PB(X_{n+1}=y,\omega_{n+1}\in\mathrm d\omega\mid\gG_n)
 =P^\star(X_n,y)\mathsf K(\mathrm d\omega\mid y,X_n).
\end{equation}
Indeed, the factor $\mu(S_s)$ in the forward path measure becomes
$\mu(S_T)$ after replacing every $\mu(y)P(y,x)$ by
$\mu(x)P^\star(x,y)$; the conditional mark factors remain unchanged.
This argument applies to joint continuous or discrete marks.

For $x=X_n$, $y=X_{n+1}$, and the components $w=w_{n+1}$ and $A$ of $\omega_{n+1}$,
the transpose update is
\begin{equation}\label{eq:block_reverse_update}
 p_{n+1}(i,\ell)=p_n(i,\ell)-a_nwp_n(y,\ell)\ind_{\{i=y\}}
       +a_nw\ind_{\{i=x\}}\sum_kp_n(y,k)A_{k\ell}.
\end{equation}
For scalar TD learning this reduces to
\begin{equation*}
 \bp_{n+1}=\bp_n+a_nwp_n(y)(\gamma\be_x-\be_y).
\end{equation*}
We also use the abstract reverse process defined by
\eqref{eq:reverse_reward_law} and \eqref{eq:block_reverse_update},
started from any state $x$ and any nonnegative vector $\bp_0$ with
$q_0=\norm{\bp_0}_1\leq1$. Its expectation is denoted by
$\E_{\bp_0,x}$. In this section the steps satisfy
\begin{equation}\label{eq:framework_steps}
 0<a_0\leq\cdots\leq a_{N-1}\leq\overline a\leq(\varrho H)^{-1}.
\end{equation}
They are deterministic; forward nonincreasing steps have this order
after reversal.

\begin{lemma}[Reverse transport budget]\label{lemma:reverse_weights}
The propagation vector $\bp_n$ is $\gG_n$-measurable and nonnegative. With
$q_n=\norm{\bp_n}_1$, it satisfies
\begin{equation}\label{eq:reverse_mass_budget}
 q_{n+1}\leq q_n\leq q_0,
 \qquad
 \sum_{n<N}a_nw_{n+1}\sum_kp_n(X_{n+1},k)\leq H(q_0-q_N)\leq Hq_0.
\end{equation}
\end{lemma}
\begin{proof}
The update removes at most the mass at each $(y,k)$ and adds
nonnegative mass, so it preserves nonnegativity, also when $x=y$.
Its total loss is
$a_nw\sum_kp_n(y,k)(1-\sum_\ell A_{k\ell})$, which is at least
$a_nw(1-\gamma)\sum_kp_n(y,k)$. Summation proves
\eqref{eq:reverse_mass_budget}. Predictability follows by induction.
\end{proof}

\subsection{An anchored correction for centered edge observables}

For now let $h_k(i,\omega,x)$ be any observable satisfying, for each
source-component pair $(i,k)$,
\begin{equation}\label{eq:source_centering}
 \sum_xP(i,x)\int h_k(i,\omega,x)
                \mathsf K(\mathrm d\omega\mid i,x)=0.
\end{equation}
Assume its values lie almost surely in an interval satisfying
\begin{equation}\label{eq:source_range}
 h_k(i,\omega,x)\in[a_{i,k},b_{i,k}],\qquad
 a_{i,k}\leq0\leq b_{i,k},\qquad b_{i,k}-a_{i,k}\leq b_{\mathrm{osc}}.
\end{equation}
Scalar subscripts omit $k$ when $m=1$. Put
$\bar h_{i,k}(x)=\int h_k(i,\omega,x)\mathsf K(\mathrm d\omega\mid i,x)$
on edges with positive probability, and set it to zero on other edges.

Hitting-time representations of Poisson solutions are used in Markovian
stochastic approximation \citep{chandak2022concentration}; for a generic
centered function, their bounds can involve the expected hitting time.
Here the source is supported on reverse transitions entering a single
physical state. The following correction therefore stays within the
range of a single source-component innovation.

\begin{lemma}[Anchored local Poisson equation]\label{lemma:local_poisson}
There are functions $U_{i,k}:\gS\to\RB$ such that
\begin{equation*}
 (\bI-\bP^\star)U_{i,k}(x)=P^\star(x,i)\bar h_{i,k}(x),
 \qquad U_{i,k}(i)=0,
 \qquad U_{i,k}(x)\in[a_{i,k},b_{i,k}].
\end{equation*}
\end{lemma}
\begin{proof}
For $x\ne i$, let $\tau_i^\star=\inf\{n\geq1:X_n=i\}$ and define
$U_{i,k}(x)=\E_x^\star[\bar h_{i,k}(X_{\tau_i^\star-1})]$.
Finite irreducibility makes this hitting time finite almost surely.
Set $U_{i,k}(i)=0$. The range claim follows immediately.
Conditioning on the first reverse transition proves the Poisson equation
at every $x\ne i$. Its right side has stationary mean
\[
 \sum_x\mu(x)P^\star(x,i)\bar h_{i,k}(x)
 =\mu(i)\sum_xP(i,x)\bar h_{i,k}(x)=0.
\]
The left side also has stationary mean zero. Since the two sides
already agree off $i$ and $\mu(i)>0$, they agree at $i$ as well.
\end{proof}

Write $\bU(x)=(U_{i,k}(x))_{i,k}$ and define
\begin{equation}\label{eq:reverse_martingale_difference}
 m_{n+1}(i,k)=\ind_{\{y=i\}}h_k(y,\omega_{n+1},x)
                   -U_{i,k}(x)+U_{i,k}(y),
 \qquad x=X_n,\quad y=X_{n+1}.
\end{equation}
Also put
\[
 Y_N(h)=\sum_{n<N}a_n\sum_kp_n(y,k)h_k(y,\omega_{n+1},x),
 \qquad G_n=\bp_n^\top\bU(X_n).
\]

\begin{lemma}[Generic reverse martingale correction]
\label{lemma:generic_corrector}
Under \eqref{eq:source_centering}--\eqref{eq:source_range},
$\E[\bm m_{n+1}\mid\gG_n]=\bm0$ and
$\norm{\bm m_{n+1}}_\infty\leq b_{\mathrm{osc}}$. Moreover,
\begin{align}
 Y_N(h)-\sum_{n<N}a_n\bp_n^\top\bm m_{n+1}
 ={}&a_0G_0-a_{N-1}G_N
       +\sum_{n=1}^{N-1}(a_n-a_{n-1})G_n\notag\\
 &+\sum_{n<N}a_n^2w_{n+1}\sum_{k,\ell}
             p_n(y,k)A_{k\ell}U_{x,\ell}(y).
 \label{eq:generic_corrector_identity}
\end{align}
The absolute value of this correction, and hence $|\E Y_N(h)|$,
is at most $2\overline aHb_{\mathrm{osc}}q_0$.
\end{lemma}
\begin{proof}
The conditional mean of the indicator term in
\eqref{eq:reverse_martingale_difference} is
$P^\star(x,i)\bar h_{i,k}(x)$; the Poisson equation identifies it
with the conditional mean of $U_{i,k}(x)-U_{i,k}(y)$.
If $y=i$, anchoring makes the increment equal to
$h_k(i,\omega,x)-U_{i,k}(x)$. Otherwise it equals
$U_{i,k}(y)-U_{i,k}(x)$. The two values being subtracted belong to
the same interval of width $b_{\mathrm{osc}}$ in either case.

The weight update and the anchors $U_{y,k}(y)=0$ give
\[
 \bp_n^\top(\bU(x)-\bU(y))
 =G_n-G_{n+1}
   +a_nw\sum_{k,\ell}p_n(y,k)A_{k\ell}U_{x,\ell}(y).
\]
Multiply by $a_n$ and sum by parts to obtain
\eqref{eq:generic_corrector_identity}. Since $|G_n|\leq b_{\mathrm{osc}}q_0$,
its boundary and step-variation terms are bounded by
$2\overline a b_{\mathrm{osc}}q_0$. Its transport term is at most
$\overline a\gamma b_{\mathrm{osc}}\sum_na_nw_{n+1}\sum_kp_n(y,k)
\leq\overline a\gamma Hb_{\mathrm{osc}}q_0$ in absolute value. Finally,
$2+\gamma H=H+1\leq2H$. The martingale sum has mean zero,
which gives the expectation bound.
\end{proof}

\subsection{A variance certificate from auxiliary routing}

Let
$\bm g(i,\omega,x)=\bm C+\bm A\bm Z^\star(x)-\bm Z^\star(i)$
and $\bm h=W\bm g$ be the unweighted and weighted fixed-point
innovations. The centering of $h$ follows from the weighted fixed-point
equation. Since each target lies in $[0,H]$, the $k$th coordinate of $h$
belongs to $[-\varrho Z_k^\star(i),\varrho(H-Z_k^\star(i))]$,
so $b_{\mathrm{osc}}=\varrho H$.

\begin{lemma}[Variance certificate for weighted affine prediction]
\label{lemma:affine_variance_certificate}
There is a table $u_0:\gS\to[0,H^2]^m$ satisfying
\[
 u_0(i)\geq\E_b[W\{\bm g(i,\Omega,S')^{\odot2}
                           +\gamma\bm A u_0(S')\}\mid S=i].
\]
Consequently $u=\varrho u_0$ satisfies $0\leq u\leq\varrho H^2\bone$
and the componentwise weighted certificate
\begin{equation}\label{eq:affine_variance_certificate}
 \E_b[\bm h(i,\Omega,S')^{\odot2}
       +W\gamma\bm A u(S')-Wu(i)\mid S=i]\leq\bm0.
\end{equation}
\end{lemma}
\begin{proof}
Run the target marked kernel obtained by multiplying the behavior
marked kernel by $W$. Introduce an auxiliary index $K$ taking values
in $\{1,\ldots,m\}$ and a cemetery state. Given physical state $i$,
index $k$, and the next mark and state $(W,C,A,x)$, receive reward
$C_k$ and select the next index $\ell$ with probability $A_{k\ell}/\gamma$.
Select the cemetery state with the remaining probability. This is a
probability rule by \eqref{eq:affine_primitives}. After the cemetery
state is reached all subsequent rewards are zero. Use fresh routing
randomness at every step, independent of the future marked process
given its current state, and discount every reward by $\gamma$ per step.
The resulting return lies in $[0,H]$.

Let its conditional mean and variance be $z(i,k)$ and $u_0(i,k)$;
give both value zero at the cemetery state. The first-step mean equation
is $z(i)=\E_b[W(C+Az(S'))\mid i]$, so uniqueness implies $z=\bm Z^\star$.
The variance decomposition for its first step gives
\[
 u_0(i,k)=\E_{\rm tar}\!\left[
 \bigl(C_k+\gamma Z_{K'}^\star(S')-Z_k^\star(i)\bigr)^2
       +\gamma^2u_0(S',K')\,\middle|\,i,k\right].
\]
Here $\E_{\rm tar}$ denotes expectation under the target marked kernel
and the auxiliary routing. The cross term with the centered future
return is zero after conditioning on the first marked transition and
$K'$. Conditional on that transition, the mean of the squared-expression's
inner quantity is $g_k$, and
$\gamma^2\E_{\rm tar}[u_0(S',K')\mid W,C,A,S',k]=\gamma(Au_0(S'))_k$.
Its remaining conditional variance is nonnegative. This proves
the first inequality; boundedness of the return gives
$0\leq u_0\leq H^2\bone$. Finally $W^2\leq\varrho W$ and
$\E_b[W\mid i]=1$ imply
\[
 \E_b[h^{\odot2}+W\gamma A(\varrho u_0)(S')
                      -W\varrho u_0(i)\mid i]
 \leq\varrho\{\E_b[W(g^{\odot2}+\gamma Au_0(S'))\mid i]-u_0(i)\}
 \leq\bm0.
\]
\end{proof}

\begin{lemma}[Scalar return variance identity]\label{lemma:return_variance}
For scalar TD learning, let $u_0(s)$ be the variance of the discounted
return generated by the target policy from $s$, and define
\[
 \sigma_0^2(s)=\E_b[W_t\{R_t+\gamma V^\pi(S_{t+1})-V^\pi(s)\}^2
                                      \mid S_t=s].
\]
Then
\begin{equation}\label{eq:return_variance_bellman}
 \bm0\leq\bu_0\leq H^2\bone,
 \qquad \bu_0=\bm\sigma_0^2+\gamma^2\bQ\bu_0.
\end{equation}
The weighted innovation has certificate \eqref{eq:affine_variance_certificate}
with $u=\varrho u_0$.
\end{lemma}
\begin{proof}
In the preceding routing construction $m=1$ and $A=\gamma$, so
the auxiliary index never changes or terminates. The extra routing
variance is zero. The first-step variance equation is consequently an
equality, namely \eqref{eq:return_variance_bellman}.
\end{proof}

\subsection{Weighted energy and reverse variance compensation}

The next argument permits a general source-centered observable $h$
with range width $b_{\mathrm{osc}}$ and a deterministic certificate
\begin{equation}\label{eq:generic_variance_certificate}
 \begin{gathered}
 \bm0\leq u(i)\leq\mathcal V\bone,\qquad b_{\mathrm{osc}}^2\leq\varrho\mathcal V,\\
 \E_b[h(i,\Omega,S')^{\odot2}+W\gamma A u(S')-Wu(i)\mid i]
 \leq\bm0.
 \end{gathered}
\end{equation}
Lemma~\ref{lemma:affine_variance_certificate} provides this input with
$b_{\mathrm{osc}}=\varrho H$ and $\mathcal V=\varrho H^2$ for the block prediction problem.

\begin{lemma}[Weighted innovation energy]\label{lemma:bellman_energy}
Under \eqref{eq:framework_steps} and
\eqref{eq:generic_variance_certificate}, uniformly over initial pairs
$(\bp_0,x)$ and interval lengths,
\begin{equation}\label{eq:framework_raw_energy}
 \E_{\bp_0,x}\sum_{n<N}a_n\sum_kp_n(y,k)
              h_k(y,\omega_{n+1},x)^2
 \leq8\mathcal Vq_0.
\end{equation}
\end{lemma}
\begin{proof}
Put $\psi_k(y,\omega,x)=h_k(y,\omega,x)^2
+w\gamma(Au(x))_k-wu_k(y)$. Its source mean $r_k(y)$ is
nonpositive by \eqref{eq:generic_variance_certificate}. Since
$|h_k|\leq b_{\mathrm{osc}}$, $0\leq w\gamma(Au(x))_k\leq\varrho\gamma^2\mathcal V$,
and $-\varrho\mathcal V\leq-wu_k(y)\leq0$, the centered observable
$\psi_k^\circ=\psi_k-r_k(y)$ has sourcewise range width at most
\[
 b_{\mathrm{osc}}^2+\varrho(1+\gamma^2)\mathcal V\leq3\varrho\mathcal V.
\]
Nonnegativity of the weights implies $Y_N(\psi)\leq Y_N(\psi^\circ)$
pathwise. Taking the inner product of \eqref{eq:block_reverse_update}
with $u$ yields the exact identity
\begin{align*}
 \sum_{n<N}a_n\sum_kp_n(y,k)h_k^2
 ={}&\bp_0^\top u-\bp_N^\top u+Y_N(\psi)\\
 &+(1-\gamma)\sum_{n<N}a_nw_{n+1}
                       \sum_kp_n(y,k)(Au(x))_k.
\end{align*}
The first difference is at most $\mathcal Vq_0$. The last term
is at most $\gamma\mathcal Vq_0$ by the mass budget. Applying
Lemma~\ref{lemma:generic_corrector} to $\psi^\circ$ gives
\[
 \E\sum_{n<N}a_n\sum_kp_n(y,k)h_k^2
 \leq[(1+\gamma)\mathcal V
                    +6\overline aH\varrho\mathcal V]q_0
 \leq8\mathcal Vq_0,
\]
using $\overline a\varrho H\leq1$ in the last step.
\end{proof}

For the Poisson solution associated with $h$, set
\begin{equation*}
 v_{i,k}(x)=\E[m_{n+1}(i,k)^2\mid X_n=x],
 \qquad \bv(x)=(v_{i,k}(x))_{i,k},
\end{equation*}
where the reverse marked kernel defines the conditional expectation for
every $x$. Also define
\begin{equation*}
 F_n=\bp_n^\top\bU^{\odot2}(X_n),
 \qquad 0\leq F_n\leq b_{\mathrm{osc}}^2q_0.
\end{equation*}

\begin{lemma}[Compensated predictable energy]
\label{lemma:variance_compensation}
For every source-component pair and reverse state,
\begin{equation}\label{eq:variance_compensation_coordinate}
 v_{i,k}(x)=
 \E[\ind_{\{X_{n+1}=i\}}h_k(i,\omega_{n+1},x)^2\mid X_n=x]
       +(\bP^\star U_{i,k}^2)(x)-U_{i,k}(x)^2.
\end{equation}
Under the preceding assumptions,
\begin{equation}\label{eq:framework_predictable_energy}
 \E_{\bp_0,x}\sum_{n<N}a_n\bp_n^\top\bv(X_n)
 \leq9\mathcal Vq_0.
\end{equation}
\end{lemma}
\begin{proof}
Write $B=\ind_{\{y=i\}}h_k(i,\omega,x)$.
The anchor gives $B U_{i,k}(y)=0$, while
$\E[B\mid x]=U_{i,k}(x)-(\bP^\star U_{i,k})(x)$.
Expanding $(B-U_{i,k}(x)+U_{i,k}(y))^2$ and substituting these two
identities proves \eqref{eq:variance_compensation_coordinate}.
The transpose update gives
\[
 F_{n+1}-F_n
 =\bp_n^\top\bigl(\bU^{\odot2}(y)-\bU^{\odot2}(x)\bigr)
     +a_nw\sum_{k,\ell}p_n(y,k)A_{k\ell}U_{x,\ell}(y)^2.
\]
Its removed-mass term vanishes by anchoring. Multiply
\eqref{eq:variance_compensation_coordinate} by $p_n(i,k)$, sum,
and use nonnegativity of the last term to obtain
\[
 \bp_n^\top\bv(X_n)
 \leq\E\!\left[\sum_kp_n(y,k)h_k^2\,\middle|\,\gG_n\right]
          +\E[F_{n+1}-F_n\mid\gG_n].
\]
Since the reverse steps are nondecreasing, summation by parts gives
\[
 \sum_{n<N}a_n(F_{n+1}-F_n)
 =a_{N-1}F_N-a_0F_0
       -\sum_{n=1}^{N-1}(a_n-a_{n-1})F_n
 \leq\overline a b_{\mathrm{osc}}^2q_0.
\]
Multiplying the preceding conditional inequality by $a_n$, summing,
and applying \eqref{eq:framework_raw_energy} proves
\eqref{eq:framework_predictable_energy}, since
$\overline a b_{\mathrm{osc}}^2\leq\overline a\varrho\mathcal V\leq\mathcal V$.
\end{proof}

\subsection{Concentration on an arbitrary deterministic interval}

\begin{proposition}[Generic interval convolution bound]
\label{prop:affine_noise}
Suppose $0<\eta_{T-1}\leq\cdots\leq\eta_s\leq(\varrho H)^{-1}$.
Let the convolution in \eqref{eq:affine_decomposition} use an observable
$h$ satisfying \eqref{eq:source_centering}--\eqref{eq:source_range}
and \eqref{eq:generic_variance_certificate}. For any initial state
distribution and $q\in(0,1)$, set $\ell_q=\log(4dm/(q\mumin))$.
Then
\begin{equation}\label{eq:framework_noise}
 \PB\!\left(\norm{\mathcal Y_{T,s}}_{\max}
   >7\sqrt{\eta_s\mathcal V}\,\ell_q+2\eta_sHb_{\mathrm{osc}}\right)\leq q.
\end{equation}
For affine prediction, one may take $\mathcal V=\varrho H^2$ and $b_{\mathrm{osc}}=\varrho H$.
\end{proposition}
\begin{proof}
If $\mathcal V=0$, then $b_{\mathrm{osc}}=0$ and source centering implies $h=0$
almost surely, so the claim is immediate. Suppose $\mathcal V>0$.
First start the interval in stationarity and fix one terminal coordinate.
Then $q_0=1$. Extend the finite reverse sequence by its last value;
it remains nondecreasing and bounded by $\overline a=\eta_s$.
The augmented process $(n,\bp_n,X_n)$ is time-homogeneous Markov:
from $(r,p,x)$ it draws $(y,\omega)$ from
\eqref{eq:reverse_reward_law}, applies \eqref{eq:block_reverse_update}
with $a_r$, and increments the clock. Its future law conditional on
the complete reverse history has this kernel, even when $A$ and $h$
are dependent. For
$f(r,p,x)=a_rp^\top\bv(x)$, the uniform estimate
\eqref{eq:framework_predictable_energy} gives
\[
 \sup_{r,p,x}\sup_{M\geq1}\E_{(r,p,x)}\sum_{n<M}f(n+r,\bp_n,X_n)
 \leq9\mathcal V,
 \qquad p\geq0,\quad\norm p_1\leq1.
\]
Every suffix of the extended steps meets the same hypotheses.
Lemma~\ref{lemma:discrete_khasminskii} therefore gives
\[
 \PB\!\left(\sum_{n<N}a_n\bp_n^\top\bv(X_n)
                        >18\mathcal V\ell\right)\leq2e^{-\ell}.
\]
The reverse martingale $M_N=\sum_{n<N}a_n\bp_n^\top\bm m_{n+1}$
has increments bounded by $\overline a b_{\mathrm{osc}}$. Weighted Cauchy--Schwarz
and $\norm{\bp_n}_1\leq1$ give
\[
 \langle M\rangle_N
 \leq\sum_{n<N}a_n^2\bp_n^\top\bv(X_n)
 \leq\overline a\sum_{n<N}a_n\bp_n^\top\bv(X_n).
\]
Lemma~\ref{lemma:scalar_freedman} consequently gives, outside an event
of probability at most $4e^{-\ell}$,
\[
 |M_N|\leq\sqrt{36\overline a\mathcal V}\,\ell
          +\frac23\overline a b_{\mathrm{osc}}\ell
 \leq7\sqrt{\overline a\mathcal V}\,\ell.
\]
The last inequality uses $\overline a b_{\mathrm{osc}}^2\leq\mathcal V$
and $6+2/3<7$. The deterministic correction in
Lemma~\ref{lemma:generic_corrector} is at most $2\overline aHb_{\mathrm{osc}}$.
A union bound over the $dm$ terminal coordinates with $\ell=\ell_q$
therefore gives stationary failure probability at most $q\mumin$.

Finally, let $\nu_s$ denote the law of $S_s$ under $\PB_\nu$.
Every event $E$ determined by this marked interval satisfies
\begin{equation*}
 \PB_\nu(E)
 =\E_\mu\!\left[\frac{\nu_s(S_s)}{\mu(S_s)}\ind_E\right]
 \leq\frac{\PB_\mu(E)}{\mumin},
\end{equation*}
because the conditional marked path law given $S_s$ is unchanged.
This proves \eqref{eq:framework_noise} for arbitrary initial laws.
\end{proof}

\begin{proposition}[Scalar TD convolution]\label{prop:stochastic_convolution}
For $0<\eta\leq(\varrho H)^{-1}$ and a deterministic terminal time $T$, the
scalar TD stochastic convolution $\bY_T$ satisfies, with
$\ell_q=\log(4d/(q\mumin))$,
\begin{equation*}
 \PB_\nu\!\left(\norm{\bY_T}_\infty
             >7H\sqrt{\varrho\eta}\,\ell_q+2\eta\varrho H^2\right)\leq q.
\end{equation*}
\end{proposition}
\begin{proof}
Apply Proposition~\ref{prop:affine_noise} with $m=1$, $s=0$,
constant steps, and the certificate from
Lemma~\ref{lemma:return_variance}. The case $T=0$ is immediate.
\end{proof}

\subsection{Proof of the transfer theorem}

\begin{proof}[Proof of Theorem~\ref{theorem:affine_transfer}]
For constant steps, Proposition~\ref{prop:affine_noise} with
$q=\delta/2$, $b_{\mathrm{osc}}=\varrho H$, and $\mathcal V=\varrho H^2$ gives
\[
 \norm{\mathcal Y_{T,0}}_{\max}
 \leq7H\sqrt{\varrho\eta}L+2\eta\varrho H^2
 =\frac7{32}\varepsilon+\frac{\varepsilon^2}{512L^2}
 <\frac\varepsilon2
\]
with failure probability at most $\delta/2$.
Proposition~\ref{proposition:initialization}, combined with
\eqref{eq:affine_domination}, bounds the propagated initial error by
$\varepsilon/2$ with the same failure budget. Its scalar block count
requires $\log_2(4dH/(\delta\varepsilon\mumin))$; the chosen count
with $dm$ is at least as large. Equation~\eqref{eq:affine_decomposition}
and a union bound prove part~(i).

For part~(ii), use the suffix $s=\lfloor T/2\rfloor$ from the proof
of Theorem~\ref{theorem:decay}. With $d$ replaced by $dm$, its
deterministic schedule, logarithmic inequalities, and block-length
calculations are unchanged. Higher-power scalar forgetting from
Corollary~\ref{corollary:decay_power_forget}, followed by
\eqref{eq:affine_domination}, gives the same bound $1/(T+8)$ for
the block propagation norm. The prefactor in this forgetting bound is
still $d$, so replacing it by $dm$ only increases the upper bound.
Since the block iterates and target lie in $[0,H]^{d\times m}$,
the propagated error is at most $H/(T+8)$. Proposition~\ref{prop:affine_noise}
has exactly the scalar noise constants with $d$ replaced by $dm$.
Thus the fixed-time error estimate and the summable allocation
$q_T=\delta/(4(T+1)^2)$ give the asserted simultaneous event, with
the constants in \eqref{eq:decay_rate}.

For part~(iii), the counts and their multiplicative concentration in
Lemma~\ref{lemma:adaptive_occupancy} involve only the physical chain.
Replacing $d$ by $dm$ in the logarithms enlarges the budgets used
in that proof. Conditional on the complete forward history at an epoch
boundary, the selected step is fixed and the marked future obeys the
same kernel from its observed initial state. Therefore
Proposition~\ref{prop:affine_noise} applies conditionally within
that epoch. Scalar higher-power forgetting and
\eqref{eq:affine_domination} give the same conditional propagation
bound. The proof of Theorem~\ref{theorem:unknown_mu} combines a
completed epoch with the current partial epoch using only these two
estimates, the invariant box, and maximum-norm contraction. Each of
these properties holds here with the same constants. Its union bounds
and deterministic sample-size calculation consequently prove
part~(iii). The sample complexity \eqref{eq:affine_sample_complexity}
follows in each case from the corresponding scalar calculation.
\end{proof}

\section{Forgetting the Initialization}
\label{Section:initialization}
We analyze the importance-weighted scalar TD learning propagator in this section.
Write $\bP=\bP_b$ for the behavior kernel and $\bQ=\bP_\pi$ for the target kernel.
The propagation comparison in Section~\ref{Section:framework} then
transfers the bound to the full affine prediction class.
We bound how long the recursion retains its initial error. By
\eqref{eq:error_decomposition}, this amounts to controlling the largest row
sum of $\widehat{\bPhi}_{T,0}$. Reading these row sums backward gives a useful
interpretation: each is the survival probability of an auxiliary particle
whose position follows the random propagation weights. We will show that
its expected lifetime is at most a constant times
\[
    \frac{H}{\eta\mumin}+\mathfrak h_\star,
\]
where $\mathfrak h_\star$ is a worst-case expected hitting time for the
reverse chain. The hitting-time cost therefore enters additively.

\subsection{Hitting times and a compensation identity}

For a chain started at $x$, let
\[
    \tau_i\coloneq\inf\{n\geq0:X_n=i\},
    \qquad
    \tau_i^+\coloneq\inf\{n\geq1:X_n=i\}.
\]
The notation $\E_x$ below refers to the state chain with kernel $\bP$,
whereas $\E_x^\star$ refers to the reverse kernel in
\eqref{eq:main_hitting_parameter}. Define
\begin{equation}\label{eq:initialization_hitting_times}
\begin{split}
    h_i(x)&\coloneq\E_x\tau_i,
    \qquad h_i^\star(x)\coloneq\E_x^\star\tau_i,\\
    h_i^{\star,+}(x)&\coloneq\E_x^\star\tau_i^+,
    \qquad
    \mathfrak h_\star=\max_{i,x\in\gS}h_i^\star(x).
\end{split}
\end{equation}
Thus $h_i(i)=h_i^\star(i)=0$. In contrast,
$h_i^{\star,+}(i)$ is the mean time to return to $i$ after at least one
step. This distinction matters when the particle and the chain start at
the same state.

The following quantities describe the average and largest waiting costs:
\begin{equation}\label{eq:hitting_potential}
\begin{split}
    a(i)&\coloneq\sum_x\mu(x)h_i^\star(x),
    \qquad a_{\max}\coloneq\max_i a(i),\\
    r_\star&\coloneq\max_{i,x}h_i^{\star,+}(x),
    \qquad c(i)\coloneq\frac1{\mu(i)}.
\end{split}
\end{equation}
Here $a(i)$ is a hitting-time potential, distinct from the interval
endpoints $a_i$ used for the generic observable in
\eqref{eq:source_range}. Write $\bm a=(a(i))_{i\in\gS}$ and
$\bm c=(c(i))_{i\in\gS}$.

\begin{lemma}[Last predecessor in a reverse return excursion]
\label{lemma:reverse_predecessor}
Start the reverse chain at $i$ and let $J=X_{\tau_i^+-1}$. Then
\begin{equation}\label{eq:reverse_return_predecessor}
    h_i^{\star,+}(i)=\frac1{\mu(i)},
    \qquad \PB_i^\star(J=j)=P(i,j),\qquad j\in\gS.
\end{equation}
The forward chain has the same mean return time:
$\E_i\tau_i^+=1/\mu(i)$.
\end{lemma}

The same identity holds with the complete mark on the final edge. If
$\omega$ denotes that mark and $\mathsf K(\mathrm d\omega\mid i,j)$
denotes its conditional law given the forward edge $(i,j)$, then
\begin{equation}\label{eq:reverse_return_mark}
 \PB_i^\star(J=j,\omega\in\mathrm d\omega)
 =P(i,j)\mathsf K(\mathrm d\omega\mid i,j).
\end{equation}
Indeed, conditional on a state excursion the last-edge mark has precisely
this conditional law, by Lemma~\ref{lemma:reverse_kernel}; averaging over
$J$ proves the assertion. If $W$ is the importance weight in this mark,
\begin{equation}\label{eq:reverse_return_weight}
 \E_i^\star W=1,
 \qquad \E_i^\star[W\ind\{J=j\}]=Q(i,j).
\end{equation}
The appearance of the target row $Q(i,\cdot)$ in
\eqref{eq:reverse_return_weight} is useful. A reverse return excursion
from $i$ is sampled from the reverse behavior law. For a particle at $i$,
conditional on choosing relocation at its end, the destination has
distribution $Q(i,\cdot)$. Moreover, since positive-time and ordinary hitting
times differ only when the chain starts at its target,
\begin{equation}\label{eq:hitting_basic_bounds}
    0\leq a_{\max}\leq\mathfrak h_\star,
    \qquad
    r_\star\leq\mathfrak h_\star+\mumin^{-1}.
\end{equation}

The next lemma combines the classical hitting-time reversal identity
\citep{beveridge2010exit} with the return-time equation to express
average relocation costs as a potential difference and a return-time term.

\begin{lemma}[Compensation for reverse hitting costs]
\label{lemma:hitting_compensation}
For every $i,j\in\gS$,
\begin{equation}\label{eq:hitting_reversal_identity}
    h_j^\star(i)=h_i(j)+a(j)-a(i).
\end{equation}
Consequently, if
\[
    q(i)\coloneq\sum_jP(i,j)h_j^\star(i),
    \qquad
    \ell(i)\coloneq\sum_jP(i,j)h_j^{\star,+}(i),
\]
then the vectors $\bq=(q(i))_{i\in\gS}$ and
$\bm\ell=(\ell(i))_{i\in\gS}$ satisfy
\begin{align}
    \bq&=(\bP-\bI)\bm a+\bm c-\bone,
       \label{eq:hitting_compensation}\\
    \ell(i)&=q(i)+\frac{P(i,i)}{\mu(i)},
    \qquad
    \bm\ell\leq(\bP-\bI)\bm a+2\bm c.
       \label{eq:positive_hitting_compensation}
\end{align}
\end{lemma}

Complete proofs of both lemmas are given in
Appendix~\ref{Appendix:hitting_details}. The term
$(\bP-\bI)\bm a$ is the expected change of the potential along one
behavior transition, which allows successive relocation costs to cancel.
For off-policy relocation, the following consequence
is the relevant form of the compensation.
\begin{corollary}[Weighted hitting-cost compensation]
\label{corollary:weighted_hitting_compensation}
Define $\ell_Q(i)=\sum_jQ(i,j)h_j^{\star,+}(i)$. Then
\begin{equation}\label{eq:weighted_hitting_compensation}
 \bm\ell_Q\leq(\bQ-\bI)\bm a+2\varrho\bm c.
\end{equation}
\end{corollary}
\begin{proof}
The coverage condition and $0\leq W\leq\varrho$ give
$Q(i,j)\leq\varrho P(i,j)$. Applying
\eqref{eq:hitting_reversal_identity} before this domination yields
\[
 \begin{split}
 \ell_Q(i)
 &=((\bQ-\bI)\bm a)(i)
   +\sum_jQ(i,j)h_i(j)+Q(i,i)c(i)\\
 &\leq((\bQ-\bI)\bm a)(i)
   +\varrho(c(i)-1)+\varrho P(i,i)c(i),
 \end{split}
\]
where the forward return-time equation gives
$\sum_jP(i,j)h_i(j)=c(i)-1$. The last two terms are at most
$2\varrho c(i)$. No irreducibility assumption on $\bQ$ is needed.
\end{proof}

\subsection{An auxiliary particle and its expected lifetime}

Use the scalar specialization $m=1$, $A_t=\gamma$ of the reverse weights in Section~\ref{Section:reverse}. On a
reverse marked transition from $x=X_n$ to $y=X_{n+1}$ with importance
weight $w_{n+1}$, their update is
\[
 \bp_{n+1}=\bp_n+\eta w_{n+1}p_n(y)(\gamma\be_x-\be_y).
\]
Only the mass currently at $y$ is redistributed: some remains at $y$, some moves to
$x$, and the rest is removed. This is exactly the update of the following
particle distribution.

The driver evolves with the full reverse marked kernel of
Lemma~\ref{lemma:reverse_kernel}. While alive, the particle has
position $I_n\in\gS$, and death is absorbing. Given a pair
$(I_n,X_n)=(i,x)$ with the particle alive, first draw the next marked driver edge.
If $y\neq i$, leave the particle at $i$. If $y=i$, use a fresh auxiliary
uniform variable to choose an action with the following conditional probabilities:
\begin{equation}\label{eq:particle_transition}
\begin{array}{c|c}
 \text{Particle action}&\text{Conditional probability}\\ \hline
 \text{Remain at }i&1-\eta w_{n+1}\\
 \text{Move to }x&\eta w_{n+1}\gamma\\
 \text{Die}&\eta w_{n+1}(1-\gamma).
\end{array}
\end{equation}
The next driver state is $y$ in all cases. The auxiliary uniform variables are
independent of the entire marked driver path. The action probabilities
may depend on that path through $w_{n+1}$; no independence between the
importance weight, reward, and state transition is assumed.
When $x=i$, remaining and moving lead to the same position, but we may
still distinguish the two choices. A step with $y=I_n$ is called an
\emph{activation}. Before its first activation the particle stays in
place.

Let $\zeta\geq1$ denote the step at which the particle dies. Start it at
$I_0=j$. Conditional on the complete marked driver path, the particle's
subprobability vector on the nonabsorbing states obeys the same recursion as $\bp_n$, with initial
vector $\be_j$. It follows by induction that
\begin{equation}\label{eq:particle_survival_representation}
 \norm{\bp_N}_1
 =\PB(\zeta>N\mid X_0,\omega_1,X_1,\ldots,\omega_N,X_N).
\end{equation}
Indeed, the mass at $y$ is split in the proportions in
\eqref{eq:particle_transition}, and all other mass is unchanged. The
marked driver path is independent of the auxiliary uniform variables, so conditioning
on its full path does not alter these proportions. For the finite
stationary path in \eqref{eq:reverse_path}, this identity represents the
row sum $\be_j^\top\widehat{\bPhi}_{T,T-N}\bone$.

\begin{lemma}[Uniform expected particle lifetime]\label{lemma:particle_lifetime}
Suppose that $0<\eta\varrho H\leq1/4$. For the particle in
\eqref{eq:particle_transition},
\begin{equation*}
    M\coloneq\max_{i,x\in\gS}\E_{i,x}\zeta
    \leq4\left(\frac{H}{\eta\mumin}+\mathfrak h_\star\right),
\end{equation*}
where $\E_{i,x}$ starts the particle alive at $i$ and the driver at $x$.
\end{lemma}

\begin{proof}\leavevmode
\paragraph{First establish finiteness.}
Let $L_1=\lceil2r_\star\rceil$ and
$L_2=\lceil2\varrho/\mumin\rceil$.
From any pair $(i,x)$ the first positive-time visit to $i$ occurs by
$L_1$ with probability at least $1/2$. At that activation the particle
remains at $i$ with conditional probability at least
$1-\eta\varrho\geq3/4$. After such a choice both the particle and
driver are at $i$. For the following reverse return excursion let
$\sigma$ be its duration and $W$ its last-edge weight.
By \eqref{eq:reverse_return_weight}, $\E_i^\star W=1$, and hence
\[
 \E_i^\star[W\ind\{\sigma\leq L_2\}]
 \geq1-\varrho\PB_i^\star(\sigma>L_2)
 \geq1-\frac{\varrho c(i)}{L_2}\geq\frac12.
\]
Conditional on this marked excursion, the death probability at its end is
$\eta W/H$. Thus from every alive pair death occurs during the next
$L_1+L_2$ steps with probability at least $3\eta/(16H)$.
The Markov property of the marked driver and independent auxiliary
uniform variables gives a geometric bound over successive blocks.
Consequently $M<\infty$.

\paragraph{Write equations at the first activation.}
Set
\[
    f_i(x)\coloneq\E_{i,x}\zeta,
    \qquad D(i)\coloneq f_i(i),
    \qquad D_{\max}\coloneq\max_iD(i).
\]
Let $J$ be the driver state just before its first positive-time visit to
$i$, and let $W$ be the weight of that final edge. At this activation the new driver state is $i$. The particle either
remains at $i$, dies, or moves to $J$. Its remaining expected lifetime in
these three cases is $D(i)$, zero, or $f_J(i)$, respectively. The strong
Markov property and the independent auxiliary uniform variable therefore give, for
every $i,x$, including $x=i$,
\begin{equation}\label{eq:particle_first_activation}
 f_i(x)=h_i^{\star,+}(x)
    +\E_x^\star[(1-\eta W)D(i)+\eta\gamma W f_J(i)].
\end{equation}
The positive-time convention includes the waiting time for a return
when $x=i$. Using $f_J(i)\leq M$ and dropping the nonpositive term
$-\eta\E_x^\star[W]D(i)$ yields
\begin{equation}\label{eq:particle_first_maximum}
 f_i(x)\leq h_i^{\star,+}(x)+D(i)+\eta\varrho\gamma M,
 \qquad
 M\leq r_\star+D_{\max}+\eta\varrho\gamma M.
\end{equation}
On the diagonal, Lemma~\ref{lemma:reverse_predecessor} and
\eqref{eq:reverse_return_weight} identify the mean waiting time and the
weighted predecessor law. Subtracting
$(1-\eta)D(i)$ from \eqref{eq:particle_first_activation} and dividing
by $\eta$ yields the exact relation
\begin{equation}\label{eq:particle_diagonal_equation}
 D(i)=\frac{c(i)}\eta+\gamma\sum_jQ(i,j)f_j(i).
\end{equation}
This use of $\bQ$ follows from the full marked excursion law; the
excursion duration itself remains the behavior-chain return time.

\paragraph{Use the compensation identity before taking a maximum.}
Apply \eqref{eq:particle_first_maximum} to each $f_j(i)$ in
\eqref{eq:particle_diagonal_equation}. Thus, with
$\bm D=(D(i))_{i\in\gS}$,
\[
 \bm D\leq\frac{\bm c}\eta+\gamma\bm\ell_Q
             +\gamma\bQ\bm D
             +\eta\varrho\gamma^2 M\bone.
\]
The inverse $\bR=(\bI-\gamma\bQ)^{-1}=\sum_{n\geq0}\gamma^n\bQ^n$
is entrywise nonnegative, and its row sums are $H$.
Insert \eqref{eq:weighted_hitting_compensation} before applying this
inverse. We obtain
\begin{equation}\label{eq:particle_resolvent_bound}
 \bm D\leq
 \bR\left[\left(\frac1\eta+2\gamma\varrho\right)\bm c
              +\gamma(\bQ-\bI)\bm a\right]
       +\eta\varrho\gamma^2HM\bone.
\end{equation}
The potential term can be estimated without a factor $H$:
\[
 \gamma\bR(\bQ-\bI)=(1-\gamma)\bR-\bI,
 \qquad
 \gamma\bR(\bQ-\bI)\bm a\leq a_{\max}\bone.
\]
Consequently,
\begin{equation}\label{eq:particle_diagonal_maximum}
 D_{\max}\leq
 \frac{H(1/\eta+2\gamma\varrho)}{\mumin}
       +a_{\max}+\eta\varrho\gamma^2HM.
\end{equation}

\paragraph{Close the lifetime inequality.}
Since $\gamma+\gamma^2H=\gamma H$, combining
\eqref{eq:particle_first_maximum} and
\eqref{eq:particle_diagonal_maximum} yields
\begin{equation}\label{eq:particle_lifetime_precise}
 (1-\eta\varrho\gamma H)M
 \leq r_\star+a_{\max}
       +\frac{H(1/\eta+2\gamma\varrho)}{\mumin}.
\end{equation}
Let $A=H/(\eta\mumin)$. Our step-size assumption implies
$1-\eta\varrho\gamma H\geq3/4$,
$2\gamma\varrho H/\mumin\leq A/2$, and
$1/\mumin\leq A$. Using \eqref{eq:hitting_basic_bounds},
\[
 M\leq\frac43\left(2\mathfrak h_\star+\frac52A\right)
 \leq4(\mathfrak h_\star+A).
\]
\end{proof}

The ratio cap enters the allowed step size, whereas the lifetime term
$H/(\eta\mumin)$ and the additive behavior hitting cost retain their
original form. The potential in \eqref{eq:weighted_hitting_compensation}
is responsible for avoiding an additional horizon factor on that cost.

\subsection{From lifetime to uniform forgetting}

\begin{proposition}[Forgetting the initialization]\label{proposition:initialization}
Suppose that $0<\eta\varrho H\leq1/4$, $0<\varepsilon\leq1$, and $0<\delta<1$.
Define
\begin{equation*}
\begin{split}
    b_{\mathrm{blk}}&\coloneq\left\lceil8\left(
            \frac{H}{\eta\mumin}+\mathfrak h_\star\right)\right\rceil,\\
    k&\coloneq\left\lceil
          \log_2\frac{4dH}{\delta\varepsilon\mumin}\right\rceil.
\end{split}
\end{equation*}
Under any initial state distribution, with probability at least
$1-\delta/2$,
\begin{equation*}
    \norm{\widehat{\bPhi}_{T,0}}_{\infty\to\infty}
       \leq\frac{\varepsilon}{2H}
       \qquad\text{for every integer }T\geq b_{\mathrm{blk}} k.
\end{equation*}
On the same event, every initialization satisfying
$\norm{\bDelta_0}_\infty\leq H$ obeys
$\norm{\widehat{\bPhi}_{T,0}\bDelta_0}_\infty\leq\varepsilon/2$
for all such $T$.
\end{proposition}

\begin{proof}
Write $K_\star\coloneq4(H/(\eta\mumin)+\mathfrak h_\star)$, so that
$b_{\mathrm{blk}}=\lceil2K_\star\rceil$. Lemma~\ref{lemma:particle_lifetime} and
Markov's inequality imply
\[
    \sup_{i,x}\PB_{i,x}(\zeta>b_{\mathrm{blk}})\leq\frac12.
\]
Because this estimate is uniform over all particle-driver pairs with the particle alive,
we can apply it conditionally at each deterministic block boundary.
If the particle is alive at time $mb_{\mathrm{blk}}$, its current position and driver
state specify such a pair. Induction using the Markov property yields
\begin{equation}\label{eq:particle_lifetime_tail}
    \sup_{i,x}\PB_{i,x}(\zeta>mb_{\mathrm{blk}})\leq2^{-m},
    \qquad m=0,1,\ldots.
\end{equation}

First start the forward state chain from $\bmu$, and fix a deterministic
$T\geq b_{\mathrm{blk}} k$. Under stationary time reversal the driver begins at
$X_0=S_T\sim\bmu$ and evolves with kernel $\bP^\star$.
For every row $j$, the representation
\eqref{eq:particle_survival_representation}, followed by
\eqref{eq:particle_lifetime_tail}, gives
\[
    \E_{\bmu}\brk{\be_j^\top\widehat{\bPhi}_{T,0}\bone}
      =\sum_x\mu(x)\PB_{j,x}(\zeta>T)
      \leq2^{-k}.
\]
For a nonnegative matrix the induced sup norm is its largest row sum.
Markov's inequality for each row and a union bound therefore show that
\begin{equation}\label{eq:initialization_stationary_tail}
    \PB_{\bmu}\left(
        \norm{\widehat{\bPhi}_{T,0}}_{\infty\to\infty}
                   >\frac{\varepsilon}{2H}\right)
      \leq\frac{2dH}{\varepsilon}\,2^{-k}.
\end{equation}

To allow any initial distribution $\bnu$, observe that every event $A$
determined by the complete forward marked path satisfies
\[
    \PB_{\bnu}(A)
      =\E_{\bmu}\left[\frac{\nu(S_0)}{\mu(S_0)}\ind_A\right]
      \leq\frac{\PB_{\bmu}(A)}{\mumin}.
\]
The propagation matrices are measurable with respect to this marked path. Applying the
last inequality to the event in
\eqref{eq:initialization_stationary_tail} makes its probability at most
$2dH\,2^{-k}/(\varepsilon\mumin)\leq\delta/2$.

Apply this fixed-time result once, at $T_0=b_{\mathrm{blk}} k$. Every subsequent
propagation factor is a sup-norm contraction. Hence, on the same event,
\[
    \norm{\widehat{\bPhi}_{T,0}}_{\infty\to\infty}
      \leq\norm{\widehat{\bPhi}_{T,T_0}}_{\infty\to\infty}
             \norm{\widehat{\bPhi}_{T_0,0}}_{\infty\to\infty}
      \leq\frac{\varepsilon}{2H},\qquad T\geq T_0.
\]
This proves the simultaneous assertion without a union bound over
later times. Multiplication by $\bDelta_0$ proves the last statement,
including when the bounded initial estimate is random.
\end{proof}

The initialization cost is therefore
\begin{equation}\label{eq:initialization_cost}
    b_{\mathrm{blk}} k\leq C\left(\frac{H}{\eta\mumin}+\mathfrak h_\star\right)
                    \log\frac{4dH}{\delta\varepsilon\mumin},
\end{equation}
for a universal constant $C$. The ceilings are absorbed because
$H/(\eta\mumin)\geq1$ and the displayed logarithm is at least
$\log4$.

\begin{lemma}[From reverse hitting to forward mixing]
\label{lemma:reverse_hitting_mixing}
Suppose additionally that $\bP$ is aperiodic. For the amplified mixing interval $\tau$ defined in
Section~\ref{subsection:state_chain},
\begin{equation}\label{eq:reverse_hitting_mixing}
    \mathfrak h_\star\leq\frac{2\tau}{\mumin}.
\end{equation}
\end{lemma}

\begin{proof}
Lemma~\ref{lemma:mixing_amplification} in
Appendix~\ref{Appendix:hitting_details} gives
$P^\tau(i,x)\geq\mu(x)/2$ for all $i,x$.
Time reversal at $\tau$ steps consequently gives
\[
    (P^\star)^\tau(x,i)
       =\frac{\mu(i)P^\tau(i,x)}{\mu(x)}
       \geq\frac{\mu(i)}2.
\]
For a reverse chain started at any $x$, let
$N_i\coloneq\inf\{m\geq1:X_{m\tau}=i\}$. At each of these
observation times, conditional on the preceding observations, the chance
of seeing $i$ is at least $\mu(i)/2$. Iterating conditional probabilities
gives
\[
    \PB_x^\star(N_i>m)\leq(1-\mu(i)/2)^m,
    \qquad m=0,1,\ldots.
\]
Thus $\E_x^\star N_i\leq2/\mu(i)$. Since the chain can hit $i$
before the first such observation, $\tau_i\leq\tau N_i$ pathwise,
including when $x=i$. Taking expectations and then maximizing over
$x,i$ proves \eqref{eq:reverse_hitting_mixing}.
\end{proof}

For aperiodic chains, this argument uses the forward mixing time directly. It does not require
an equality or comparison between the forward and reverse total-variation
mixing times. Combining \eqref{eq:initialization_cost} with
\eqref{eq:reverse_hitting_mixing} gives an additive cost of order
$\wtilde O(H/(\eta\mumin)+\tmix/\mumin)$. The more precise
hitting-time bound can be smaller on particular chains.

\section{Extensions}
\label{Section:extensions}
The affine recursion covers several prediction algorithms that use the
same observed transition to update multiple quantities. We first give
applications with uncoupled prediction coordinates, then an example
in which the coordinates interact through an observed mode transition.
Throughout, the physical state chain is the finite irreducible behavior chain
of Section~\ref{subsection:state_chain}; periodicity is allowed.
All observed marks obey the fixed conditional Markov kernel in the
affine framework: conditional on the current state, their joint law
with the next state is independent of the preceding history.

\subsection{Successor features and simultaneous reward evaluation}
\label{subsection:successor_features}

Let $\bm\phi_t\in\RB^m$ be an observed feature vector whose joint law
with $S_{t+1}$, conditional on the past, depends only on $S_t$.
The feature map and its observation rule are fixed throughout learning.
The successor feature vector is
\[
 \bm\Psi^\star(x)
 :=\E_x^\pi\left[\sum_{t\geq0}\gamma^t\bm\phi_t\right].
\]
Successor representations and successor features separate predictions
about future state or feature occupancies from the reward weights used
to evaluate a task \citep{dayan1993successor,barreto2017successor}.
Their importance-weighted tabular update is
\begin{equation}\label{eq:successor_update}
 \bm\Psi_{t+1}(S_t)
 =(1-\eta_t\rho_t)\bm\Psi_t(S_t)
  +\eta_t\rho_t\{\bm\phi_t+\gamma\bm\Psi_t(S_{t+1})\},
\end{equation}
with other state entries unchanged.
When $\bm0\leq\bm\phi_t\leq\bone$, the update
\eqref{eq:successor_update} is
\eqref{eq:affine_update} with $\bm C_t=\bm\phi_t$ and
$\bm A_t=\gamma\bI_m$.
The same interpretation applies to any collection of bounded reward,
cost, or sensor signals observed along one trajectory.

Signed features are handled by an exact affine normalization.
Suppose that, for a known vector $\bm a\in\RB^m$ and $B>0$,
\[
 \bm a\leq\bm\phi_t\leq\bm a+B\bone,
 \qquad
 H\bm a\leq\bm\Psi_0(x)\leq H(\bm a+B\bone).
\]
Define $\bm Z_t(x)=(\bm\Psi_t(x)-H\bm a)/B$ and
$\bm C_t=(\bm\phi_t-\bm a)/B$.
Since $(1-\gamma)H\bm a=\bm a$, the normalized iterates satisfy
\eqref{eq:affine_update} with $\bm A_t=\gamma\bI_m$ and
$\bm0\leq\bm C_t\leq\bone$.
Theorem~\ref{theorem:affine_transfer}, applied with precision
$\varepsilon/B$, therefore gives the following consequence.

\begin{corollary}[Simultaneous successor feature estimation]
\label{corollary:successor_features}
Under the feature and initialization bounds above, for
$0<\varepsilon\leq B$, a sufficient transition count for
\[
 \max_{x\in\gS,\,k\in[m]}
   |\Psi_{T,k}(x)-\Psi_k^\star(x)|\leq\varepsilon
\]
with probability at least $1-\delta$ is
\[
 \wtilde O\left(
   \frac{\varrho B^2H^3}{\mumin\varepsilon^2}+\mathfrak h_\star\right).
\]
The constant-step, decreasing-step, and empirical-occupancy guarantees
of Theorem~\ref{theorem:affine_transfer} all apply, with the accuracy
rescaled by $B$. The latter two guarantees hold simultaneously after
their respective transient thresholds.
\end{corollary}

The number of prediction coordinates $m$ enters the sample bound
through logarithms; all $m$ coordinates are updated at each visited
state. For the successor representation, take
$\bm\phi_t=\be_{S_t}\in\RB^d$. With the successor vectors arranged as
rows, the target matrix is $(\bI-\gamma\bQ)^{-1}$.

The same event gives a uniform guarantee for reward transfer.
For $w\in\RB^m$, define the task reward
$R_t^w=\bm\phi_t^\top w$ and its value
$V_w(x)=\bm\Psi^\star(x)^\top w$.
Writing $\widehat{V}_{T,w}(x)=\bm\Psi_T(x)^\top w$, duality gives
\begin{equation}\label{eq:successor_transfer}
 \sup_{\norm{w}_1\leq W}
       \norm{\widehat{V}_{T,w}-V_w}_\infty
 =W\max_{x,k}|\Psi_{T,k}(x)-\Psi_k^\star(x)|.
\end{equation}
Thus the weights may be selected after observing the trajectory.
By \eqref{eq:successor_transfer}, for $W>0$ and
$0<\varepsilon\leq BW$, the sufficient count for uniform error
$\varepsilon$ over $\{w:\norm{w}_1\leq W\}$ is
\[
 \wtilde O\left(
    \frac{\varrho B^2W^2H^3}{\mumin\varepsilon^2}+\mathfrak h_\star\right).
\]
For a different set of reward weights, its $\ell_1$ radius determines
the corresponding transfer guarantee.

\begin{example}[Evaluating several inventory costs]\label{example:inventory_costs}
Consider the finite-capacity inventory system in
Example~\ref{example:seasonal}, and fix a target replenishment policy
covered by its behavior policy. Assume bounded demand.
Let the coordinates of $\bm\phi_t$ record normalized holding,
ordering, and lost-sales costs, each in $[0,1]$.
The update \eqref{eq:successor_update} estimates their expected
discounted totals under $\pi$ along the same trajectory.
For nonnegative cost weights $w$ with $\norm{w}_1\leq W$,
\eqref{eq:successor_transfer} gives a simultaneous guarantee for all
such weighted costs. The weights can therefore be changed after the
data have been collected without collecting a new trajectory.
\end{example}

\subsection{Generalized value predictions and finite modes}
\label{subsection:generalized_predictions}

Generalized value functions specify the signal being predicted and
the continuation of the prediction separately
\citep{sutton2011horde}.
Consider $m$ predictions under a fixed target policy $\pi$ with observed cumulants
$0\leq C_{t,k}\leq1$ and continuation factors
$0\leq\Gamma_{t,k}\leq\gamma<1$.
Initialize the prediction table in $[0,H]^{d\times m}$.
The joint mark $(\bm C_t,\bm\Gamma_t)$ may depend on both endpoints
of the observed transition. The target for coordinate $k$ is
\[
 Z_k^\star(x)
 =\E_x^\pi\left[
   \sum_{t\geq0}
      \left(\prod_{j=0}^{t-1}\Gamma_{j,k}\right)C_{t,k}
           \right],
\]
where an empty product equals one.
Updating all predictions by
\[
 Z_{t+1,k}(S_t)
 =(1-\eta_t\rho_t)Z_{t,k}(S_t)
  +\eta_t\rho_t\{C_{t,k}+\Gamma_{t,k}Z_{t,k}(S_{t+1})\}
\]
corresponds to $\bm A_t=\diag(\bm\Gamma_t)$.
Theorem~\ref{theorem:affine_transfer} therefore yields
\begin{equation}\label{eq:generalized_prediction_rate}
 \wtilde O\left(
       \frac{\varrho H^3}{\mumin\varepsilon^2}+\mathfrak h_\star\right),
 \qquad H=(1-\gamma)^{-1},\quad 0<\varepsilon\leq1,
\end{equation}
for simultaneous maximum-coordinate accuracy $\varepsilon$.
The caps $\gamma$ and $\varrho$ are inputs to the step-size schedules;
the law of the continuation factors is not.
For example,
$\Gamma_{t,k}=\gamma\ind\{S_{t+1}\notin D_k\}$ gives a prediction
that stops at a specified set $D_k$, while the observed physical
chain continues to evolve. In an inventory system, this can represent
accumulated shortage cost until the next visit to a replenished state.

The affine formulation also allows predictions to switch between
coordinates. Let $k\in[m]$ denote a finite mode, and let the observed
transition mark $\Omega_t$ specify a known mode transition
$f(k,\Omega_t)\in[m]$ and a cumulant $C_{t,k}\in[0,1]$.
The modes describe predictions and do not change the physical dynamics.
Starting from physical state $x$ and mode $k$, define
$K_0=k$ and $K_{t+1}=f(K_t,\Omega_t)$, and consider
\[
 Z_k^\star(x)
 =\E_{x,k}^\pi\left[\sum_{t\geq0}\gamma^t C_{t,K_t}\right].
\]
The sample update is now
\begin{equation}\label{eq:mode_prediction_update}
 Z_{t+1,k}(S_t)
 =(1-\eta_t\rho_t)Z_{t,k}(S_t)
   +\eta_t\rho_t\{C_{t,k}
         +\gamma Z_{t,f(k,\Omega_t)}(S_{t+1})\},
 \qquad k\in[m].
\end{equation}
It has the affine matrix
\[
A_{t,k\ell}
   =\gamma\ind\{\ell=f(k,\Omega_t)\},
 \qquad \bm A_t\bone=\gamma\bone.
\]
Again initialize every prediction in $[0,H]$.
This describes, for example, predictions whose mode records progress
through a prescribed sequence of service or inspection events.
Equation~\eqref{eq:mode_prediction_update} updates every possible
starting mode at the currently observed physical state, using the
same observed mark.
Consequently, Theorem~\ref{theorem:affine_transfer} gives
\eqref{eq:generalized_prediction_rate} with the stationary probabilities
and reverse hitting times of the physical chain $P$.
No irreducibility assumption on the joint state--mode process is
needed, and $m$ again enters the transition count only through logarithms.

\paragraph{Fixed-policy action-value prediction.}
For a fixed behavior policy $b$, the pair $X_t=(S_t,A_t)$ is a Markov chain
with kernel
\[
 P_b^X\bigl((s,a),(s',a')\bigr)
     =p(s'\mid s,a)b(a'\mid s').
\]
Here $p$ is the controlled state-transition kernel. For a fixed target
policy $\pi$ covered by $b$, let
\[
 W_t^X=\frac{\pi(A_{t+1}\mid S_{t+1})}
                  {b(A_{t+1}\mid S_{t+1})}\leq\varrho.
\]
The importance-weighted SARSA update is
\[
 \widehat Q_{t+1}(S_t,A_t)=\widehat Q_t(S_t,A_t)
 +\eta_tW_t^X\{R_t+\gamma\widehat Q_t(S_{t+1},A_{t+1})
                     -\widehat Q_t(S_t,A_t)\}.
\]
The next action is sampled from $b(\cdot\mid S_{t+1})$, conditionally
independently of the preceding reward given $S_{t+1}$. Hence
$\E_b[W_t^X\mid X_t]=1$, and the update is exactly scalar
importance-weighted TD learning on $X_t$ with target $Q^\pi$.
For $b=\pi$, it is the usual on-policy SARSA update.
If this chain is irreducible on a state--action support $\mathcal X$
and $0\leq\widehat Q_0(s,a)\leq H$ on $\mathcal X$,
all three guarantees apply on $\mathcal X$ with $d=|\mathcal X|$,
minimum stationary probability
$\min_{(s,a)\in\mathcal X}\mu_{b,S}(s)b(a\mid s)$ and its reverse hitting-time
parameter. Thus state--action prediction is covered by changing the
driving chain, while the vector and mode examples above retain the
original physical chain.

\subsection{Online estimation of the return variance}
\label{subsection:return_variance}

TD methods also estimate the variance of the return, a classical
criterion for assessing the risk of a fixed policy
\citep{tamar2016variance}.
Let
\[
 G\coloneq\sum_{k\geq0}\gamma^kR_k,
 \qquad M(s)\coloneq\E_s^\pi G^2,
 \qquad \sigma^2(s)\coloneq M(s)-V^\pi(s)^2.
\]
The second moment satisfies
\begin{equation}\label{eq:second_moment_bellman}
 M(s)=\E_\pi\!\left[
 R_t^2+2\gamma R_tV^\pi(S_{t+1})
       +\gamma^2M(S_{t+1})\mid S_t=s\right].
\end{equation}
This equation is triangular: its one-step reward uses the mean value.
We estimate both quantities online, using the same observed transitions.
The algorithm stores two scalar entries per state and updates only the
two entries at the visited state after each transition.
The mean estimate follows \eqref{eq:td_update_vector}, and the
second-moment estimate follows
\begin{equation}\label{eq:online_second_moment}
\begin{split}
 \widehat{\bm M}_{t+1}
 =\widehat{\bm M}_t+\theta_t\rho_t\be_{S_t}
 \bigl[&R_t^2+2\gamma R_tV_t(S_{t+1})\\
       &+\gamma^2\widehat{M}_t(S_{t+1})-\widehat{M}_t(S_t)\bigr].
\end{split}
\end{equation}
Initialize $\bm0\leq\widehat{\bm M}_0\leq H^2\bone$.
Put $H_2\coloneq(1-\gamma^2)^{-1}=H/(1+\gamma)$.
Use \eqref{eq:decay_schedule} for the mean estimate, and set
\begin{equation}\label{eq:variance_steps}
 \theta_t=\min\left\{\frac1{\varrho H_2},
 \frac{4096H_2\log(16d\varrho H_2(t+8)/\mumin)}{\mumin(t+8)}\right\}.
\end{equation}
The estimates remain in $[0,H]^d$ and $[0,H^2]^d$, respectively.
The output is the coordinatewise projection
\[
 \widehat{\bm\sigma}_t^{\,2}
 \coloneq\Pi_{[0,H^2/4]^d}
       \left(\widehat{\bm M}_t-\bV_t^{\odot2}\right).
\]
The projection uses the fact that the variance of a random variable
in $[0,H]$ is at most $H^2/4$.

\begin{proposition}[Time-uniform return-variance estimation]
\label{prop:variance_extension}
Under Assumptions~\ref{assumption:bounded} and~\ref{assumption:mixing},
run \eqref{eq:td_update_vector} and \eqref{eq:online_second_moment}
with \eqref{eq:decay_schedule} and \eqref{eq:variance_steps}.
For any $0<\delta<1$, define
\[
 \overline L_t=\log\frac{16d\varrho H(t+8)}{\mumin},
 \qquad L_t=\overline L_t+\log\frac2\delta.
\]
There are universal constants $C_1,C_2>0$ such that, under every
initial state distribution, with probability at least $1-\delta$,
simultaneously for every integer $T\geq16$ satisfying
\[
 T\geq C_1L_T\left(\frac{\varrho H^2}{\mumin}+\mathfrak h_\star\right),
\]
we have
\begin{equation}\label{eq:variance_extension_rate}
 \left\|\widehat{\bm\sigma}_T^{\,2}-\bm\sigma^2\right\|_\infty
 \leq C_2\left\{
 \frac{H^2}{T+8}
 +\sqrt{\frac{\varrho H^5\overline L_TL_T^2}{\mumin(T+8)}}
 +\frac{\varrho H^4\overline L_T}{\mumin(T+8)}\right\}.
\end{equation}
Consequently, for $0<\varepsilon\leq1$, a sufficient transition
count for both the mean and variance to have sup-norm error at most
$\varepsilon$ is
\[
 \wtilde O\left(\frac{\varrho H^5}{\mumin\varepsilon^2}
                  +\mathfrak h_\star\right).
\]
Both step-size schedules depend only on $H$, $d$, $\varrho$, and $\mumin$.
They are independent of $\varepsilon$, $\delta$, and the terminal time.
\end{proposition}

The triangular dependence in \eqref{eq:second_moment_bellman} introduces
an error in the one-step reward as the mean is learned.
On the last half of the trajectory, compare the second-moment recursion
with an oracle recursion using the true mean.
A reward rescaling places this oracle in the affine class of
Theorem~\ref{theorem:affine_transfer}.
A propagation-mass identity controls their difference by the largest
mean error on that interval. The scalar suffix bounds then control
the oracle error with the same additive hitting-time cost.
Appendix~\ref{Appendix:variance} gives the proof.

\section{Lower Bounds and the Scope of Optimality}
\label{Section:lower_bounds}
The upper bound contains a statistical term and a term for reaching states that the trajectory visits infrequently.
We first explain the second obstruction with a two-state example.
We then construct a single three-state family in which both terms are necessary, even when the reward function is known.
Throughout this section, an estimator may be randomized and may use the entire observation
$(S_0,A_0,R_0,S_1,\ldots,A_{T-1},R_{T-1},S_T)$.
The lower bounds therefore apply to every estimation procedure based on the trajectory.

\subsection{A two-state coverage obstruction}\label{subsection:coverage_lower_bound}

The following one-action example isolates the coverage obstruction and is also
a special case of off-policy evaluation.
Fix $0<p\leq1/2$ and $0<\alpha\leq1/2$.
On the state space $\{0,1\}$, consider
\begin{equation}\label{eq:lower_two_transition}
\bP_{\alpha,p}\coloneq
\begin{pmatrix}
1-\alpha p&\alpha p\\
\alpha(1-p)&1-\alpha(1-p)
\end{pmatrix}.
\end{equation}
The stationary distribution is $\bmu=(1-p,p)^\top$, so $\mumin=p$.
The chain is reversible, and the two mean hitting times are
$1/(\alpha p)$ and $1/(\alpha(1-p))$; hence
$\mathfrak h_\star=1/(\alpha p)$.
Writing $\bm\Pi\coloneq\bone\bmu^\top$, direct multiplication gives
\[
\bP_{\alpha,p}^t=\bm\Pi+(1-\alpha)^t(\bI-\bm\Pi),
\qquad
d_{\mathrm{TV}}(t)=(1-p)(1-\alpha)^t.
\]
Consequently,
\begin{equation}\label{eq:lower_two_mixing}
\tmix=\left\lceil
\frac{\log(4(1-p))}{-\log(1-\alpha)}
\right\rceil\asymp\alpha^{-1},
\qquad \tmix\leq\frac2\alpha.
\end{equation}
For the last inequality, use $-\log(1-\alpha)\geq\alpha$ and
$\log4+\alpha<2$.

\begin{proposition}[Information-theoretic coverage obstruction]
\label{proposition:coverage_lower_bound}
Consider two models with the same transition matrix \eqref{eq:lower_two_transition},
deterministic rewards $r_0(s)=0$ and $r_1(s)=\ind\{s=1\}$, and initial state $S_0=0$.
Let $\bV_\theta$ be the value function in model $\theta\in\{0,1\}$.
If
\[
2\varepsilon<pH+\frac{1-p}{1-\gamma+\gamma\alpha},
\]
then every estimator $\widehat{\bV}_T$ satisfies
\begin{equation}\label{eq:lower_two_testing}
\max_{\theta\in\{0,1\}}
\PB_\theta\prn{\norm{\widehat{\bV}_T-\bV_\theta}_\infty>\varepsilon}
\geq\frac12(1-\alpha p)^T.
\end{equation}
The estimator may know the transition matrix, the stationary distribution, and both candidate reward functions.
In particular, if $H\geq8$, $0<\varepsilon\leq1$, $\alpha\leq1/(2H)$,
and the estimator has failure probability at most $\delta<1/2$ in both models, then
\begin{equation*}
T\geq\frac{\mathfrak h_\star}{2}\log\frac1{2\delta}
\geq\frac{\tmix}{4\mumin}\log\frac1{2\delta}.
\end{equation*}
\end{proposition}

\begin{proof}
The transition matrix acts as the identity on constant vectors and as multiplication by $1-\alpha$ on vectors with stationary mean zero.
Decomposing the reward vector into these two components yields
\[
\bV_0=\bm0,
\qquad
V_1(1)=pH+\frac{1-p}{1-\gamma+\gamma\alpha}.
\]
Thus the two closed sup-norm balls of radius $\varepsilon$ about the value functions are disjoint.

Consider the common event
\[
E_T\coloneq\{S_0=\cdots=S_T=0\}.
\]
Its probability is $(1-\alpha p)^T$ in both models.
On this event every observed reward is zero, and the full data have the same distribution under the two models.
An estimator receiving these data cannot be accurate for both separated value functions.
More precisely, its two conditional failure probabilities, including any independent randomization, sum to at least one.
Multiplying by the common probability of $E_T$ proves \eqref{eq:lower_two_testing}.

Under the additional parameter restrictions,
$1-\gamma+\gamma\alpha\leq3/(2H)$ and $1-p\geq1/2$.
Hence $V_1(1)\geq H/3>2\varepsilon$.
If both failure probabilities are at most $\delta$, then
\[
T\geq\frac{\log(1/(2\delta))}{-\log(1-\alpha p)}
\geq\frac1{2\alpha p}\log\frac1{2\delta},
\]
where $-\log(1-x)\leq2x$ for $0\leq x\leq1/2$.
Equation~\eqref{eq:lower_two_mixing} completes the proof.
\end{proof}

The event in this proof describes a lack of information about the rare state.
It also gives a direct lower bound for TD learning: if $r\equiv0$ and $\bV_0=H\bone$, then coordinate $1$ remains exactly $H$ on $E_T$, whatever step-size schedule is used.
For a stationary initial state, the right-hand side of \eqref{eq:lower_two_testing} is multiplied by $1-p$.
Thus the coverage obstruction persists even when the trajectory starts in equilibrium.

The cover-time example in \citet[Appendix A.2]{li2022async} already establishes that a cost of order $\tmix/\mumin$ can be necessary for state coverage, up to logarithms.
The argument above turns this mechanism into an estimation lower bound and makes its confidence dependence explicit.

\subsection{A joint lower bound with known rewards}\label{subsection:joint_lower_bound}

The two-state example leaves the reward at state $1$ unknown until that state is visited.
We now use fixed, known policies and a fixed, known reward function, and vary only transitions.
The construction contains two rare states: entering their set controls the coverage cost, while transitions within the set control the statistical difficulty.
The behavior policy selects the informative target action with probability
$1/\varrho$. A separate refresh mechanism controls behavior-chain
coverage while preserving the order of the target value's statistical sensitivity.

Fix
\begin{equation}\label{eq:lower_joint_parameters}
H\geq64,\qquad 0<p\leq\frac14,\qquad
\varrho\geq1,\qquad 0<\varepsilon\leq1,\qquad
0<\delta\leq\frac14,
\end{equation}
and choose $\alpha$ in either of the following ranges:
\begin{equation}\label{eq:lower_joint_alpha_range}
\begin{cases}
0<\alpha\leq1/(8H),&\text{for every }\varrho\geq1,\\
0<\alpha\leq1/2,&\text{if }\varrho\geq2.
\end{cases}
\end{equation}
On the states $\{0,1,2\}$ and actions $\{0,1\}$, set
\[
\pi(1\mid s)=1,\qquad b(1\mid s)=\frac1\varrho,
\qquad
\bmu=(1-2p,p,p)^\top,\qquad
\bm\Pi=\bone\bmu^\top,\qquad
\br=(0,1,0)^\top.
\]
Both policies are independent of the state and the unknown parameter.
Rewards are $R_t=r(S_t)$ under both actions, and $S_0=0$.
The importance weight is $\varrho$ for action $1$ and zero for action
$0$; when $\varrho=1$, only action $1$ is sampled.

For $\theta\in[0,1/2]$, define
\[
\bB_\theta\coloneq
\begin{pmatrix}
1&0&0\\
0&1-\theta&\theta\\
0&\theta&1-\theta
\end{pmatrix},
\qquad
\beta\coloneq\min\brc{\alpha,\frac1{8H}},
\qquad
\alpha_0\coloneq
\begin{cases}
(\varrho\alpha-\beta)/(\varrho-1),&\varrho>1,\\
\alpha,&\varrho=1.
\end{cases}
\]
The controlled transition kernels and the induced target and behavior
kernels are
\begin{equation}\label{eq:lower_joint_family}
\begin{aligned}
\bm K_{\theta,1}&\coloneq(1-\beta)\bB_\theta+\beta\bm\Pi,\qquad
\bm K_{\theta,0}\coloneq(1-\alpha_0)\bI+\alpha_0\bm\Pi,\\
\bQ_\theta&=\bm K_{\theta,1},\\
\bP_\theta&=\varrho^{-1}\bm K_{\theta,1}
 +(1-\varrho^{-1})\bm K_{\theta,0}.
\end{aligned}
\end{equation}
These are valid kernels. Indeed, in the first range of
\eqref{eq:lower_joint_alpha_range}, $\beta=\alpha$ and
$\alpha_0=\alpha$. In the second range, $0\leq\alpha_0\leq1$ follows
from $\beta\leq\alpha\leq1/2$ and $\varrho\geq2$.
When $\varrho=1$, $\beta=\alpha$ and the unused kernel does not affect
any observation or value.

The behavior kernel can be written as
\begin{equation}\label{eq:lower_behavior_refresh}
\bP_\theta=(1-\alpha)\bB_{\widetilde\theta}+\alpha\bm\Pi,
\qquad
\widetilde\theta\coloneq
\frac{(1-\beta)\theta}{\varrho(1-\alpha)}\in[0,1/2].
\end{equation}
To check the displayed range, its multiplier is $1/\varrho$ when
$\beta=\alpha$, and is at most one when
$\varrho\geq2$ and $\alpha\leq1/2$.
All entries contributed by $\alpha\bm\Pi$ are positive.
Every chain is therefore irreducible and aperiodic, and all have stationary distribution $\bmu$ with $\mumin=p$.
Both $\bB_{\widetilde\theta}$ and $\bm\Pi$ satisfy detailed balance with respect to
$\bmu$, so $\bP_\theta$ is reversible and its reverse hitting times agree
with its forward hitting times. Uniformly for $\theta\in[0,1/2]$,
\begin{equation}\label{eq:lower_joint_hitting}
\frac1{2\alpha p}
\leq\mathfrak h_\star(\bP_\theta)
\leq\frac1{\alpha p}.
\end{equation}
Indeed, from state $0$, the time to enter $A\coloneq\{1,2\}$ is
geometric with success probability $2\alpha p$.
The time to hit state $1$ is at least this entrance time, proving the
lower bound. For every target state $y$, each transition has conditional
probability at least $\alpha\mu(y)$ of reaching $y$.
Thus its mean hitting time from any state is at most
$1/(\alpha\mu(y))\leq1/(\alpha p)$, proving the upper bound.

We also relate $\alpha$ to the mixing time uniformly throughout the family.
The identities $\bB_{\widetilde\theta}\bm\Pi
=\bm\Pi\bB_{\widetilde\theta}=\bm\Pi$ imply
\[
\bP_\theta^t=(1-\alpha)^t\bB_{\widetilde\theta}^t
 +\bigl(1-(1-\alpha)^t\bigr)\bm\Pi.
\]
The set $A$ is closed under $\bB_{\widetilde\theta}$.
Starting from either state in $A$, the probability of $A$ exceeds its stationary probability by $(1-2p)(1-\alpha)^t$.
It follows that
\begin{equation}\label{eq:lower_joint_mixing}
\begin{split}
(1-2p)(1-\alpha)^t
&\leq d_{\mathrm{TV},\theta}(t)\leq(1-\alpha)^t,\\
\frac{\log2}{2\alpha}
&\leq\tmix(\bP_\theta)\leq\frac2\alpha.
\end{split}
\end{equation}
Indeed, the lower mixing bound uses $1-2p\geq1/2$ and $-\log(1-\alpha)\leq2\alpha$;
the upper bound follows as in \eqref{eq:lower_two_mixing}.
Together, \eqref{eq:lower_joint_hitting} and \eqref{eq:lower_joint_mixing}
show that the family fixes $H$, $\varrho$, and $\mumin$ exactly and fixes the orders
of both $\mathfrak h_\star$ and $\tmix$ up to universal constants.

\begin{theorem}[Joint statistical and coverage lower bound]
\label{theorem:joint_lower_bound}
Under \eqref{eq:lower_joint_parameters}--\eqref{eq:lower_joint_alpha_range},
suppose an estimator based on $T$ behavior-policy transitions satisfies
\[
\sup_{\theta\in[0,1/2]}
\PB_\theta\prn{\norm{\widehat{\bV}_T-\bV_\theta}_\infty>
\varepsilon}\leq\delta,
\qquad
\bV_\theta\coloneq(\bI-\gamma\bQ_\theta)^{-1}\br.
\]
Then
\begin{equation}\label{eq:lower_joint_sample_size}
T\geq\frac1{18432}
\prn{\frac{\varrho H^3}{p\varepsilon^2}+\frac1{p\alpha}}
\log\frac1{2\delta}.
\end{equation}
The estimator may know $H,p,\varrho,\alpha$, both policies, the family \eqref{eq:lower_joint_family}, and the reward function.
\end{theorem}

\begin{proof}
We use two pairs of models from the same family.
One pair has well-separated value functions but identical observations until a rare state is visited.
The other pair has nearby transition probabilities and quantifies the statistical difficulty of estimating values from visits to the rare states.

First record the dependence of the target value on $\theta$.
Put
\[
u\coloneq\gamma(1-\beta),\qquad
D_\theta\coloneq1-u(1-2\theta).
\]
The contrast $\be_1-\be_2$ is a right eigenvector of $\bQ_\theta$
with eigenvalue $(1-\beta)(1-2\theta)$.
The other component of the reward, $(0,1/2,1/2)^\top$, is constant on $A$.
The subspace of vectors constant on $A$ is invariant under $\bQ_\theta$, and the restriction of $\bQ_\theta$ to this subspace is independent of $\theta$.
Decomposing the reward into these two components therefore gives
\begin{equation}\label{eq:lower_joint_value_contrast}
\bV_\theta=\bm W+\frac1{2D_\theta}(\be_1-\be_2),
\end{equation}
where $\bm W$ is independent of $\theta$.

\paragraph{The coverage term.}
Compare $\theta=0$ and $\theta=1/2$.
Since $1-u\leq H^{-1}+\beta\leq9/(8H)$, their target-value difference at state $1$ is
\[
\frac12\prn{\frac1{1-u}-1}
\geq\frac12\prn{\frac{8H}{9}-1}>2\varepsilon.
\]
Starting from $0$, the probability of remaining there for all $T$ transitions is
\[
\PB_\theta(E_T)=(1-2p\alpha)^T,
\qquad E_T\coloneq\{S_0=\cdots=S_T=0\}.
\]
The probability of $E_T$ and the conditional law of the observations, including the actions, given $E_T$ are independent of $\theta$:
both action kernels have identical rows at state $0$ throughout the family.
The disjoint-ball argument from Proposition~\ref{proposition:coverage_lower_bound} thus gives
\[
\delta\geq\frac12(1-2p\alpha)^T,
\qquad
T\geq\frac1{4p\alpha}\log\frac1{2\delta}.
\]
Here we used $-\log(1-2p\alpha)\leq4p\alpha$.

\paragraph{The statistical term.}
Now compare the parameters
\[
\theta_0\coloneq\frac1{4H},\qquad
\theta_1\coloneq\theta_0+\Delta,\qquad
\Delta\coloneq\frac{16\varepsilon}{H^2}.
\]
The restrictions $H\geq64$ and $\varepsilon\leq1$ imply $\theta_1\leq1/(2H)$.
Also $u\geq3/4$ and
\[
D_{\theta_j}=1-u+2u\theta_j
\leq\frac{9}{8H}+\frac1H
\leq\frac9{4H},\qquad j\in\{0,1\}.
\]
Using \eqref{eq:lower_joint_value_contrast}, we obtain
\begin{equation}\label{eq:lower_joint_statistical_gap}
\norm{\bV_{\theta_1}-\bV_{\theta_0}}_\infty
=\frac{u\Delta}{D_{\theta_0}D_{\theta_1}}
\geq\frac{64}{27}\varepsilon>2\varepsilon.
\end{equation}
Thus distinguishing these nearby transition matrices is necessary for the requested value accuracy.

Write $D_{\mathrm{KL}}(v\|w)$ for relative entropy between probability distributions,
and $\mathsf P_\theta^{(T)}$ for the law of the full observation under model $\theta$.
Only the target-action rows indexed by $A$ differ between the two models.
Within each such row, the probabilities of states $1$ and $2$ change by an absolute amount of $(1-\beta)\Delta\leq\Delta$.
Under $\bm K_{\theta_1,1}$, the self-transition probability is at least $1/2$, and the probability of the other rare state is at least $1/(8H)$.
The elementary inequality
$D_{\mathrm{KL}}(v\|w)\leq\sum_j(v_j-w_j)^2/w_j$ gives
\[
D_{\mathrm{KL}}\prn{
\bm K_{\theta_0,1}(i,\cdot)\|\bm K_{\theta_1,1}(i,\cdot)}
\leq(2+8H)\Delta^2\leq9H\Delta^2,
\qquad i\in A.
\]
The behavior policy chooses this action with probability $1/\varrho$;
all other conditional observation laws are identical in the two
models.
The process recording membership in $A$ has the same law for every $\theta$.
Starting from $0$, the formula for $\bP_\theta^t$ yields
\[
\PB_\theta(S_t\in A)=2p\bigl(1-(1-\alpha)^t\bigr)\leq2p.
\]
Therefore the chain rule for relative entropy, summed over transitions at times $0,\ldots,T-1$, gives
\begin{align}
D_{\mathrm{KL}}\prn{
\mathsf P_{\theta_0}^{(T)}\|\mathsf P_{\theta_1}^{(T)}}
&=\sum_{t=0}^{T-1}\E_{\theta_0}\brk{
\sum_{a=0}^1 b(a\mid S_t)
D_{\mathrm{KL}}\prn{\bm K_{\theta_0,a}(S_t,\cdot)
\|\bm K_{\theta_1,a}(S_t,\cdot)}}\notag\\
&\leq\frac{2pT}{\varrho}\,9H\Delta^2
=4608\frac{pT\varepsilon^2}{\varrho H^3}.
\label{eq:lower_joint_path_entropy}
\end{align}
The initial state, action probabilities, and known deterministic
reward function contribute no additional relative entropy.

To relate this information bound to estimation error, let $B$ be the event that the estimator lies in the closed radius-$\varepsilon$ ball about $\bV_{\theta_0}$.
The assumed guarantees and \eqref{eq:lower_joint_statistical_gap} imply
\[
\mathsf P_{\theta_0}^{(T)}(B)\geq1-\delta,
\qquad \mathsf P_{\theta_1}^{(T)}(B)\leq\delta.
\]
For a randomized estimator, adjoin its independent random seed to the observation; this leaves the relative entropy unchanged.
Applying the log-sum inequality to $B$ and its complement then gives
\begin{align*}
D_{\mathrm{KL}}\prn{\mathsf P_{\theta_0}^{(T)}\|\mathsf P_{\theta_1}^{(T)}}
&\geq(1-2\delta)\log\frac{1-\delta}{\delta}\\
&\geq\frac12\log\frac1{2\delta},
\end{align*}
where $\delta\leq1/4$.
Combining this with \eqref{eq:lower_joint_path_entropy} yields
\[
T\geq\frac{\varrho H^3}{9216p\varepsilon^2}
\log\frac1{2\delta}.
\]
Both necessary bounds concern the same family with the same $H,p,\varrho,\alpha$.
Their maximum is at least half their sum.
Weakening the larger coefficient of the coverage term to the coefficient of the statistical term proves \eqref{eq:lower_joint_sample_size}.
\end{proof}

\subsection{Minimax optimality under hitting-time and mixing-time budgets}
\label{subsection:budget_minimax}

We first formulate the minimax result in terms of the reverse hitting
time, matching the main upper bounds.
For prescribed $H,p,\varrho$ as above, fix the state and action spaces,
policies, reward function and initial state used in the preceding
construction. Let $\mathcal M$ range over all pairs of controlled
transition kernels on these spaces. Write $\bP_{\mathcal M}$ and
$\bQ_{\mathcal M}$ for their behavior and target kernels.
Define
\begin{equation}\label{eq:lower_hitting_class}
\mathcal C_{\mathrm{hit}}(H,p,\varrho,\mathsf h)
\coloneq\left\{\mathcal M:
\begin{array}{l}
\bP_{\mathcal M}\text{ is irreducible},\\
\mumin(\bP_{\mathcal M})\geq p,\quad
\mathfrak h_\star(\bP_{\mathcal M})\leq\mathsf h
\end{array}\right\}.
\end{equation}
Every model uses discount factor $\gamma=1-H^{-1}$,
the known reward vector $\br=(0,1,0)^\top$, and initial state $S_0=0$.
For any such class $\mathcal C$, define
$N_\star(\varepsilon,\delta;\mathcal C)$ as the smallest integer $T$
for which there exists a possibly randomized estimator based on $T$
behavior transitions with
\[
\sup_{\mathcal M\in\mathcal C}
\PB_{\mathcal M}\prn{
\norm{\widehat{\bV}_T-(\bI-\gamma\bQ_{\mathcal M})^{-1}\br}_\infty
>\varepsilon}\leq\delta.
\]
The estimator may use the parameters defining the class, both known policies, and the known
reward function.

\begin{theorem}[Minimax rate under a hitting-time budget]
\label{theorem:hitting_minimax}
Let \eqref{eq:lower_joint_parameters} hold. Suppose either
$\mathsf h\geq8H/p$, or $\varrho\geq2$ and $\mathsf h\geq2/p$.
For the class \eqref{eq:lower_hitting_class},
\begin{equation}\label{eq:lower_hitting_minimax_bound}
N_\star(\varepsilon,\delta;\mathcal C_{\mathrm{hit}})
\geq\frac1{18432}
\prn{\frac{\varrho H^3}{p\varepsilon^2}+\mathsf h}
\log\frac1{2\delta}.
\end{equation}
Importance-weighted last-iterate TD learning, with either a constant step size or the
decreasing schedule of Theorem~\ref{theorem:decay}, using $p$ in place
of $\mumin$, achieves the corresponding upper bound up to logarithmic
factors. Consequently,
\[
N_\star(\varepsilon,\delta;\mathcal C_{\mathrm{hit}})
=\widetilde\Theta\prn{
\frac{\varrho H^3}{p\varepsilon^2}+\mathsf h}.
\]
\end{theorem}

\begin{proof}
For the lower bound, take $\alpha=1/(p\mathsf h)$ in
\eqref{eq:lower_joint_family}.
The assumed alternatives ensure the corresponding admissible
ranges in \eqref{eq:lower_joint_alpha_range}.
Every member of the family has $\mumin=p$ and
$\mathfrak h_\star\leq1/(\alpha p)=\mathsf h$ by
\eqref{eq:lower_joint_hitting}, so the family lies in
$\mathcal C_{\mathrm{hit}}(H,p,\varrho,\mathsf h)$.
Theorem~\ref{theorem:joint_lower_bound} then gives
\eqref{eq:lower_hitting_minimax_bound}.

For the upper bound, the parameters can be chosen uniformly over the
class. Set
\[
\overline L\coloneq\log\frac{24}{\delta p},
\qquad
\overline\eta\coloneq\frac{\varepsilon^2}{1024\varrho H^2\overline L^2},
\qquad
\overline T\coloneq
\left\lceil8\prn{\frac{H}{\overline\eta p}+\mathsf h}\right\rceil
\left\lceil\log_2\frac{12H}{\delta\varepsilon p}\right\rceil.
\]
Use $\bV_0=\bm0$ and the constant step size $\overline\eta$.
In Proposition~\ref{prop:stochastic_convolution} with $q=\delta/2$,
the noise logarithm is at most $\overline L$, so the stochastic error
is at most $\varepsilon/2$ with probability at least $1-\delta/2$,
as in the proof of Theorem~\ref{theorem:main}.
The inequalities $\mumin\geq p$ and $\mathfrak h_\star\leq\mathsf h$
show that the block length and number of blocks in
Proposition~\ref{proposition:initialization} are bounded by the two
factors defining $\overline T$.
Thus every deterministic integer terminal time $T\geq\overline T$
gives error at most $\varepsilon$ with probability at least $1-\delta$,
uniformly over the class.
Substitution of $\overline\eta$ gives the claimed upper order.

For decreasing steps, set
\[
\overline L_t\coloneq\log\frac{48\varrho H(t+8)}p,
\qquad
\eta_t\coloneq\min\brc{\frac1{\varrho H},
                 \frac{4096H\overline L_t}{p(t+8)}},
\qquad
L_t\coloneq\overline L_t+\log(1/\delta).
\]
To apply the interval bounds underlying Theorem~\ref{theorem:decay},
observe that
\[
\frac{H}{\eta_{T-1}\mumin}
\leq\frac{H}{\eta_{T-1}p}
=\max\brc{\frac{\varrho H^2}{p},
           \frac{T+7}{4096\overline L_{T-1}}}.
\]
The density factor satisfies $1/\mumin\leq1/p$, and for every interval
failure probability $q\in(0,1)$,
$\log(4d/(q\mumin))\leq\log(4d/(qp))$.
The interval noise and forgetting bounds in
Appendix~\ref{Appendix:decay} therefore apply with these upper bounds
in terms of $p$ and with $\mathfrak h_\star\leq\mathsf h$.
Consequently, the proof of \eqref{eq:decay_sample} gives a guarantee
uniform over the class whenever
\[
T\geq C\prn{\frac{\varrho H^3\overline L_T L_T^2}{p\varepsilon^2}
                    +\mathsf h L_T},
\]
which has the same upper order.
\end{proof}

For comparison with mixing-time guarantees, define
\begin{equation}\label{eq:lower_budget_class}
\mathcal C_{\mathrm{mix}}(H,p,\varrho,\mathsf t)
\coloneq\left\{\mathcal M:
\begin{array}{l}
\bP_{\mathcal M}\text{ is irreducible and aperiodic},\\
\mumin(\bP_{\mathcal M})\geq p,\quad
\tmix(\bP_{\mathcal M})\leq\mathsf t
\end{array}\right\}.
\end{equation}
The same family gives a matching lower bound for this class.

\begin{corollary}[Minimax rate under a mixing-time budget]
\label{corollary:budget_minimax}
Let \eqref{eq:lower_joint_parameters} hold. Suppose either
$\mathsf t\geq16H$, or $\varrho\geq2$ and $\mathsf t\geq4$.
For the class \eqref{eq:lower_budget_class},
\begin{equation}\label{eq:lower_budget_minimax_bound}
N_\star(\varepsilon,\delta;\mathcal C_{\mathrm{mix}})
\geq\frac1{36864}
\prn{\frac{\varrho H^3}{p\varepsilon^2}+\frac{\mathsf t}{p}}
\log\frac1{2\delta}.
\end{equation}
Importance-weighted last-iterate TD learning with either constant or decreasing step sizes
attains the corresponding upper bound up to logarithmic factors, so
\[
N_\star(\varepsilon,\delta;\mathcal C_{\mathrm{mix}})
=\widetilde\Theta\prn{
\frac{\varrho H^3}{p\varepsilon^2}+\frac{\mathsf t}{p}}.
\]
\end{corollary}

\begin{proof}
For the lower bound, set $\alpha=2/\mathsf t$ in
\eqref{eq:lower_joint_family}.
The assumptions ensure \eqref{eq:lower_joint_alpha_range}, and
\eqref{eq:lower_joint_mixing} gives
$\tmix(\bP_\theta)\leq2/\alpha=\mathsf t$.
Thus the family belongs to $\mathcal C_{\mathrm{mix}}(H,p,\varrho,\mathsf t)$,
and Theorem~\ref{theorem:joint_lower_bound} implies
\eqref{eq:lower_budget_minimax_bound}.
For the upper bound, Lemma~\ref{lemma:reverse_hitting_mixing} gives
\[
\mathcal C_{\mathrm{mix}}(H,p,\varrho,\mathsf t)
\subseteq\mathcal C_{\mathrm{hit}}(H,p,\varrho,\overline{\mathsf h}),
\qquad
\overline{\mathsf h}\coloneq
\frac{2\lceil\mathsf t\rceil}{p}
\left\lceil\log_2\frac2p\right\rceil.
\]
The upper bound of Theorem~\ref{theorem:hitting_minimax} on this larger
class proves the claim.
\end{proof}

\begin{remark}[Stationary initialization and the on-policy limit]
If the three-state family starts from its common stationary distribution, the common event in the coverage argument has probability
$(1-2p)(1-2p\alpha)^T$, whereas the expected number of transitions whose starting state lies in $A$ is exactly $2pT$.
The statistical calculation is therefore unchanged, and the coverage lower bound holds with
$\log(1/(2\delta))$ replaced by $\log(1/(4\delta))$ for $\delta<1/4$.
In particular, for $\delta\leq1/8$, these logarithms are comparable by
universal constants. Both minimax orders in
Theorem~\ref{theorem:hitting_minimax} and
Corollary~\ref{corollary:budget_minimax} therefore also hold when each trajectory starts from the stationary distribution of its behavior chain.
At $\varrho=1$ the construction has $\bP_\theta=\bQ_\theta$,
$\beta=\alpha$, and recovers the on-policy lower bounds in the ranges
$\mathsf h\geq8H/p$ and $\mathsf t\geq16H$.
The larger parameter ranges for $\varrho\geq2$ follow because the
nontarget action can refresh the behavior chain without making the
target chain mix quickly.
\end{remark}

\section{Discussion}
\label{Section:discussion}
With either a constant step size chosen for the target accuracy or the
decreasing schedule in Theorem~\ref{theorem:decay}, importance-weighted asynchronous
TD learning achieves sup-norm error at most $\varepsilon$, for $0<\varepsilon\leq1$, with high probability using
$\wtilde O(\varrho H^3/(\mumin\varepsilon^2)+\mathfrak h_\star)$ transitions.
Two properties of the exact random propagation explain this rate.
The anchored local Poisson construction controls accumulated noise without a mixing-time factor, while compensation of successive hitting costs bounds the initialization transient by a single reverse hitting-time term.
The guarantees allow periodic behavior chains, and the empirical state-count rule
retains the same rate when $\mumin$ is unknown.
For aperiodic behavior chains, the relation between hitting and mixing times recovers the bound
$\wtilde O(\varrho H^3/(\mumin\varepsilon^2)+\tmix/\mumin)$.
The three-state, two-action lower bound establishes minimax optimality, up to logarithms,
on the parameter classes with hitting-time budgets in Theorem~\ref{theorem:hitting_minimax}
and the mixing-time budgets in Corollary~\ref{corollary:budget_minimax}.
Because that construction uses a common known reward function, the lower bound concerns the information in the observed transitions as well as the time needed to reach rare states.

The same reverse-time analysis applies to importance-weighted nonnegative affine prediction
updates, whose components may interact through a random matrix on each
observed transition.
The variance certificate in Lemma~\ref{lemma:affine_variance_certificate}
and the propagation comparison in Proposition~\ref{prop:affine_domination}
isolate the properties needed to retain the
cubic statistical term and additive hitting-time cost.
Successor features and generalized value predictions inherit this rate;
the perturbation argument in Section~\ref{subsection:return_variance}
also gives a guarantee for simultaneous mean and variance estimation.

{\setlength{\bibsep}{2pt}
\bibliographystyle{abbrvnat}
\bibliography{ref}}

\clearpage
\appendix

\section{Reversal of the Marked Behavior Process}
\label{Appendix:reverse}
\subsection{The stationary reverse marked law}

The reverse-time argument requires the conditional marked law given the full reverse history, not only the transition matrix of the reversed state chain. The following lemma records both in a form that permits dependence among the components of the mark and the next state.

\begin{lemma}[Stationary reverse marked kernel]\label{lemma:reverse_kernel}
Let $\Omega_t$ denote the complete forward edge mark, with conditional
kernel $\mathsf K(\mathrm d\omega\mid y,x)$ given
$(S_t,S_{t+1})=(y,x)$. Assume that its forward full-history law is
\[
 \PB(S_{t+1}=x,\Omega_t\in\mathrm d\omega\mid\gF_t)
 =P(S_t,x)\mathsf K(\mathrm d\omega\mid S_t,x).
\]
Here $\gF_t$ is the complete forward history through $S_t$, including
all past edge marks. For scalar TD learning, the mark contains the
action, reward, and importance weight $W_t$; for the affine framework,
take $\Omega_t=(W_t,\bm C_t,\bm A_t)$. Suppose $S_0\sim\bmu$.
On an edge with $P(y,x)=0$, choose any probability law for
$\mathsf K(\cdot\mid y,x)$; this choice never affects the path law.
Fix $T\geq1$ and set
\[
 X_n=S_{T-n},\qquad \omega_{n+1}=\Omega_{T-n-1},\qquad
 \gG_n=\sigma(X_0,\omega_1,X_1,\ldots,\omega_n,X_n).
\]
Then $X_0\sim\bmu$ and, for every $0\leq n<T$,
\begin{equation}\label{eq:reverse_kernel_conditional}
 \PB(X_{n+1}=y,\omega_{n+1}\in\mathrm d\omega\mid\gG_n)
 =P^\star(X_n,y)\mathsf K(\mathrm d\omega\mid y,X_n),
\end{equation}
where $P^\star(x,y)=\mu(y)P(y,x)/\mu(x)$. Moreover, $\bP^\star$ is a transition matrix with stationary distribution $\bmu$ and is irreducible whenever $\bP$ is irreducible.
\end{lemma}

\begin{proof}
First, stationarity of $\bmu$ for the forward chain implies
\begin{equation*}
    \sum_yP^\star(x,y)
    =\frac{\sum_y\mu(y)P(y,x)}{\mu(x)}=1.
\end{equation*}
Likewise,
\begin{equation*}
    \sum_x\mu(x)P^\star(x,y)
    =\mu(y)\sum_xP(y,x)=\mu(y).
\end{equation*}
Thus $\bP^\star$ is stochastic and has stationary distribution $\bmu$. Reversing any positive-probability path for the irreducible forward chain gives a positive-probability path for $\bP^\star$, so the reverse chain is also irreducible.

To verify the marked law, take a finite forward path $(s_0,\omega_0,s_1,\ldots,\omega_{T-1},s_T)$. The marked Markov property gives its joint measure as
\begin{equation*}
    \mu(s_0)\prod_{t=0}^{T-1}
       P(s_t,s_{t+1})\mathsf K(\mathrm{d}\omega_t\mid s_t,s_{t+1}).
\end{equation*}
The identity $\mu(s_t)P(s_t,s_{t+1})
=\mu(s_{t+1})P^\star(s_{t+1},s_t)$ telescopes across the state-transition factors. Consequently, the same measure equals
\begin{equation}\label{eq:reverse_path_factorization}
    \mu(s_T)\prod_{t=0}^{T-1}
       P^\star(s_{t+1},s_t)\mathsf K(\mathrm{d}\omega_t\mid s_t,s_{t+1}).
\end{equation}
This is an identity of measures, so it does not assume that the mark law has a density or is discrete. Reindexing with $x_n=s_{T-n}$ and $\omega_{n+1}=\Omega_{T-n-1}$ rewrites~\eqref{eq:reverse_path_factorization} as
\begin{equation*}
    \mu(x_0)\prod_{n=0}^{T-1}
       P^\star(x_n,x_{n+1})
       \mathsf K(\mathrm{d}\omega_{n+1}\mid x_{n+1},x_n).
\end{equation*}
Each factor is a probability kernel. Integrating the factors after time $n+1$ and then conditioning on the first $n$ factors proves~\eqref{eq:reverse_kernel_conditional}. In particular, the marks already included in $\gG_n$ do not change the conditional law of the next reverse edge once $X_n$ is known.
\end{proof}

The same reverse kernel defines an abstract process of any duration, started from an arbitrary state $x$. For scalar TD learning with a constant step size $\eta\leq1/\varrho$, appending the weight update~\eqref{eq:decay_weights} with $a_n=\eta$ makes $(\bp_n,X_n)$ a time-homogeneous Markov process on
\begin{equation*}
    \brc{\bp\in\RB^d:\bp\geq\bm0,\ \norm{\bp}_1\leq1}
       \times\gS.
\end{equation*}
Its transition kernel integrates the full edge mark. The expectation bounds in the main proof apply to every initial point of this state space. This formulation is what permits their conditional application after an intermediate reverse time.
For a deterministic changing step size, the pair is a time-inhomogeneous
Markov process; adjoining the time index gives the time-homogeneous
process $(n,\bp_n,X_n)$ used in Section~\ref{Section:reverse}.
For the block recursion, the weight vector has $dm$ coordinates and its
update also uses the matrix in the sampled edge mark.

\section{Exponential Bounds from Expected Energy}
\label{Appendix:exponential}
\subsection{A discrete exponential-moment lemma}
\label{app:exponential_moments}

The following discrete Khasminskii bound converts a uniform expectation
bound for a nonnegative additive functional into an exponential-moment bound.

\begin{lemma}[Discrete Khasminskii bound]
\label{lemma:discrete_khasminskii}
Let $(Z_n)_{n\ge0}$ be a time-homogeneous Markov process on a measurable
state space, and let $f$ be a nonnegative measurable function. Suppose
that, for some $K\in(0,\infty)$,
\[
 \sup_z\sup_{N\ge1}
 \E_z\sum_{n=0}^{N-1}f(Z_n)\le K.
\]
Writing $A_N\coloneq\sum_{n=0}^{N-1}f(Z_n)$, we have, for every
initial state $z$, finite $N$, and integer $k\ge1$,
\begin{equation}
 \E_z A_N^k\le k!K^k.
 \label{eq:khas_moments}
\end{equation}
Consequently,
\begin{equation}
 \E_z\exp\!\left(\frac{A_N}{2K}\right)\le2,
 \qquad
 \PB_z(A_N>2K\ell)\le2e^{-\ell}
 \quad(\ell>0).
 \label{eq:khas_exponential}
\end{equation}
\end{lemma}

\begin{proof}
The assumption with $N=1$ gives $f(z)\le K$ for every $z$, so all
moments over a fixed finite horizon are finite. We prove
\cref{eq:khas_moments} by induction on $k$, simultaneously over all
initial states and all finite horizons. The case $k=1$ is the
assumption.

For $k\ge2$, fix $z$ and $N$ and let
\[
 B_n\coloneq\sum_{r=n}^{N-1}f(Z_r),\qquad 0\le n\le N,
 \quad B_N=0.
\]
Since $B_n-B_{n+1}=f(Z_n)$ and $B_n\ge B_{n+1}\ge0$,
\[
 B_n^k-B_{n+1}^k
 =k\int_{B_{n+1}}^{B_n}u^{k-1}\,\mathrm{d}u
 \le kf(Z_n)B_n^{k-1}.
\]
Let $\mathcal H_n\coloneq\sigma(Z_0,\ldots,Z_n)$. The Markov
property and the induction hypothesis, applied with initial state
$Z_n$ and remaining horizon $N-n$, imply
\[
 \E[B_n^{k-1}\mid\mathcal H_n]
 \le(k-1)!K^{k-1}.
\]
Taking expectations in the telescoping inequality now gives
\begingroup
\allowdisplaybreaks[0]
\begin{align*}
 \E_z A_N^k
 &=\E_z\sum_{n=0}^{N-1}(B_n^k-B_{n+1}^k)\\
 &\le k\sum_{n=0}^{N-1}
   \E_z\!\left[
     f(Z_n)\E[B_n^{k-1}\mid\mathcal H_n]
   \right]\\
 &\le k!K^{k-1}\E_z A_N
 \le k!K^k.
\end{align*}
\endgroup
The induction includes the occurrence of $f(Z_n)$ inside $B_n$,
and hence terms with repeated time indices.

Expanding the exponential into its nonnegative power series and
using monotone convergence and \cref{eq:khas_moments}, we obtain
\[
 \E_z\exp\!\left(\frac{A_N}{2K}\right)
 =\sum_{k=0}^{\infty}\frac{\E_z A_N^k}{k!(2K)^k}
 \le\sum_{k=0}^{\infty}2^{-k}=2.
\]
Markov's inequality yields the tail bound in
\cref{eq:khas_exponential}.
\end{proof}

\subsection{A scalar martingale tail bound}
\label{app:scalar_martingale}

We record the fixed-horizon form of Freedman's inequality \citep{freedman1975tail} used in
the proof and give its exponential-supermartingale derivation.
This also makes explicit the probability cost of controlling a
random predictable quadratic variation.

\begin{lemma}[Scalar Freedman bound]
\label{lemma:scalar_freedman}
Let $(D_n)_{n=1}^N$ be real martingale differences with respect to a
filtration $(\mathcal H_n)_{n=0}^N$, and suppose $\abs{D_n}\le b$
almost surely for some $b>0$. Define
\[
 M_N\coloneq\sum_{n=1}^N D_n,
 \qquad
 W_N\coloneq\sum_{n=1}^N
   \E[D_n^2\mid\mathcal H_{n-1}].
\]
For every $a,v>0$,
\begin{equation}
 \PB(\abs{M_N}\ge a,\ W_N\le v)
 \le2\exp\!\left(-\frac{a^2}{2(v+ba/3)}\right).
 \label{eq:freedman_joint}
\end{equation}
In particular, for every $\ell>0$,
\begin{equation}
 \PB\!\left(
   \abs{M_N}>\sqrt{2v\ell}+\frac{2b\ell}{3},\ W_N\le v
 \right)\le2e^{-\ell}.
 \label{eq:freedman_inverted}
\end{equation}
\end{lemma}

\begin{proof}
Fix $0\le\lambda<3/b$. For every integer $k\ge2$,
$k!\ge2\cdot3^{k-2}$. The martingale-difference property and
$\abs{D_n}\le b$ therefore give
\begin{align*}
 \E[e^{\lambda D_n}\mid\mathcal H_{n-1}]
 &\le1+
   \sum_{k=2}^{\infty}\frac{\lambda^k b^{k-2}}{k!}
     \E[D_n^2\mid\mathcal H_{n-1}]\\
 &\le1+
   \frac{\lambda^2}{2(1-\lambda b/3)}
     \E[D_n^2\mid\mathcal H_{n-1}]\\
 &\le\exp\!\left\{
   \frac{\lambda^2}{2(1-\lambda b/3)}
     \E[D_n^2\mid\mathcal H_{n-1}]
   \right\}.
\end{align*}
Thus
\[
 L_n\coloneq\exp\!\left\{
    \lambda M_n-
    \frac{\lambda^2}{2(1-\lambda b/3)}W_n
  \right\},\qquad 0\le n\le N,
\]
is a nonnegative supermartingale with $L_0=1$ and $\E L_N\le1$.
On the event $\{M_N\ge a,\ W_N\le v\}$ it is at least
\[
 \exp\!\left\{
   \lambda a-\frac{\lambda^2v}{2(1-\lambda b/3)}
 \right\}.
\]
Markov's inequality and the choice
$\lambda=a/(v+ba/3)<3/b$ yield
\[
 \PB(M_N\ge a,\ W_N\le v)
 \le\exp\!\left(-\frac{a^2}{2(v+ba/3)}\right).
\]
Applying the same argument to $-D_n$ proves
\cref{eq:freedman_joint}. Finally,
$a=\sqrt{2v\ell}+2b\ell/3$ satisfies
$a^2\ge2\ell(v+ba/3)$, which gives
\cref{eq:freedman_inverted}.
\end{proof}

\section{Hitting-Time Identities and Mixing Amplification}
\label{Appendix:hitting_details}
This appendix proves the hitting-time identities and the forward mixing
estimate used in Section~\ref{Section:initialization}. All arguments
concern the behavior state chain alone. They allow the reward-transition dependence
in \eqref{eq:data_model}, and none requires reversibility.
For background on finite-chain hitting times and mixing, see
\citet{levin2017markov}; the identities needed here are proved in full below.

\subsection{Return excursions and their last predecessor}

\begin{proof}[Proof of Lemma~\ref{lemma:reverse_predecessor}]
We first justify the finiteness of the hitting and return means used
below. In a finite irreducible chain, from every state there is a path of
positive probability to any prescribed target. By allowing at least one
step when starting at the target, there is also a positive-probability
return path. There are finitely many starting states, so some finite
integer $L$ and some $p_0>0$ bound the length of these paths from above
and their probabilities from below. The Markov property then bounds the
probability of avoiding the target for $mL$ steps by $(1-p_0)^m$.
Thus both hitting times and first return times after time zero have
finite expectations under either kernel $\bP$ or $\bP^\star$.

For the reverse chain started at $i$, set $\sigma\coloneq\tau_i^+$ and
let
\[
    n_j\coloneq\E_i^\star\left[
         \sum_{t=0}^{\sigma-1}\ind\{X_t=j\}\right].
\]
These are the expected occupation counts in an excursion that includes
its initial state $i$ and excludes its returning state $i$. For each
state $\ell$, conditional expectation gives
\begin{align*}
    \sum_j n_jP^\star(j,\ell)
       &=\sum_{t\geq0}\E_i^\star\brk{
             \ind\{t<\sigma\}P^\star(X_t,\ell)}\\
       &=\E_i^\star\left[
             \sum_{t=0}^{\sigma-1}\ind\{X_{t+1}=\ell\}\right]
        =n_\ell.
\end{align*}
Here $\{t<\sigma\}$ is measurable before the transition at time $t$,
and sums and expectations can be interchanged by nonnegativity.
The last equality holds pathwise after subtracting the two occupation
counts: their difference is
$\ind\{X_\sigma=\ell\}-\ind\{X_0=\ell\}=0$.
Thus $\bm n^\top\bP^\star=\bm n^\top$. Normalizing the nonzero
finite vector $\bm n$ yields a stationary distribution. Uniqueness of
the stationary distribution, together with $n_i=1$, implies
\[
    n_j=\frac{\mu(j)}{\mu(i)},
    \qquad
    \E_i^\star\sigma=\sum_j n_j=\frac1{\mu(i)}.
\]
The same occupation-count argument for the forward kernel proves
$\E_i\tau_i^+=1/\mu(i)$.

To identify the last predecessor $J=X_{\sigma-1}$, observe that there
is exactly one transition into $i$ at times $0,\ldots,\sigma-1$,
namely the final transition of the excursion. Consequently,
\begingroup
\allowdisplaybreaks[0]
\begin{align*}
    \PB_i^\star(J=j)
       &=\E_i^\star\left[
          \sum_{t=0}^{\sigma-1}
             \ind\{X_t=j,X_{t+1}=i\}\right]\\
       &=n_jP^\star(j,i)
        =\frac{\mu(j)}{\mu(i)}
             \frac{\mu(i)P(i,j)}{\mu(j)}
        =P(i,j).
\end{align*}
\endgroup
This computation also includes a one-step self-return at $i$.
\end{proof}

\subsection{Time reversal of hitting means}

\begin{proof}[Proof of Lemma~\ref{lemma:hitting_compensation}]
We give a self-contained derivation of the classical reversal identity
\eqref{eq:hitting_reversal_identity} \citep{beveridge2010exit} from a finite
matrix identity, including the required boundary-value calculation.
Set
\[
    \bm\Pi\coloneq\bone\bmu^\top,
    \bm Z\coloneq(\bI-\bP+\bm\Pi)^{-1}.
\]
Here $\bm\Pi$ and $\bm Z$ denote matrices only within this proof.
The inverse exists. Indeed, if
$(\bI-\bP+\bm\Pi)\bv=\bm0$, multiplying by $\bmu^\top$ gives
$\bmu^\top\bv=0$. Hence $\bP\bv=\bv$. Every real vector fixed
by a finite irreducible stochastic matrix is constant: at a state
attaining its maximum, the averaging equation forces all possible
successors to have the same maximum, and irreducibility propagates
this to every state. The condition $\bmu^\top\bv=0$ therefore
forces $\bv=\bm0$, proving invertibility.

The equalities
$(\bI-\bP+\bm\Pi)\bone=\bone$ and
$\bmu^\top(\bI-\bP+\bm\Pi)=\bmu^\top$ give
\[
    \bm Z\bone=\bone,
    \qquad \bmu^\top\bm Z=\bmu^\top,
    \qquad (\bI-\bP)\bm Z=\bI-\bm\Pi.
\]
Write $Z_{ji}$ for the $(j,i)$ entry of $\bm Z$.
For a fixed target $i$, define
\[
    v_i(j)\coloneq\frac{Z_{ii}-Z_{ji}}{\mu(i)}.
\]
Then $v_i(i)=0$, and the last matrix identity implies
\[
    (\bI-\bP)v_i=\bone-\frac{\be_i}{\mu(i)}.
\]
In particular, for $j\neq i$,
$v_i(j)=1+\sum_xP(j,x)v_i(x)$. These are exactly the hitting-time
equations with zero boundary value at $i$. They have a unique solution:
the difference of two solutions is harmonic away from $i$ and zero at
$i$; its value at a starting state equals its expected value after
stopping at the earlier of time $n$ and the first hit of $i$.
The difference is bounded on the finite state space, and the hitting
time is almost surely finite, so this expectation tends to zero.
It follows that
\begin{equation}\label{eq:fundamental_hitting_formula}
    h_i(j)=\frac{Z_{ii}-Z_{ji}}{\mu(i)}.
\end{equation}

Let $\bD_\mu\coloneq\diag(\bmu)$.
Since $\bP^\star=\bD_\mu^{-1}\bP^\top\bD_\mu$ and
$\bD_\mu^{-1}\bm\Pi^\top\bD_\mu=\bm\Pi$, the corresponding reverse
matrix is
\[
    \bm Z^\star\coloneq(\bI-\bP^\star+\bm\Pi)^{-1}
            =\bD_\mu^{-1}\bm Z^\top\bD_\mu.
\]
Apply \eqref{eq:fundamental_hitting_formula} to the reverse chain,
with target $j$ and starting state $i$. We obtain
\begin{equation}\label{eq:fundamental_reverse_hitting}
    h_j^\star(i)
       =\frac{Z^\star_{jj}-Z^\star_{ij}}{\mu(j)}
       =\frac{Z_{jj}}{\mu(j)}-\frac{Z_{ji}}{\mu(i)}.
\end{equation}
Averaging this formula over the starting state and using that each row
of $\bm Z$ sums to one gives
\begin{equation}\label{eq:fundamental_hitting_potential}
    a(i)=\sum_x\mu(x)h_i^\star(x)
       =\frac{Z_{ii}}{\mu(i)}-\sum_xZ_{ix}
       =\frac{Z_{ii}}{\mu(i)}-1.
\end{equation}
Equations~\eqref{eq:fundamental_hitting_formula}--
\eqref{eq:fundamental_hitting_potential} now imply
\[
    h_j^\star(i)=h_i(j)+a(j)-a(i),
\]
including when $i=j$.

Average this equality over $j$ with probabilities $P(i,j)$.
A first-step decomposition of the forward return time gives
\[
    1+\sum_jP(i,j)h_i(j)
       =\E_i\tau_i^+=\frac1{\mu(i)},
\]
where the return-time formula is proved in
Lemma~\ref{lemma:reverse_predecessor}. Therefore
\[
    q(i)=\sum_jP(i,j)h_j^\star(i)
       =\sum_jP(i,j)a(j)-a(i)+c(i)-1.
\]
This is \eqref{eq:hitting_compensation}. Finally,
$h_j^{\star,+}(i)=h_j^\star(i)$ unless $j=i$, and in that case
the two values are $1/\mu(i)$ and zero. It follows that
\[
    \ell(i)=q(i)+\frac{P(i,i)}{\mu(i)}
       \leq\bigl((\bP-\bI)\bm a\bigr)(i)+2c(i),
\]
which proves \eqref{eq:positive_hitting_compensation}.
\end{proof}

\subsection{Amplifying forward mixing}

\begin{lemma}[Mixing amplification]\label{lemma:mixing_amplification}
Suppose that $\tmix<\infty$. For the total-variation profile in \eqref{eq:mixing_profile},
\begin{equation}\label{eq:mixing_amplification}
    d_{\mathrm{TV}}(q\tmix)\leq2^{-q},
    \qquad q=1,2,\ldots.
\end{equation}
Consequently,
$\tau=\tmix\lceil\log_2(2/\mumin)\rceil$ satisfies
\begin{equation}\label{eq:hitting_forward_minorization}
    P^\tau(x,i)\geq\frac{\mu(i)}2,
    \qquad x,i\in\gS.
\end{equation}
\end{lemma}

\begin{proof}
Define the pairwise total-variation profile
\[
    \overline d_{\mathrm{TV}}(t)
       \coloneq\max_{x,y\in\gS}
            \norm{\bP^t(x,\cdot)-\bP^t(y,\cdot)}_{\TV}.
\]
Stationarity and the triangle inequality show that
\[
    d_{\mathrm{TV}}(t)
       \leq\overline d_{\mathrm{TV}}(t)
       \leq2d_{\mathrm{TV}}(t).
\]
The first inequality follows by expressing $\bmu$ as the mixture
$\sum_y\mu(y)\bP^t(y,\cdot)$.

For any two probability distributions $\bnu$ and $\bm\lambda$, we claim
\begin{equation}\label{eq:dobrushin_contraction}
    \norm{\bnu\bP^t-\bm\lambda\bP^t}_{\TV}
       \leq\overline d_{\mathrm{TV}}(t)
                        \norm{\bnu-\bm\lambda}_{\TV}.
\end{equation}
If $a\coloneq\norm{\bnu-\bm\lambda}_{\TV}>0$, the positive and negative
parts of $\bnu-\bm\lambda$ each have mass $a$. Dividing them by $a$
gives probability distributions $\bnu_+$ and $\bnu_-$. Since
\[
    \bnu_+\bP^t-\bnu_-\bP^t
       =\sum_{x,y}\nu_+(x)\nu_-(y)
             \bigl[\bP^t(x,\cdot)-\bP^t(y,\cdot)\bigr],
\]
their total-variation distance is at most
$\overline d_{\mathrm{TV}}(t)$. Multiplication by $a$ proves
\eqref{eq:dobrushin_contraction}; the case $a=0$ is immediate.

Apply \eqref{eq:dobrushin_contraction} to any two rows of $\bP^s$.
After maximizing over those rows, it gives
\[
    \overline d_{\mathrm{TV}}(s+t)
       \leq\overline d_{\mathrm{TV}}(s)
                 \overline d_{\mathrm{TV}}(t).
\]
By the definition of $\tmix$,
$\overline d_{\mathrm{TV}}(\tmix)\leq1/2$. Iteration now proves
\eqref{eq:mixing_amplification}.
For the displayed $\tau$, this entails
$d_{\mathrm{TV}}(\tau)\leq\mumin/2$. Total variation controls the
probability of each singleton, so
\[
    P^\tau(x,i)\geq\mu(i)-\frac{\mumin}{2}
                    \geq\frac{\mu(i)}2.
\]
This proves \eqref{eq:hitting_forward_minorization} directly from the
forward mixing profile.
\end{proof}

\section{Proof of the Constant-Step Theorem}
\label{Appendix:main_proof}
\begin{proof}[Proof of Theorem~\ref{theorem:main}]
Since $L>1$, $0<\varepsilon\leq1$, and $H>1$, the chosen step size satisfies
\[
0<\eta\varrho H=\frac{\varepsilon^2}{1024HL^2}\leq\frac14.
\]
It therefore meets the conditions of both the stochastic and initialization bounds.
Fix a deterministic integer $T\geq T_\star=b_{\mathrm{blk}}k$, with $b_{\mathrm{blk}}$ and $k$ as in
\eqref{eq:main_hitting_budget}.
Apply Proposition~\ref{prop:stochastic_convolution} with $q=\delta/2$.
Then $\ell_q=L$, so, with probability at least $1-\delta/2$,
\begin{align*}
\norm{\bY_T}_\infty
&\leq7H\sqrt{\varrho\eta}L+2\eta\varrho H^2\\
&=\frac7{32}\varepsilon+\frac{\varepsilon^2}{512L^2}
\leq\frac{113}{512}\varepsilon
<\frac\varepsilon2.
\end{align*}

The parameters $b_{\mathrm{blk}},k$ agree with the block length and number of blocks in Proposition~\ref{proposition:initialization}.
Since $\norm{\bDelta_0}_\infty\leq H$, that proposition gives, with probability at least $1-\delta/2$,
\[
\norm{\widehat{\bPhi}_{T,0}\bDelta_0}_\infty\leq\varepsilon/2.
\]
A union bound and \eqref{eq:error_decomposition} give \eqref{eq:main_guarantee}.

Finally, absorbing the ceiling operations into a universal constant gives
\[
T_\star=b_{\mathrm{blk}}k
 \leq C\prn{\frac{\varrho H^3L^2}{\mumin\varepsilon^2}
                   +\mathfrak h_\star}
          \log\frac{4dH}{\delta\varepsilon\mumin}.
\]
Suppressing logarithmic factors gives \eqref{eq:hitting_sample_complexity}.
\end{proof}

\begin{proof}[Proof of Corollary~\ref{corollary:mixing_time}]
The additional aperiodicity assumption allows us to apply
Lemma~\ref{lemma:reverse_hitting_mixing}, which gives
$\mathfrak h_\star\leq2\tau/\mumin$, whence
\[
b_{\mathrm{blk}}\leq
\left\lceil\frac8{\mumin}\prn{\frac H\eta+2\tau}\right\rceil=b_{\mathrm{mix}}.
\]
Thus Theorem~\ref{theorem:main} applies to every deterministic integer
$T\geq b_{\mathrm{mix}}k$.
As above,
\[
b_{\mathrm{mix}}k
\leq\frac C{\mumin}
\prn{\frac{\varrho H^3L^2}{\varepsilon^2}+\tau}
\log\frac{4dH}{\delta\varepsilon\mumin}.
\]
Using $\tau\leq C\tmix(1+\log(1/\mumin))$ yields
\eqref{eq:sample_complexity} after suppressing logarithmic factors.
\end{proof}

\section{Decreasing Step Sizes}
\label{Appendix:decay}
We prove Theorem~\ref{theorem:decay} by first establishing interval
estimates for any deterministic nonincreasing step sequence.
The stochastic estimate is a specialization of Section~\ref{Section:reverse}.
To bound initialization, we control the sum of the reverse masses
using the hitting-time compensation identity.
Throughout this appendix, Assumptions~\ref{assumption:bounded}
and~\ref{assumption:mixing} are in force.

\subsection{Reverse propagation on a deterministic interval}
Fix integers $0\leq s<T$ and put $N\coloneq T-s$.
Use the matrices in \eqref{eq:random_matrix}--\eqref{eq:random_product}
with the deterministic step size $\eta_t$ at forward time $t$.
Unrolling the error recursion on this interval gives
\begin{equation}\label{eq:decay_decomposition}
\begin{split}
 \bV_T-\bV^\pi
   &=\widehat{\bPhi}_{T,s}(\bV_s-\bV^\pi)+\bY_{T,s},\\
 \bY_{T,s}
   &\coloneq\sum_{t=s}^{T-1}\eta_t
       \widehat{\bPhi}_{T,t+1}\be_{S_t}
       \beta(S_t,W_t,R_t,S_{t+1}).
\end{split}
\end{equation}
For a stationary forward interval and terminal coordinate $j$, set
\[
\begin{split}
 X_n&=S_{T-n},\quad \omega_{n+1}=\Omega_{T-n-1},\quad
 w_{n+1}=W_{T-n-1},\\
 a_n&=\eta_{T-n-1},\qquad
 \bp_n=\widehat{\bPhi}_{T,T-n}^{\top}\be_j.
\end{split}
\]
Here the reverse edge marks, importance weights, and steps are defined for $0\leq n<N$,
and the states and propagation weights are defined through $n=N$.
Use the filtration $\gG_n$ in \eqref{eq:reverse_filtration}.
Lemma~\ref{lemma:reverse_kernel}, applied to the interval $[s,T)$,
shows that conditional on $\gG_n$ and $X_n=x$, the next pair
$(X_{n+1},\omega_{n+1})$ has law
$P^\star(x,y)\mathsf K(\mathrm{d}\omega\mid y,x)$.
Multiplication by $\widehat{\bB}_{T-n-1}^{\top}$ gives
\begin{equation}\label{eq:decay_weights}
 \bp_{n+1}=\bp_n+a_nw_{n+1}p_n(y)(\gamma\be_x-\be_y),
 \qquad x=X_n,\quad y=X_{n+1}.
\end{equation}
Forward nonincreasing steps become reverse nondecreasing steps:
\begin{equation}\label{eq:decay_range}
 0<\underline a\leq a_0\leq\cdots\leq a_{N-1}\leq\overline a\leq(\varrho H)^{-1}.
\end{equation}
We also use the abstract reverse process defined by this marked kernel
and \eqref{eq:decay_weights}, started from any fixed
$\bp_0\geq\bm0$ and $X_0=x$ with
$q_0\coloneq\norm{\bp_0}_1\leq1$.
Each $\bp_n$ is $\gG_n$-measurable.
Nonnegativity follows because the update removes at most the mass at
$y$ and adds nonnegative mass at $x$; when $x=y$, it multiplies that
coordinate by $1-a_nw_{n+1}/H$. Writing $q_n\coloneq\norm{\bp_n}_1$,
summing the coordinates yields
$q_{n+1}=q_n-a_nw_{n+1}p_n(X_{n+1})/H$ and hence
\begin{equation}\label{eq:decay_mass}
 \bp_n\geq\bm0,\quad q_{n+1}\leq q_n,\qquad
 \sum_{n<N}a_nw_{n+1}p_n(X_{n+1})=H(q_0-q_N)\leq Hq_0.
\end{equation}

\subsection{Stochastic concentration on an interval}
The general stochastic estimate already allows every deterministic
nonincreasing step sequence. We record its scalar specialization for
use with the schedule in Theorem~\ref{theorem:decay}.

\begin{proposition}[Variable-step stochastic convolution]\label{prop:decay_noise}
Suppose $0<\eta_{T-1}\leq\cdots\leq\eta_s\leq(\varrho H)^{-1}$.
For any $q\in(0,1)$ and any initial distribution at the beginning of
 the interval, set $\ell_q=\log(4d/(q\mumin))$. Then
\begin{equation*}
 \PB\left(\norm{\bY_{T,s}}_\infty>
             7H\sqrt{\varrho\eta_s}\,\ell_q+2\eta_s\varrho H^2\right)\leq q.
\end{equation*}
\end{proposition}
\begin{proof}
Apply Proposition~\ref{prop:affine_noise} with one component,
$A_t=\gamma$, $C_t=R_t$, and the Bellman innovation $h=\beta$.
The weighted variance certificate has parameters $b_{\mathrm{osc}}=\varrho H$ and
$\mathcal V=\varrho H^2$.
\end{proof}

We will also use the initial-distribution comparison for suffix events.
For an arbitrary distribution $\bm\lambda$ at time $s$, write
$\PB_{\bm\lambda,s}$ for its marked state-path law on $[s,T)$.
The Markov property of the marked trajectory gives, for every event $B$ measurable on
this interval,
\begin{equation}\label{eq:decay_suffix_density}
 \PB_{\bm\lambda,s}(B)
  =\E_{\bmu,s}\!\left[
       \frac{\lambda(S_s)}{\mu(S_s)}\ind_B\right]
  \leq\frac{\PB_{\bmu,s}(B)}{\mumin}.
\end{equation}
This includes the unconditional suffix of a trajectory begun earlier,
with $\bm\lambda$ its marginal law at time $s$.

\subsection{Integrated reverse mass and initialization}
For varying steps, the expected lifetime of the auxiliary particle
depends on the current time as well as its position and driver state.
We use the reverse weights directly: the following argument bounds
their integrated mass and retains an additive hitting-time term.

\paragraph{Hitting potentials.}
Recall the notation in
\eqref{eq:initialization_hitting_times}--\eqref{eq:hitting_potential},
with zero-time hitting conventions:
\[
 h_i(x)=\E_x\tau_i,\quad
 h_i^\star(x)=\E_x^\star\tau_i,\quad
 c(i)=\mu(i)^{-1},\quad
 a(i)=\sum_x\mu(x)h_i^\star(x).
\]
The notation $a(i)$ always refers to this state potential; $a_n$
denotes a reverse step size.
Lemma~\ref{lemma:hitting_compensation} gives the first identity below.
The second follows by a first-step decomposition of the forward
return time and Lemma~\ref{lemma:reverse_predecessor}:
\begin{equation}\label{eq:decay_hitting_input}
 h_j^\star(i)=h_i(j)+a(j)-a(i),
 \qquad \sum_xP(i,x)h_i(x)=c(i)-1.
\end{equation}
In particular,
\begin{equation}\label{eq:decay_hitting_range}
 0\leq a(i)\leq\mathfrak h_\star,
 \qquad \max_{i,x}h_i(x)\leq2\mathfrak h_\star.
\end{equation}
To compensate for the random update intensity, apply
Lemma~\ref{lemma:local_poisson} to the source-centered observable $W-1$.
It gives anchored functions $\phi_i$ such that
\[
 (I-P^\star)\phi_i(x)
 =P^\star(x,i)\E_b[W-1\mid S=i,S'=x],\quad
 \phi_i(i)=0,\quad -1\leq\phi_i\leq\varrho-1.
\]
Define an anchored potential
\begin{equation*}
 \widetilde w_i(x)\coloneq h_i^\star(x)-c(i)\ind_{\{x\ne i\}}-c(i)\phi_i(x),
 \qquad \widetilde{\bm w}(x)\coloneq(\widetilde w_i(x))_{i\in\gS}.
\end{equation*}
The reverse hitting equation implies
\begin{equation}\label{eq:decay_wpoisson}
\begin{split}
 \widetilde w_i(i)&=0,\\
 ((I-P^\star)\widetilde w_i)(x)
 &=1-c(i)P^\star(x,i)\E_b[W\mid S=i,S'=x],\\
 -\varrho c(i)&\leq\widetilde w_i(x)\leq\mathfrak h_\star.
\end{split}
\end{equation}
Indeed, the reverse hitting equations away from $i$ and the reverse
return-time formula at $i$ give
$((\bI-\bP^\star)h_i^\star)(x)=1-c(i)\ind_{\{x=i\}}$.
Adding the function $c(i)\ind_{\{x=i\}}-c(i)-c(i)\phi_i(x)$ to $h_i^\star(x)$
then gives \eqref{eq:decay_wpoisson}.
For the range bound, if $x\ne i$ then
$\widetilde w_i(x)=h_i^\star(x)-c(i)(1+\phi_i(x))$,
where $0\leq1+\phi_i(x)\leq\varrho$; at $x=i$ the value is zero.

\begin{lemma}[Expected integrated mass]\label{lemma:decay_integrated}
Under \eqref{eq:decay_range}, for every integer $N\geq1$ and every
admissible reverse initial pair,
\begin{equation}\label{eq:decay_integrated}
 \E\sum_{n=0}^{N-1}q_n
 \leq7q_0\left(\frac{H}{\underline a\mumin}+\mathfrak h_\star\right).
\end{equation}
\end{lemma}
\begin{proof}
Write $Z_n\coloneq\bp_n^\top\widetilde{\bm w}(X_n)$ and abbreviate
$x=X_n$, $y=X_{n+1}$, $w=w_{n+1}$.
The anchor in \eqref{eq:decay_wpoisson} and the update
\eqref{eq:decay_weights} give
\[
 Z_{n+1}-Z_n
 =\bp_n^\top(\widetilde{\bm w}(y)-\widetilde{\bm w}(x))
       +a_nw\gamma p_n(y)\widetilde w_x(y).
\]
Conditioning on $\gG_n$ and using \eqref{eq:decay_wpoisson} gives
\[
 \E[Z_{n+1}-Z_n\mid\gG_n]
  =-q_n+\E[wp_n(y)c(y)\mid\gG_n]
     +\gamma a_n\E[wp_n(y)\widetilde w_x(y)\mid\gG_n].
\]
Summing the expectations yields the exact identity
\begin{equation}\label{eq:decay_mass_time_identity}
\begin{split}
 \E\sum_{n<N}q_n
 ={}&\E\sum_{n<N}wp_n(y)c(y)+Z_0-\E Z_N\\
 &+\gamma\E\sum_{n<N}a_nwp_n(y)\widetilde w_x(y).
\end{split}
\end{equation}
The first term is at most $Hq_0/(\underline a\mumin)$ because
\[
 \sum_{n<N}wp_n(y)c(y)
    \leq\frac1{\underline a\mumin}\sum_{n<N}a_nwp_n(y)
    \leq\frac{Hq_0}{\underline a\mumin}.
\]
The bounds on $\widetilde w_i$ and $q_N\leq q_0$ imply
$Z_0\leq\mathfrak h_\star q_0$ and
$-Z_N\leq\varrho q_0/\mumin$. Hence the boundary term is at most
$q_0(\mathfrak h_\star+\varrho/\mumin)$.

To control the last term, use \eqref{eq:decay_hitting_input} and $1+\phi_x(y)\geq0$ to write
\begin{equation}\label{eq:decay_wdecomp}
 \widetilde w_x(y)\leq h_y(x)+a(x)-a(y).
\end{equation}
Set
\[
 m(y)=\E_b[W h_y(S')\mid S=y]
       \leq\varrho(c(y)-1),\qquad
 \chi(y,\omega,x)=W(\omega)(h_y(x)-m(y)).
\]
Since $\E_b[W\mid S=y]=1$, $\chi$ is centered at every source.
Also $0\leq h_y(x),m(y)\leq2\mathfrak h_\star$. Its sourcewise
range lies in
$[-\varrho m(y),\varrho(2\mathfrak h_\star-m(y))]$,
of width $2\varrho\mathfrak h_\star$.
Thus Lemma~\ref{lemma:generic_corrector} gives
\[
 \left|\E\sum_{n<N}a_np_n(y)\chi(y,\omega_{n+1},x)\right|
 \leq4\overline a\varrho H\mathfrak h_\star q_0
 \leq4\mathfrak h_\star q_0.
\]
Also,
\[
 \gamma\sum_{n<N}a_nwp_n(y)m(y)
 \leq\frac{\varrho Hq_0}{\mumin}.
\]
Writing $\bm a=(a(i))_{i\in\gS}$, the update
\eqref{eq:decay_weights} gives the telescoping identity
\begin{align*}
 \gamma\sum_{n<N}a_nwp_n(y)(a(x)-a(y))
 &=\bp_N^\top\bm a-\bp_0^\top\bm a
    +(1-\gamma)\sum_{n<N}a_nwp_n(y)a(y)\\
 &\leq2\mathfrak h_\star q_0.
\end{align*}
The last inequality uses $\bm0\leq\bm a\leq\mathfrak h_\star\bone$
and the mass identity \eqref{eq:decay_mass} for the two nonnegative
contributions. Consequently, the last term of
\eqref{eq:decay_mass_time_identity} is at most
$\varrho Hq_0/\mumin+6\mathfrak h_\star q_0$.
Combining the bounds gives
\[
 \E\sum_{n<N}q_n
 \leq q_0\left\{\frac{H}{\underline a\mumin}
       +\frac{\varrho(H+1)}{\mumin}+7\mathfrak h_\star\right\}.
\]
Since $\underline a\varrho H\leq1$ and $H>1$, we have
$\varrho(H+1)/\mumin\leq2H/(\underline a\mumin)$.
Thus the coefficient of $H/(\underline a\mumin)$ can be bounded
by $3\leq7$, proving \eqref{eq:decay_integrated}.
Every calculation is over a finite interval and involves bounded
state potentials.
\end{proof}

\begin{proposition}[Forgetting under decreasing steps]\label{prop:decay_forget}
Suppose $0<\eta_{T-1}\leq\cdots\leq\eta_s\leq(\varrho H)^{-1}$, and set
\begin{equation*}
 b_{\mathrm{blk}}\coloneq\left\lceil14\left(
       \frac{H}{\eta_{T-1}\mumin}+\mathfrak h_\star\right)\right\rceil.
\end{equation*}
For every $0<\rho\leq1$ and arbitrary initial distribution at the beginning of the interval,
\begin{equation*}
 \PB\left(\norm{\widehat{\bPhi}_{T,s}}_{\infty\to\infty}>\rho\right)
 \leq\frac{d}{\mumin\rho}\,2^{-\lfloor(T-s)/b_{\mathrm{blk}}\rfloor}.
\end{equation*}
\end{proposition}
\begin{proof}
Take $\underline a\coloneq\eta_{T-1}$ and
$\overline a\coloneq\eta_s$.
In any complete reverse block starting at time $r$, its steps still
satisfy \eqref{eq:decay_range} with these bounds.
Conditional on $\gG_r$, the future reverse process starts from
$(\bp_r,X_r)$ and uses this deterministic suffix of the steps.
Lemma~\ref{lemma:decay_integrated}, applied to that pair and block,
and pathwise monotonicity of $q_n$ imply
\[
 \E[q_{r+b_{\mathrm{blk}}}\mid\gG_r]
 \leq\frac1{b_{\mathrm{blk}}}\E\!\left[\sum_{n=r}^{r+b_{\mathrm{blk}}-1}q_n\mid\gG_r\right]
 \leq\frac{7q_r}{b_{\mathrm{blk}}}
       \left(\frac{H}{\underline a\mumin}+\mathfrak h_\star\right)
 \leq\frac{q_r}{2}.
\]
The estimate includes $q_r=0$, when every later mass is zero.
Iterating at $r=0,b_{\mathrm{blk}},2b_{\mathrm{blk}},\ldots$ and using monotonicity over the
remaining incomplete block gives
$\E[q_N]\leq2^{-\lfloor N/b_{\mathrm{blk}}\rfloor}q_0$.
For a stationary forward interval and initial reverse weight $\be_j$,
$q_N$ is the $j$th row sum of $\widehat{\bPhi}_{T,s}$.
The latter matrix is nonnegative, so its induced sup norm is its
largest row sum. Markov's inequality for each row and a union bound
therefore give
\[
 \PB_{\bmu,s}\!\left(
   \norm{\widehat{\bPhi}_{T,s}}_{\infty\to\infty}>\rho\right)
   \leq\frac d\rho\,2^{-\lfloor N/b_{\mathrm{blk}}\rfloor}.
\]
The density comparison \eqref{eq:decay_suffix_density} gives the
additional factor $1/\mumin$ for an arbitrary initial distribution
at time $s$.
\end{proof}

\begin{corollary}[Higher-power forgetting]\label{corollary:decay_power_forget}
Suppose $0<\eta_{T-1}\leq\cdots\leq\eta_s$ and let $p\geq1$ be an
integer such that $p\varrho H\eta_s\leq1/2$. Set
\begin{equation}\label{eq:decay_power_blocks}
 b_p\coloneq\left\lceil14\left(
       \frac{2H}{p\eta_{T-1}\mumin}+\mathfrak h_\star\right)\right\rceil.
\end{equation}
For every $0<\rho\leq1$ and arbitrary initial distribution at the
beginning of the interval,
\begin{equation}\label{eq:decay_power_forget}
 \PB\left(\norm{\widehat{\bPhi}_{T,s}}_{\infty\to\infty}>\rho\right)
 \leq\frac{d}{\mumin\rho^p}\,2^{-\lfloor(T-s)/b_p\rfloor}.
\end{equation}
\end{corollary}
\begin{proof}
Define deterministic auxiliary parameters
\[
 c_p=1+\frac{p-1}{2H},\qquad
 \gamma_p=\frac\gamma{c_p},\qquad
 H_p=\frac{c_p}{c_p-\gamma},\qquad
 \eta_t^{(p)}=c_p\eta_t.
\]
We retain exactly the original importance weights and marked kernel.
For $v\in[0,\eta_s\varrho]$, put
$f_p(v)=1-(1-v)(1-v/H)^{p-1}$.
For $p\geq2$, the geometric-sum identity gives
\[
 \frac{f_p(v)}v
 =1+\frac{1-v}{H}\sum_{k=0}^{p-2}(1-v/H)^k
 \geq1+\frac{p-1}{2H}=c_p.
\]
Indeed, $(1-v)(1-v/H)^{p-2}\geq1-v-(p-2)v/H
\geq1-pv\geq1/2$. At $v=0$ the inequality follows by continuity,
and for $p=1$ it is an equality.
For $x,y\geq0$, weighted Jensen's inequality now gives
\begin{align*}
 ((1-v)x+v\gamma y)^p
 &\leq(1-v/H)^{p-1}\bigl((1-v)x^p+v\gamma y^p\bigr)\\
 &\leq(1-c_pv)x^p+c_pv\gamma_p y^p.
\end{align*}
The final coefficients are nonnegative because
$c_pv\leq pv\leq1/2$.

Let $\widehat{\bPhi}^{(p)}_{T,s}$ be the propagation matrix along
the same marked trajectory, with discount $\gamma_p$ and steps
$\eta_t^{(p)}$ in place of $\gamma$ and $\eta_t$.
Applying the preceding scalar comparison at each updated coordinate
and using nonnegativity of the auxiliary matrices gives, by forward
induction,
\[
 (\widehat{\bPhi}_{T,s}\bone)^{\odot p}
 \leq\widehat{\bPhi}^{(p)}_{T,s}\bone.
\]
Here the power is taken coordinatewise. Since the induced sup norm
of a nonnegative matrix is its largest row sum,
\begin{equation}\label{eq:decay_power_comparison}
 \norm{\widehat{\bPhi}_{T,s}}_{\infty\to\infty}^{p}
 \leq\norm{\widehat{\bPhi}^{(p)}_{T,s}}_{\infty\to\infty}.
\end{equation}
The auxiliary steps are positive and nonincreasing. Also,
\[
 H_p\leq H,\qquad c_p\leq p,\qquad
 \varrho H_p\eta_s^{(p)}\leq p\varrho H\eta_s\leq\frac12,
 \qquad
 \frac{H_p}{\eta_{T-1}^{(p)}\mumin}
 =\frac{2H}{(p+1)\eta_{T-1}\mumin}
 \leq\frac{2H}{p\eta_{T-1}\mumin}.
\]
Thus Proposition~\ref{prop:decay_forget} applies to the auxiliary
propagation. The behavior chain, importance weights, $\mumin$, and $\mathfrak h_\star$ are
unchanged, and its block length is at most $b_p$.
Apply that proposition with threshold $\rho^p$ and use
\eqref{eq:decay_power_comparison} to obtain
\eqref{eq:decay_power_forget}.
\end{proof}

\subsection{Proof of Theorem~\ref{theorem:decay}}
\begin{proof}[Proof of Theorem~\ref{theorem:decay}]
For this proof, write $A\coloneq4096$ and take the universal constant
$C_0\geq8A=32768$.
Extend $\overline L_t$ in \eqref{eq:decay_schedule} to real $t\geq0$.
Since $\overline L_t>1$, the derivative of $\overline L_t/(t+8)$ is
$(1-\overline L_t)/(t+8)^2<0$. The schedule is therefore positive,
nonincreasing, and bounded by $1/(\varrho H)$.

Fix $0<\delta<1$. For $t\geq8$,
\[
\frac{\mathrm d}{\mathrm dt}\frac{t}{L_t}
=\frac{L_t-t/(t+8)}{L_t^2}>0,
\qquad \frac{t}{L_t}\longrightarrow\infty.
\]
Thus $T_\star$ in \eqref{eq:decay_transient} is finite, and every
integer $T\geq T_\star$ satisfies
\begin{equation}\label{eq:decay_transient_condition}
T\geq8,\qquad
T\geq C_0L_T\prn{\frac{\varrho H^2}{\mumin}+\mathfrak h_\star}.
\end{equation}
Fix such an integer $T$.
Take $s\coloneq\lfloor T/2\rfloor$, $N\coloneq T-s$, and put
\[
 q_T\coloneq\frac{\delta}{4(T+1)^2},\qquad
 \rho_T\coloneq\frac1{T+8}.
\]
Choose the analysis parameters
\[
 p_T\coloneq\left\lceil\frac{L_T}{\overline L_T}\right\rceil,
 \qquad
 k_T\coloneq\left\lceil
       \log_2\frac{d}{\mumin q_T\rho_T^{p_T}}\right\rceil.
\]
We have
$L_T\leq p_T\overline L_T\leq2L_T$, and the logarithms satisfy
\begin{equation}\label{eq:decay_logs}
 k_T\leq6L_T,\qquad
 \log\frac{4d}{q_T\mumin}\leq2L_T,\qquad
 \overline L_{T-1}\geq\overline L_T/2.
\end{equation}
For the first, use
\[
 \log\frac{d}{\mumin q_T\rho_T^{p_T}}
 =\log\frac{d}{\mumin q_T}+p_T\log(T+8)
 \leq2L_T+2L_T=4L_T,
\]
and $4L_T/\log2+1\leq6L_T$, since $L_T\geq\log256$.
The second follows from
$4d/(q_T\mumin)=16d(T+1)^2/(\delta\mumin)\leq e^{2L_T}$.
For the third, subtract
$\overline L_T-\overline L_{T-1}
=\log((T+8)/(T+7))\leq\log2<\overline L_T/2$.

Let $D\coloneq\varrho H^2/\mumin+\mathfrak h_\star$.
The schedule and $s+8\geq(T+8)/2$ give
\[
 \eta_s\leq\frac{2AH\overline L_T}{\mumin(T+8)},
 \qquad
 p_T\varrho H\eta_s\leq\frac{4A\varrho H^2L_T}{\mumin(T+8)}\leq\frac12,
\]
where the last inequality uses \eqref{eq:decay_transient_condition} and
$C_0\geq8A$.
Corollary~\ref{corollary:decay_power_forget} therefore applies with
$p=p_T$. Moreover, the same bound on the uncapped step at $s$
shows that the cap is inactive throughout $[s,T)$, so
$\eta_{T-1}=AH\overline L_{T-1}/(\mumin(T+7))$.
Thus the block length in \eqref{eq:decay_power_blocks} satisfies
\begin{align*}
 b_{p_T}&\leq\frac{56(T+7)}{A L_T}+14\mathfrak h_\star+1,\\
 b_{p_T}k_T&\leq\frac{336(T+7)}A
                  +84\mathfrak h_\star L_T+6L_T\\
 &\leq\left(\frac{672}A+\frac{90}{C_0}\right)T
       <\frac T2\leq N,
\end{align*}
where $D\geq1$, $\mathfrak h_\star\leq D$, and
\eqref{eq:decay_transient_condition} give the last line.
In particular, $\lfloor N/b_{p_T}\rfloor\geq k_T$.
By Corollary~\ref{corollary:decay_power_forget}, with probability at least
$1-q_T$, we have
$\norm{\widehat{\bPhi}_{T,s}}_{\infty\to\infty}\leq\rho_T$.
The update preserves $[0,H]^d$, so
$\norm{\bV_s-\bV^\pi}_\infty\leq H$ pathwise.
On this high-probability event, the first term in
\eqref{eq:decay_decomposition} is therefore at most $H/(T+8)$.

The bound on $\eta_s$, Proposition~\ref{prop:decay_noise}, and
\eqref{eq:decay_logs} yield, outside another event of probability at most $q_T$,
\[
 \norm{\bY_{T,s}}_\infty
   \leq14\sqrt{2A}
         \sqrt{\frac{\varrho H^3\overline L_T L_T^2}{\mumin(T+8)}}
       +\frac{4A\varrho H^3\overline L_T}{\mumin(T+8)}.
\]
Substituting $14\sqrt{2A}=896\sqrt2$ and $4A=16384$ and using
\eqref{eq:decay_decomposition} proves
the error bound in \eqref{eq:decay_rate} at the fixed time $T$ with exceptional
probability at most $2q_T$.
Both probability estimates concern only the suffix trajectory, whose
law conditional on $S_s$ is specified by \eqref{eq:data_model}.
The density comparison \eqref{eq:decay_suffix_density} applies to
its marginal initial law. We do not condition on the random estimate
$\bV_s$: the pathwise bound on $\norm{\bV_s-\bV^\pi}_\infty$
already controls its contribution, regardless of its dependence on
the suffix.

A union bound over all integers $T\geq T_\star$, using
\[
 \sum_{T\geq8}2q_T
   =\frac\delta2\sum_{T\geq8}\frac1{(T+1)^2}<\delta,
\]
gives $\PB_{\bnu}(\mathcal E_\delta)\geq1-\delta$ for the event
defined in \eqref{eq:decay_rate}.
Finally, when $0<\varepsilon\leq1$, $H>1$, $L_T\geq\overline L_T>1$, and
$\mumin\leq1$, the quantity
$\varrho H^3\overline L_T L_T^2/(\mumin\varepsilon^2)$ dominates
$\varrho H^2L_T/\mumin$, $H/\varepsilon$, and
$\varrho H^3\overline L_T/(\mumin\varepsilon)$.
Thus, for a sufficiently large universal constant $C$,
\eqref{eq:decay_sample} implies \eqref{eq:decay_transient_condition},
hence $T\geq T_\star$, and makes each of the three error terms in
\eqref{eq:decay_rate} at most
$\varepsilon/3$.
The same event $\mathcal E_\delta$ proves the asserted guarantee for every
target precision.
\end{proof}

\section{TD Learning with Unknown Minimum Stationary Probability}
\label{Appendix:unknown_mu}
We give an empirical-occupancy step-size rule for importance-weighted asynchronous TD learning
when $\mumin$ is unknown. The use of piecewise constant steps estimated
from state frequencies is related to \citet[Section~III-D and Theorem~3]{li2022async}.
Here the output is the current iterate $\bV_T$.

\subsection{An empirical-occupancy step-size rule}
Let $n_j\coloneq2^j$ for $j=0,1,\ldots$, and define
\begin{equation}\label{eq:adaptive_counts}
\begin{split}
 N_n(x)\coloneq\sum_{t=0}^{n-1}\ind\{S_t=x\},
 \qquad
 \ell_n\coloneq\log\bigl(16d\varrho H(n+8)^2\bigr),\\
 \widehat{m}_j\coloneq
 \max\left\{\frac1{n_j},\min_{x\in\gS}\frac{N_{n_j}(x)}{n_j}\right\}.
\end{split}
\end{equation}
At the beginning of the epoch $[n_j,n_{j+1})$, compute
$\widehat{m}_j$ from the states already observed, and use
\begin{equation}\label{eq:adaptive_schedule}
 \eta_0\coloneq(\varrho H)^{-1},\qquad
 \eta_t\coloneq\alpha_j
 :=\min\left\{\frac1{\varrho H},
       \frac{4096H\ell_{n_j}}{n_j\widehat{m}_j}\right\},
 \qquad n_j\leq t<n_{j+1}.
\end{equation}
Apply the usual update \eqref{eq:td_update_vector}, carrying the same
value estimate across epochs. The floor in \eqref{eq:adaptive_counts}
makes the rule well defined before every state has been visited.
The step-size rule uses only $H$, $d$, $\varrho$, and the observed state counts.
Its steps are constant within each epoch; they can increase at an epoch
boundary because the occupancy estimate changes. Once the frequencies
are within constant factors of their stationary probabilities and the cap
is inactive, the steps have order $H\log(t)/(\mumin t)$.

\begin{theorem}[TD learning with unknown minimum stationary probability]
\label{theorem:unknown_mu}
Under Assumptions~\ref{assumption:bounded} and~\ref{assumption:mixing},
run \eqref{eq:td_update_vector} with \eqref{eq:adaptive_schedule}.
For $0<\delta<1$, define
\[
 \mathcal L_T\coloneq\log\frac{64d\varrho H(T+8)^2}{\delta\mumin}.
\]
There are universal constants $C_1,C_2>0$ and an event of probability
at least $1-\delta$, under any initial state distribution, on which
simultaneously for every integer $T\geq8$ satisfying
\begin{equation}\label{eq:adaptive_transient}
 T\geq C_1\left(\frac{\varrho H^2}{\mumin}+\mathfrak h_\star\right)
                  \mathcal L_T,
\end{equation}
we have
\begin{equation}\label{eq:adaptive_rate}
 \norm{\bV_T-\bV^\pi}_\infty
 \leq C_2\left\{
  \frac H{T+8}
  +\sqrt{\frac{\varrho H^3\ell_T\mathcal L_T^2}{\mumin(T+8)}}
  +\frac{\varrho H^3\ell_T}{\mumin(T+8)}\right\}.
\end{equation}
Consequently, for a universal constant $C>0$, on the same event,
for every $0<\varepsilon\leq1$ and every integer $T\geq8$ with
\begin{equation}\label{eq:adaptive_sample}
 T\geq C\left\{
 \frac{\varrho H^3\ell_T\mathcal L_T^2}{\mumin\varepsilon^2}
       +\mathfrak h_\star\mathcal L_T\right\},
\end{equation}
we have $\norm{\bV_T-\bV^\pi}_\infty\leq\varepsilon$.
The sufficient transition count is therefore
$\wtilde O(\varrho H^3/(\mumin\varepsilon^2)+\mathfrak h_\star)$.
\end{theorem}

The occupancy-estimation cost appears in the additive term in
\eqref{eq:adaptive_sample}. The quantities $\mumin$, $\mathfrak h_\star$, and
$\delta$ enter the guarantee but are not inputs to
\eqref{eq:adaptive_schedule}.

\subsection{Concentration of the empirical occupancies}

\begin{lemma}[Multiplicative occupancy bound from hitting times]
\label{lemma:adaptive_occupancy}
For every integer $n\geq1$ and any initial state distribution,
\begin{equation}\label{eq:adaptive_occupancy}
 \PB\left(\exists x\in\gS:
    \frac{N_n(x)}n\notin[\mu(x)/2,2\mu(x)]\right)
 \leq4d\exp\left(-\frac{n}{96(1+\mathfrak h_\star)}\right).
\end{equation}
\end{lemma}
\begin{proof}
The reversal identity \eqref{eq:hitting_reversal_identity} and
$0\leq a(i)\leq\mathfrak h_\star$ imply
\[
 h_i(j)=h_j^\star(i)+a(i)-a(j)\leq2\mathfrak h_\star.
\]
Put $F\coloneq1+2\mathfrak h_\star$ and fix a state $x$.
Let $\tau\coloneq\inf\{t\geq0:S_t=x\}$ be the first hitting time
of $x$, and let $\xi_1,\xi_2,\ldots$ be the successive positive
return times to $x$ after $\tau$. By the strong Markov property,
these return times are independent and identically distributed.
Their common mean is $m\coloneq1/\mu(x)$, by
Lemma~\ref{lemma:reverse_predecessor}, and
\[
 \E\tau\leq F,
 \qquad m=1+\sum_yP(x,y)h_x(y)\leq F.
\]
For every integer $k\geq1$, we claim that
\[
 \E\tau^k\leq k!F^k,
 \qquad \E\xi_1^k\leq k!mF^{k-1}.
\]
To prove this, consider the first hitting times of $x$ from all
states and its first positive return time. Every such time $Z$ has
mean at most $F$. On $\{Z>s\}$, its conditional remaining time at
an integer $s\geq1$ is again a first hitting time of $x$; at $s=0$
it is the original time. Induction using
\[
 Z^k\leq k\sum_{s=0}^{Z-1}(Z-s)^{k-1}
\]
gives $\E Z^k\leq k!F^{k-1}\E Z$.
These moments are finite: finite irreducibility gives a positive
probability, uniform over starting states, of hitting $x$ within a
fixed number of steps, hence a geometric tail on that time scale.
This also gives finite moments for the positive return time.

The moment bounds imply, for $0\leq\lambda<F^{-1}$,
\[
 \E e^{\lambda\tau}\leq(1-\lambda F)^{-1},
 \qquad
 \log\E e^{\lambda\xi_1}
       \leq\frac{m\lambda}{1-\lambda F}.
\]
Moreover, $e^{-z}\leq1-z+z^2/2$ for $z\geq0$ gives
\[
 \log\E e^{-\lambda\xi_1}
       \leq-m\lambda+mF\lambda^2.
\]
Set $\lambda=(4F)^{-1}$ and
$W_k\coloneq\sum_{i=1}^k\xi_i$, with $W_0=0$.
For the lower occupancy tail, let
$k=\lceil n/(2m)\rceil-1$. If $N_n(x)<n/(2m)$, the $(k+1)$st
visit has not occurred before time $n$, so $\tau+W_k\geq n$.
Since $km<n/2$, exponential Markov inequalities give
\begin{align*}
 \PB\left(N_n(x)<\frac{n}{2m}\right)
 &\leq\PB(\tau\geq n/4)+\PB(W_k\geq3n/4)\\
 &\leq\frac43e^{-n/(16F)}
       +\exp\left(-\frac{3n}{16F}+\frac{km}{3F}\right)\\
 &\leq\frac73e^{-n/(48F)}.
\end{align*}
The same calculation is valid for $k=0$.
For the upper tail, first suppose that $n\geq2m$ and put
$k=\lfloor2n/m\rfloor$. If $N_n(x)>2n/m$, the $(k+1)$st visit
occurs before time $n$, and therefore $W_k<n$.
Since $km\geq2n-m\geq3n/2$, the negative exponential bound gives
\[
 \PB\left(N_n(x)>\frac{2n}{m}\right)
 \leq\exp\left(\frac{n}{4F}-\frac{3km}{16F}\right)
 \leq e^{-n/(32F)}.
\]
Combining the tails gives $4e^{-n/(48F)}$ when $n\geq2m$.
When $n<2m\leq2F$, this bound exceeds one and is still valid.
Finally, $48F\leq96(1+\mathfrak h_\star)$, and a union bound over
all states proves \eqref{eq:adaptive_occupancy}.
\end{proof}

\subsection{Proof of Theorem~\ref{theorem:unknown_mu}}
Set $a\coloneq4096$ for this proof and put
\[
 q_j\coloneq\frac{\delta}{32(j+1)^2},\qquad
 u_t\coloneq\frac{\delta}{32(t+1)^2},\qquad
 D\coloneq8a.
\]
Call epoch $j$ eligible if its starting time $n=n_j$ satisfies
\begin{equation}\label{eq:adaptive_eligible}
 n\geq D\left(\frac{\varrho H^2}{\mumin}+\mathfrak h_\star\right)
                   \mathcal L_n.
\end{equation}
All eligibility conditions are deterministic and used only in the
analysis. Write $E_j$ for the event that
$N_n(x)/n\in[\mu(x)/2,2\mu(x)]$ for every state $x$.
For an eligible epoch, Lemma~\ref{lemma:adaptive_occupancy} and
\[
 \log\frac{4d}{q_j}\leq\mathcal L_n,
 \qquad
 n\geq96(1+\mathfrak h_\star)\mathcal L_n
\]
give $\PB(E_j^c)\leq q_j$.
Also $n\mumin\geq1$, so on $E_j$ the floor in
\eqref{eq:adaptive_counts} preserves
\[
 \frac{\mumin}2\leq\widehat{m}_j\leq2\mumin.
\]
The cap in \eqref{eq:adaptive_schedule} is inactive in an eligible
epoch on $E_j$, and consequently
\begin{equation}\label{eq:adaptive_step_bracket}
 \frac{aH\ell_n}{2\mumin n}
 \leq\alpha_j\leq\frac{2aH\ell_n}{\mumin n}.
\end{equation}

\paragraph{Conditioning at an epoch boundary.}
Conditional on the forward history $\gF_n$, the scalar $\alpha_j$ is
fixed, the state at the beginning of the epoch is $S_n$, and the future
marked transitions, including the importance weights, obey the original
behavior kernel.
Proposition~\ref{prop:decay_noise} and
Corollary~\ref{corollary:decay_power_forget} therefore apply
conditionally on this history within the epoch. Their bounds are
uniform over the initial state. Events $E_j$ are $\gF_n$-measurable;
thus a conditional failure bound on $E_j$ bounds the unconditional
probability of its intersection with $E_j$.

\paragraph{Forgetting over a completed epoch.}
For an eligible epoch beginning at $n$, put
\[
 p\coloneq\left\lceil\frac{\mathcal L_n}{\ell_n}\right\rceil,
 \qquad \rho\coloneq(n+8)^{-1}.
\]
Then $\mathcal L_n\leq p\ell_n\leq2\mathcal L_n$.
On $E_j$, \eqref{eq:adaptive_step_bracket} and
\eqref{eq:adaptive_eligible} show that
\[
 p\varrho H\alpha_j\leq\frac{4a\varrho H^2\mathcal L_n}{\mumin n}\leq\frac12.
\]
For the block length in Corollary~\ref{corollary:decay_power_forget},
\[
 b_p\leq\frac{56n}{a\mathcal L_n}+14\mathfrak h_\star+1.
\]
If
\[
 k\coloneq\left\lceil\log_2
              \frac{d}{\mumin q_j\rho^p}\right\rceil,
\]
then $k\leq4\mathcal L_n$: indeed,
$\log(d/(\mumin q_j))\leq\mathcal L_n$ and
$p\log(n+8)\leq\mathcal L_n$, since
$\ell_n\geq2\log(n+8)$.
The eligibility condition therefore gives
\[
 b_pk\leq\frac{224n}{a}
          +56\mathfrak h_\star\mathcal L_n+4\mathcal L_n
 \leq\left(\frac{224}{a}+\frac{60}{D}\right)n<n.
\]
It follows from the conditional higher-power bound that
\begin{equation}\label{eq:adaptive_forgetting}
 \PB\left(E_j\cap
    \left\{\norm{\widehat{\bPhi}_{2n,n}}_{\infty\to\infty}
                       >(n+8)^{-1}\right\}\right)\leq q_j.
\end{equation}

\paragraph{Noise at every time in an epoch.}
For every deterministic integer $t\geq2$, let $n=2^j$ be the unique
epoch start with $n<t\leq2n$. Conditional on $\gF_n$,
Proposition~\ref{prop:decay_noise} bounds the failure probability by
$u_t$ for the inequality
\begin{equation}\label{eq:adaptive_noise_all}
 \norm{\bY_{t,n}}_\infty
 \leq7H\sqrt{\varrho\alpha_j}\log\frac{4d}{u_t\mumin}
                +2\alpha_j\varrho H^2.
\end{equation}
This conditional assertion holds for every realized $\alpha_j$, since
$0<\alpha_j\leq(\varrho H)^{-1}$. For an eligible epoch on $E_j$, use
$t\leq2n$, $\log(4d/(u_t\mumin))\leq2\mathcal L_n$, and
\eqref{eq:adaptive_step_bracket} to bound its right side by
\[
 B_n\coloneq14\sqrt{2a}
      \sqrt{\frac{\varrho H^3\ell_n\mathcal L_n^2}{\mumin n}}
       +\frac{4a\varrho H^3\ell_n}{\mumin n}.
\]

\paragraph{One event for all terminal times.}
Take the intersection of the occupancy events for eligible epochs,
their forgetting events \eqref{eq:adaptive_forgetting}, and all the
noise events \eqref{eq:adaptive_noise_all}. The preceding conditional
arguments and a union bound give total failure probability at most
\[
 2\sum_{j\geq0}q_j+\sum_{t\geq2}u_t<\delta.
\]
On this event, every eligible completed epoch satisfies, by
\eqref{eq:decay_decomposition} and the pathwise bound
$\norm{\bDelta_n}_\infty\leq H$,
\begin{equation}\label{eq:adaptive_endpoint}
 \norm{\bDelta_{2n}}_\infty\leq\frac{H}{n+8}+B_n.
\end{equation}
Now fix an integer $T\geq8$, and choose $n=2^j$ with $n\leq T<2n$.
If \eqref{eq:adaptive_transient} holds with $C_1\geq4D$, both $n/2$
and $n$ satisfy \eqref{eq:adaptive_eligible}, because
$n/2>T/4$ and $\mathcal L_{n/2}\leq\mathcal L_n\leq\mathcal L_T$.
Equation~\eqref{eq:adaptive_endpoint} for the preceding epoch and
sup-norm contraction over the current partial epoch give
\[
 \norm{\bDelta_T}_\infty
 \leq\norm{\bDelta_n}_\infty+\norm{\bY_{T,n}}_\infty
 \leq\frac{H}{n/2+8}+B_{n/2}+B_n.
\]
At $T=n$ the current noise is zero. Using $n/2>T/4$,
$T\geq8$, and monotonicity of $\ell_t$ and $\mathcal L_t$ proves
\eqref{eq:adaptive_rate} for a universal $C_2$.
Finally, increasing the universal constant in \eqref{eq:adaptive_sample}
ensures \eqref{eq:adaptive_transient} and makes all three terms in
\eqref{eq:adaptive_rate} at most $\varepsilon/3$. This uses
$H>1$, $0<\varepsilon\leq1$, and
$\ell_T,\mathcal L_T\geq1$. The same event works for every
$\varepsilon$ and $T$, proving the theorem.

\section{Proof of the Return-Variance Guarantee}
\label{Appendix:variance}
The following elementary perturbation identity transfers the scalar
analysis to this coupled recursion. Its conclusion is pathwise and
therefore does not require the perturbations to be centered.

\begin{lemma}[Bounded target perturbations]\label{lemma:target_perturbation}
Let two asynchronous recursions use the same states, step sizes
$\alpha_t>0$, nonnegative weights $W_t$ satisfying $\alpha_tW_t\leq1$,
and discount $q\in[0,1)$. In both recursions the weight multiplies
the entire temporal-difference residual. Suppose they agree at time $s$,
and that their one-step reward terms differ by $b_t$ at each time
$s\leq t<T$.
If $|b_t|\leq a$ on this interval, then
\[
 \|\bm W_T-\widetilde{\bm W}_T\|_\infty\leq\frac a{1-q}.
\]
\end{lemma}
\begin{proof}
Write
$B_t=I-\alpha_tW_te_{S_t}e_{S_t}^\top
       +q\alpha_tW_te_{S_t}e_{S_{t+1}}^\top$
and $\Phi_{T,t+1}=B_{T-1}\cdots B_{t+1}$.
All these matrices are nonnegative. Since
$B_t\bone=\bone-(1-q)\alpha_tW_te_{S_t}$, telescoping gives
\[
 \sum_{t=s}^{T-1}\alpha_tW_t\Phi_{T,t+1}e_{S_t}
 =\frac1{1-q}\bigl(\bone-\Phi_{T,s}\bone\bigr).
\]
Unroll the difference recursion, take entrywise absolute values, and
use $0\leq\Phi_{T,s}\bone\leq\bone$.
\end{proof}

\begin{proof}[Proof of Proposition~\ref{prop:variance_extension}]
Applying the Markov reward property to the first-step decomposition
of the discounted return gives \eqref{eq:second_moment_bellman}.
Set
\[
 B\coloneq1+2\gamma H=(1+\gamma)H,
 \qquad H_2=\frac H{1+\gamma},\qquad BH_2=H^2.
\]
The normalized oracle reward
\[
 Z_t\coloneq\frac{R_t^2+2\gamma R_tV^\pi(S_{t+1})}{B}
\]
lies in $[0,1]$ and defines a time-homogeneous reward-transition kernel
under both $b$ and $\pi$. With discount $\gamma^2$, its target-policy value
function is $\bm M/B$. The same reward bound with $V_t$ in place of
$V^\pi$ proves the asserted invariant range of
$\widehat{\bm M}_t$.

Apply Theorem~\ref{theorem:decay} to the mean process with confidence
parameter $\delta/2$, and let $\mathcal E_V$ be its event of
probability at least $1-\delta/2$.
Choose $C_1=4C_0$, where $C_0$ is a valid transient constant in that
theorem, increasing it if necessary.
Fix an eligible integer $T$, and put $s=\lfloor T/2\rfloor$.
Then $s\geq8$ and
\[
 s\geq C_0L_s\left(\frac{\varrho H^2}{\mumin}+\mathfrak h_\star\right).
\]
Indeed, $s\geq T/3$ for $T\geq16$, $L_s\leq L_T$, and the
choice of $C_1$ gives this inequality. The same condition holds at
every integer $t\geq s$, by monotonicity of $t/L_t$.
On $\mathcal E_V$, use $t+8\geq(s+8)\geq(T+8)/2$ and the
monotonicity of the logarithms to obtain, for all $s\leq t\leq T$,
\[
 \|\bV_t-\bV^\pi\|_\infty\leq A_T,
\]
where
\[
 A_T\coloneq C\left\{\frac{H}{T+8}
 +\sqrt{\frac{\varrho H^3\overline L_TL_T^2}{\mumin(T+8)}}
 +\frac{\varrho H^3\overline L_T}{\mumin(T+8)}\right\},
\]

where $C$ is a sufficiently large universal constant.

For this $T$ only, define an auxiliary oracle recursion on $[s,T]$.
It starts at $\widehat{\bm M}_s/B$, uses the steps $\theta_t$ with
their original global time indices and weights $\rho_t$, and replaces
$V_t$ by $V^\pi$ in the normalized reward. Denote its terminal iterate by $\bm U_T^{(T)}$.
This is an auxiliary process for the proof; neither actual recursion
is restarted.

The suffix bounds in the proof of Theorem~\ref{theorem:decay}, applied
with horizon $H_2$ and confidence parameter $\delta/2$, give
\begin{equation}\label{eq:variance_oracle_suffix}
 B\left\|\bm U_T^{(T)}-\frac{\bm M}{B}\right\|_\infty
 \leq C\left\{\frac{H^2}{T+8}
 +\sqrt{\frac{\varrho H^5\overline L_TL_T^2}{\mumin(T+8)}}
 +\frac{\varrho H^4\overline L_T}{\mumin(T+8)}\right\}
\end{equation}
outside an event of probability at most
$\delta/[4(T+1)^2]$.
For completeness, the only probability estimates used here are
Proposition~\ref{prop:decay_noise} for the oracle stochastic convolution
and Corollary~\ref{corollary:decay_power_forget} for its propagation
matrix. Use
$q_T=\delta/[8(T+1)^2]$, $\rho_T=(T+8)^{-1}$,
$\overline L_t^{(2)}=\log(16d\varrho H_2(t+8)/\mumin)$,
and $L_t^{(2)}=\overline L_t^{(2)}+\log(2/\delta)$ in that proof.
Our threshold implies its threshold because $H_2\leq H$ and
$L_t^{(2)}\leq L_t$.
The oracle initial error is at most $H_2$ pathwise, since both its
initial estimate and target lie in $[0,H_2]^d$.
Thus its random initialization introduces no extra failure event
and requires no conditioning on $\widehat{\bm M}_s$.
Finally,
\[
 B^2H_2^3=\frac{H^5}{1+\gamma}\leq H^5,
 \qquad BH_2^3=\frac{H^4}{(1+\gamma)^2}\leq H^4,
\]
which gives \eqref{eq:variance_oracle_suffix} after rescaling.
A union bound over all eligible $T$ gives an event
$\mathcal E_M$ of probability at least $1-\delta/2$ on which this
inequality holds simultaneously.

On $\mathcal E_V$, the actual and oracle normalized rewards differ
by at most $2\gamma A_T/B$ throughout $[s,T)$.
Lemma~\ref{lemma:target_perturbation}, with $W_t=\rho_t$, therefore gives
\[
 \|\widehat{\bm M}_T-B\bm U_T^{(T)}\|_\infty
 \leq2\gamma H_2 A_T\leq HA_T.
\]
Also, the range of the mean estimates implies
\[
 \|\bV_T^{\odot2}-(\bV^\pi)^{\odot2}\|_\infty
 \leq2H\|\bV_T-\bV^\pi\|_\infty\leq2HA_T.
\]
Since the true variance belongs to $[0,H^2/4]^d$, projection does
not increase the error. On $\mathcal E_V\cap\mathcal E_M$, combine
the last two displays with \eqref{eq:variance_oracle_suffix} and the
definition of $A_T$ to obtain \eqref{eq:variance_extension_rate}
for a universal constant $C_2$.

For $0<\varepsilon\leq1$, a sufficiently large universal $C$ makes
\[
 T\geq C\left\{
 \frac{\varrho H^5\overline L_TL_T^2}{\mumin\varepsilon^2}
 +\mathfrak h_\star L_T\right\}
\]
imply the transient condition and make the three terms of
\eqref{eq:variance_extension_rate} sum to at most $\varepsilon$.
The same count also makes the mean error at most $\varepsilon$
on $\mathcal E_V$, by Theorem~\ref{theorem:decay}.
This proves the joint transition complexity on the same event.
\end{proof}

\end{document}